\documentclass[twoside]{article}
\usepackage[preprint]{aistats2026}
\usepackage[round]{natbib}

\usepackage[utf8]{inputenc} 
\usepackage[T1]{fontenc}    
\usepackage{hyperref}       
\usepackage{url}            
\usepackage{booktabs}       
\usepackage{amsfonts}       
\usepackage{nicefrac}       
\usepackage{xcolor}         

\hypersetup{urlcolor=blue, colorlinks=true, citecolor=green!50!black, linkcolor=blue}
\usepackage{multirow}
\usepackage[T1]{fontenc} 
\usepackage{amsmath}
\usepackage{amsfonts}
\usepackage{amssymb}
\usepackage{amsthm}
\usepackage{thmtools}
\usepackage{pgffor}
\usepackage{xcolor}
\usepackage{cleveref}
\usepackage{graphicx}
\usepackage{geometry}
\usepackage{enumitem}

\newtheorem{definition}{Definition}
\newtheorem{theorem}{Theorem}
\newtheorem{lemma}{Lemma}
\newtheorem{proposition}{Proposition}
\newtheorem{corollary}{Corollary}

\usepackage{thmtools}

\newcommand{\Real}{\mathbb{R}}

\newcommand{\E}{\mathbb{E}}

\renewcommand{\tilde}{\widetilde}
\renewcommand{\hat}{\widehat}
\renewcommand{\bar}{\overline}

\newcommand{\dist}{\operatorname{dist}}%

\foreach \x in {A,...,Z}{%
	\expandafter\xdef\csname m\x\endcsname{\noexpand\mathbf{\x}}
}
\foreach \x in {A,...,Z}{%
	\expandafter\xdef\csname om\x\endcsname{\noexpand\overline{\noexpand\mathbf{\x}}}
}

\foreach \x in {A,...,Z}{%
	\expandafter\xdef\csname c\x\endcsname{\noexpand\mathcal{\x}}
}

\usepackage{mathtools}

\makeatletter
\renewcommand{\paragraph}{%
	\@startsection{paragraph}{4}%
	{\z@}{1.25ex \@plus 1ex \@minus .2ex}{-1em}%
	{\normalfont\normalsize\bfseries}%
}
\makeatother

\usepackage{tcolorbox}

\newcommand{\shortminus}{\scalebox{0.5}[1.0]{\( - \)}}

\usepackage{cancel}

\usepackage{geometry}
\usepackage{tikz}

\usetikzlibrary{arrows.meta,
                positioning,
                shadows}
                
\usetikzlibrary{shapes,snakes}
\usepackage{bm}
\usepackage{xspace}

\usepackage{cancel}

\setlist[itemize]{topsep=1pt, itemsep=1pt, parsep=1pt, leftmargin=*}

\usepackage{subcaption}

\newcommand{\PE}{\mathsf{p}}

\foreach \x in {a,...,z}{%
	\expandafter\xdef\csname v\x\endcsname{{\noexpand\boldsymbol{\noexpand\mathrm{\x}}}}
}

\newcommand{\Lip}[1]{\left\|#1\right\|_{\mathrm{Lip}}}
\newcommand{\LocLip}[2]{\left\|#1\right\|_{\mathrm{Lip}\left(#2\right)}}
\newcommand{\LLip}[1]{L_{#1}}

\newcommand{\msg}{\mathsf{v}}

\renewcommand{\next}{\bm{\mathsf{next}}}
\newcommand{\att}{\mathsf{logit}}
\newcommand{\Att}{\mathsf{Att}}
\newcommand{\logit}{\mathsf{Logit}}

\newcommand{\Msg}{\mathsf{V}}

\newcommand{\hst}{\mathsf{h}}
\newcommand{\pool}{\mathsf{pool}}

\newcommand{\event}{\mathcal{E}}

\newcommand{\sparsity}{\gamma}

\usepackage{amsmath}
\usepackage{xfrac}

\newcommand{\rvx}{X}

\newcommand{\sampledset}{\bm{X}}

\newcommand{\GAT}{\Theta}

\newcommand{\nlayer}{k}
\newcommand{\nlayers}{K}

\newcommand{\ef}{\mathrm{w}}

\newcommand{\nf}{\mathrm{f}}

\newcommand{\HSdim}{{d_{h}}}
\newcommand{\PEdim}{{d_{p}}}

\newcommand{\outdim}{{d_{out}}}

\DeclareMathOperator*{\diag}{diag}

\newcommand{\regions}{\mathcal{J}}
\newcommand{\region}{\bm{I}}
\newcommand{\nsamples}{n}

\newcommand{\tokenset}{tokenset\xspace}
\newcommand{\tokensets}{tokensets\xspace}

\newcommand{\Tokenset}{Tokenset\xspace}

\newcommand{\hypothesis}{\bm{\mathcal{H}}}

	\newcommand{\loss}{\mathcal{L}}

	\newcommand{\natdist}{P_{\geq N}}

	\newcommand{\sample}[1]{\sim_{#1}}

	\newcommand{\outEE}{
		\mathop{\mathbb{E}}\limits_{%
			\substack{
			n_i \sim \natdist
		}}
	}

\newcommand{\TT}{\bm{\mathcal{T}}}
\newcommand{\ST}{\bm{T}}
\newcommand{\admconsts}{(L_{\nf}, C_\chi, D_\chi)}

\newcommand{\outE}{
		\mathop{\mathbb{E}}\limits_{%
			 \substack{
				\{\TT_i,\,y_i\}_{i=1}^m \sim \nu^m 
			\\
			n_i \sim \natdist
			\\
			{\ST}_i\sample{n_i} \TT}
		}
	}
\newcommand{\stdef}{(\sampledset, \nf)}

\newcommand{\ttdef}{((\chi, \dist, \mu), \nf)}
\newcommand{\eventtTS}[3]{\mathcal{E}(#1, #2, #3)}

\newcommand{\eventT}[3]{\mathcal{E}(#1, #2, #3)}

\newcommand{\emp}{\mathrm{emp}}
\newcommand{\RNum}[1]{\uppercase\expandafter{\romannumeral #1\relax}}

\newcommand{\constI}{C_{I}}

\newcommand{\tauUB}{\tfrac{1}{4\sqrt{C_\chi}} n^{\nicefrac{1}{(D_\chi + 2)}}}

\newcommand{\PEstable}{\rho}
\newcommand{\PEparams}{(\rho, R)}

\newcommand{\offdiag}{\mathrm{offdiag}}

\newcommand{\roughconst}{54}

\newcommand{\ncondition}[1]{#1^{\frac{1}{D_\chi + 2}} \geq 24C_\chi\log #1}

\newcommand{\normadj}{p}
\newcommand{\mPi}{\bm{\Pi}}
\newcommand{\mDelta}{\bm{\Delta}}

\newcommand{\HIdef}{16^{\nlayers + 1} L_{\pool}  \nlayers \LLip{\msg}^{\frac{\nlayers(\nlayers + 1)}{2}}
\LLip{\att}^{\nlayers} 
\left(\LLip{\nf} + \nlayers \LLip{\PE}\right)  \sqrt{C_\chi}}

\newcommand{\HIIdef}{16^{\nlayers + 1} L_{\pool}  \nlayers \LLip{\msg}^{\frac{\nlayers(\nlayers + 1)}{2}}
					\LLip{\att}^{\nlayers} 
					{C_\chi} }

\newcommand{\HIIIdef}{16^{\nlayers + 1} L_{\pool}  \nlayers \LLip{\msg}^{\frac{\nlayers(\nlayers + 1)}{2}}
					\LLip{\att}^{\nlayers} R \tau }

\begin{document}

\twocolumn[
	\aistatstitle{From Small to Large: Generalization Bounds for Transformers on Variable-Size Inputs}

	\aistatsauthor{ Anastasiia Alokhina \And Pan Li }

	\aistatsaddress{ Georgia Institute of Technology } 
]

\begin{abstract} 
	Transformers exhibit a notable property of \emph{size generalization}, demonstrating an ability to extrapolate from smaller token sets to significantly longer ones. This behavior has been documented across diverse applications, including point clouds, graphs, and natural language. Despite its empirical success, this capability still lacks some rigorous theoretical characterizations. In this paper, we develop a theoretical framework to analyze this phenomenon for geometric data, which we represent as discrete samples from a continuous source (e.g., point clouds from manifolds, graphs from graphons). Our core contribution is a bound on the error between the Transformer's output for a discrete sample and its continuous-domain equivalent. We prove that for Transformers with stable positional encodings, this bound is determined by the sampling density and the intrinsic dimensionality of the data manifold. Experiments on graphs and point clouds of various sizes confirm the tightness of our theoretical bound. 
\end{abstract}

\section{Introduction}
Originally introduced for language processing, Transformers have emerged as fundamental backbones in a wide variety of machine learning models~\citep{vaswani2017attention}. Their key mechanism that treats input as sets of tokens enables seamless adaptation to diverse data modalities that can be tokenized, including text, image patches, point clouds, and even graph-structured data~\citep{zhao2021point, ying2021transformersreallyperformbad, qu2024particletransformerjettagging, dosovitskiy2021imageworth16x16words}. Although the data may possess an inherent order or structure, Transformers typically treat tokens as an unordered set and rely on domain-specific positional encodings to convey the necessary ordering or spatial information of these tokens.

Transformers often display an intriguing capability known as \emph{size generalization}, where a model trained on a specific range of input sizes can successfully extrapolate to substantially different, often larger inputs, for example, from smaller molecules to larger ones or from fixed to longer context windows in language~\citep{zhou2024transformersachievelengthgeneralization, qu2024particletransformerjettagging, Ngo_2023}.

To shed light on this phenomenon, existing work has largely concentrated on the language domain. Some studies provide empirical evidence by focusing on structured tasks like arithmetic operations, algorithmic problems, and string transformations \citep{cai2025extrapolationassociationlengthgeneralization, anil2022exploringlengthgeneralizationlarge, zhou2023algorithmstransformerslearnstudy}. Others have sought to establish theoretical foundations: \citet{cho2024positioncouplingimprovinglength} proved that augmenting position IDs for key tokens improves length generalization on addition tasks, while \citet{luca2025positional} derived generalization bounds with positional-only attention for tasks like sorting and prefix sum.  \citet{golowich2025rolesparsitylengthgeneralization} showed that length generalization can be explained by sparsity assumptions in the language structure, and \citet{huang2025formalframeworkunderstandinglength} introduced a theoretical framework for causal Transformers (thus only for sequential data) with learnable absolute positional encodings.


However, the aforementioned theoretical frameworks are insufficient to explain the remarkable success of Transformers in geometric domains due to the different data formats~\citep{huang2025formalframeworkunderstandinglength, trauger2023sequencelengthindependentnormbased}. This success includes wide-ranging applications on point clouds of varying cardinalities, from particle configurations in physics to 3D scenes in autonomous driving~\citep{mikuni2021point, miao2024locality, wu2024point, zhao2021point}, and on graphs of various sizes, such as molecules and proteins~\citep{fabian2020molecularrepresentationlearninglanguage, choukroun2022geometrictransformerendtoendmolecule}.


On the other side, in contrast to Transformers, size generalization has been studied more extensively for Message Passing Neural Networks (MPNNs) on graph-structured data \citep{huang2024enhancing,maskey2022generalization,le2023limits}. A dominant line of work leverages the graphon framework~\citep{lovasz2012large}, which views individual graphs as discretizations of an underlying continuous structure. Within this framework, studies have shown that MPNNs trained on smaller graphs converge to a consistent solution on arbitrarily large graphs drawn from the same graphon distribution~\citep{ruiz2023transferability,ruiz2020graphon,levie2023graphon,cai2022convergence,herbst2025higher}. The expanding scope of Transformer applications thus raises a crucial question: can a similar theoretical foundation be established for Transformers, and what unique aspects of their architecture must such an analysis account for?


\paragraph{Main Contribution.}
In this paper, we establish size-generalization guarantees for Transformer architectures on geometric inputs. We model each data sample as a \emph{tokenset}: a collection of tokens sampled independently from an underlying continuous domain, such as point clouds from a manifold or graphs from their continuous limits (graphons). Our approach extends the theoretical framework developed for MPNNs to the Transformer architecture \citep{maskey2022generalization}. Unlike MPNNs, which perform local message passing, Transformers use softmax attention for global information aggregation while employing positional encodings to capture geometric relationships. Consequently, our analysis must overcome additional challenges stemming from the influence of these positional encodings and the non-Lipschitz nature of the softmax function.

We begin by formally defining a Transformer class for processing both discrete, variable-sized tokensets and their continuous counterparts. We then derive an error bound on the quantity that compares the model's output on a discrete tokenset to that on its continuous limit. This bound scales as $\tilde{O}(n^{-\PEstable} + n^{-\nicefrac{1}{(D_\chi + 2)}})$, where $n$ is the number of sampled tokens, $\PEstable$ characterizes the stability of the positional encoding method, and $D_\chi$ represents an intrinsic dimension of the underlying domain. This theoretical bound is subsequently extended to provide size-generalization guarantees for learning tasks using Transformer architectures. Finally, we empirically validate the tightness of our theory on semi-synthetic graph and point cloud datasets. Our experiments show that the worst-case empirical convergence rate matches our theoretical prediction. Moreover, they demonstrate that stable positional encodings are crucial for achieving strong generalization performance.



\section{Preliminaries} \label{sec:preliminaries} 

This section introduces our notations, including the data model, the Transformer class of interest, and formulates the size-generalization problem. 
Let \((\chi,\dist,\mu)\) denote a probability-metric space. We use $\Lip{f}$ to denote the Lipschitz constant of a function $f: \chi \to \Real^d$  equipped with metric $\dist$: $\Lip{f} = \sup_{x, y \in \chi} \frac{\|f(x) - f(y)\|_2}{\dist(x, y)}$. For a function $f$ on a domain $\chi$, we use $\|f \|_{2, \infty}$ to denote its maximum norm over the domain: $\|f \|_{2, \infty} = \sup_{x \in \chi} \|f(x)\|_2$. For vector $\vv \in \Real^d$, we use $\| \vv \|$ and $\|\vv\|_2$ to denote Euclidean norm and $\|\vv\|_\infty$ to denote the infinity norm $\|\vv\|_\infty = \max_{i=1}^d |\vv_i|$.

\noindent \textbf{The Data Model.} \label{sec:data-model} As our main goal is to study the size-generalization property, the data model should capture a variety of continuous structures and their corresponding discretizations, including manifolds and point clouds, graphons and graphs. To do that, we consider the following definition that corresponds to a set of tokens, referred to as a \emph{tokenset}.

\begin{definition}[\Tokenset] \label{def:tokenset}
	A tuple $\TT = \ttdef$ is called a continuous \tokenset, where $(\chi, \dist, \mu)$ is a probability-metric space of tokens and $\nf: \chi \to \Real^\HSdim$ denotes the token features.
\end{definition}

We assume that each data sample is a discretization of some continuous 
\tokenset $\ttdef$ 
.

\begin{definition}[Sampled \tokenset] \label{def:sampled-tokenset}
	A discrete \tokenset $\stdef$ is called a \emph{sampled} \tokenset from a continuous \tokenset $\ttdef$, denoted as \(\stdef \sample{n} \ttdef\), if  for a given integer $n$ it is constructed by sampling $n$ tokens independently:
	\begin{align*}
		\sampledset = \{\rvx_1, \ldots,  \rvx_n\}, \rvx_i                 
		\overset{\text{i.i.d.}}{\sim} \mu(\chi) 
	\end{align*}
\end{definition}
Note that the discrete \tokenset $\stdef$ can also be viewed as a continuous one $((\sampledset, \dist, \mu_{\sampledset}), \nf)$ with uniform measure $ \mu_{\sampledset} = \frac{1}{|\sampledset|} \sum_{\rvx_i \in \sampledset} \delta_{\rvx_i}$ where $\delta_{\rvx_i}$ is the Dirac measure at $\rvx_i$. 

This data model can be used to represent data from various domains where Transformers have been empirically successful ~\citep{zhao2021point, rong2020selfsupervisedgraphtransformerlargescale, ying2021transformers, xiong2019pushing, hebert2023predictinghatefuldiscussionsreddit}. For instance, tokens in point clouds in $\mathbb{R}^d$ can be viewed as points sampled from a continuous space. 
In this case, the underlying space $(\mathcal{M}, \| \cdot \|_2, \mu)$ is a compact manifold $\mathcal{M} \subset \mathbb{R}^d$ with a probability measure $\mu$. 
Each point's feature map, $\nf(x)$, may contain its spatial features, e.g., coordinates or surface normal vectors, and other attributes, such as color or intensity \citep{qi2017pointnetdeeplearningpoint}. 

Analogously, graphs can be modeled as samples from a graphon~\citep{erdos1959random}, the continuous limit of graphs as their size tends to infinity. Here, tokens correspond to vertices subsampled from a continuous manifold, typically the unit interval $[0, 1]$, and edges are generated according to a kernel defined on pairs of vertex tokens. 

\paragraph{The Transformer Architecture.} Transformers \citep{vaswani2017attention} over \tokensets define mappings from token-wise signals to a finite-dimensional output space via a sequence of global, attention-based updates. Given token embeddings $h_i$, $i=1,2,...,n$, the output of an attention layer for token $i$ is defined as:
\begin{gather*}
z_i = \sum\nolimits_{j=1}^{n} \Att_{ij} \cdot \underbrace{(h_j \mV)}_{\text{value}}, \\
\text{ where }
	\Att_{ij} = \text{softmax}
	      \Bigl({
		\tfrac{1}{\sqrt{d_k}}
		\langle 
				h_i \mQ
				,
				h_j \mK
		\rangle
		}
		\Bigr)
\end{gather*}
where $h_i$ are the token embeddings, $\mQ, \mK, \mV \in \mathbb{R}^{d \times d_k} $ are learnable projection matrices, and $d_k$ is the dimension of the key and value vectors. This architecture was extended to point clouds and graphs by stacking multiple attention layers \citep{veličković2018graphattentionnetworks,zhao2021point}. One key component in the attention mechanism is the use of \emph{positional encodings} (PEs), enabling the model to capture structural or spatial relationships between tokens. In this case, the attention logits are modified to incorporate positional information, for example
$\logit_{ij} = \frac{g(h_i \mQ, h_j \mK, \PE(h_i, h_j))}{\sqrt{d_k}}$ for some function $g$,
where \( \PE(h_i, h_j) \) encodes the relative or absolute positional relationship between the $i$-th and $j$-th tokens. This allows attention scores to reflect both semantic similarity and structural proximity of the tokens.


\paragraph{Relative Positional Encodings.}
Empirical evidence suggests that relative positional encodings (RPEs) often outperform absolute PEs in tasks where local or relational structure matters~\citep{shaw2018selfattentionrelativepositionrepresentations, black2024comparinggraphtransformerspositional, su2023roformerenhancedtransformerrotary}. For example, for point clouds in $\Real^d$, an RPE of the form $\phi(x - y)$ for points $x, y$ and some function $\phi$ is naturally shift-invariant \citep{zhao2021point}. More broadly, RPEs
have been observed to improve model stability and generalization in a wide range of applications~ \citep{huang2024stabilityexpressivepositionalencodings, li2024functionalinterpolationrelativepositions, su2023roformerenhancedtransformerrotary, dai2019transformerxlattentivelanguagemodels}.

While many Relative Positional Encodings (RPEs) are designed to encode pairwise relations, in practice, their values often depend on the global structure of the entire input. For instance, the shortest-path distance between two graph vertices, a common RPE, is determined by the connectivity of other vertices~\citep{black2024comparinggraphtransformerspositional,li2020distance}. To make this dependence explicit, we denote a discrete RPE as $\PE(\rvx_i, \rvx_j; \ST)$, where $\rvx_i,\rvx_j\in\sampledset$ and $\ST=\stdef$, and its continuous counterpart as $\PE(x,y;\TT)$, where $x,y\in\chi$ for the domain $\TT=(\chi,\dist,\mu)$. We also use the shorthand notations $\PE_{\ST}(\rvx_i, \rvx_j)$ and $\PE_{\TT}(x, y)$, respectively. As we will show, the stability of these RPEs is crucial for the generalization capability of the Transformers that use them. We will formally define this property later, but a concrete example of a stable RPE is one derived from the k-step transition probability matrix on graphs, which we discuss in detail in \Cref{sec:stable-positional-encodings}.



To enable size-invariant analysis, we extend the Transformer architecture to accommodate both continuous and discrete data domains in the following definition.

	\begin{definition}[Transformer] \label{def:transformer} 
		Suppose we are given a domain of tokensets $\mathcal{D}$, the number of layers $\nlayers$, logit functions $\att^{(i)}: \Real^\HSdim \times \Real^\HSdim \times \Real^\PEdim \to \Real$, value functions $\msg^{(i)}: \Real^\HSdim \to \Real^\HSdim$ for $i \in [\nlayers]$, pooling function $\pool: \Real^\HSdim \to \Real^\outdim$ and relative positional encoding (RPE) function $\PE(x, y; \TT) \in \Real^{\PEdim}$ defined for any tokenset $\TT = \ttdef \in \mathcal{D}$ and a pair of tokens from its underlying space $x, y \in \chi$ 
		where
		$\HSdim$ is a hidden state dimension and $\PEdim$ is RPE dimension. 

		The corresponding \textbf{transformer} model $\GAT_{(\att, \msg, \PE)}: \mathcal{D} \to \Real^\outdim$ takes a \tokenset $\TT = \ttdef$ as an input and defines the following layers.
		
		\begin{itemize}
			\item Initial hidden state: $\hst^{(0)}(x) = \nf(x)$ 
			\item Value embedding of $y$ : $\Msg^{(\nlayer)}(y) = \msg^{(t)} \left(\hst^{(\nlayer)}(y) \right)$ 
			\item Attention coefficients:	
			$$\Att^{(\nlayer)}(x, y) = 
			\frac{\exp \left(\logit^{(\nlayer)}(x, y)\right)}
			{\int\nolimits_{y \in \chi} \exp \left(\logit^{(\nlayer)}(x, y)\right) d\mu(y)}$$
			where logit function is defined as
			$$\logit^{(\nlayer)}(x, y) = \att(\hst^{(\nlayer)}(x), \hst^{(\nlayer)}(y), \PE(x, y; \TT)) $$

			\item Next hidden state: 
			$$\hst^{(\nlayer+1)}(x)
			=
			\int\nolimits_{y \in \chi} \Att^{(\nlayer)}
			(x, y)
			\Msg^{(\nlayer)}(y) d\mu(y) $$	

			\item Last pooling layer that defines the output: 
			$\GAT(\TT) = \pool\left(\int_{x\in \chi} \hst^{(\nlayers)}(x) d\mu(x) \right)$
		\end{itemize}
	
	\end{definition}

	Note that the value embedding of a token is often defined by projecting its input embedding $\hst(x)$ through a learned value matrix $\mV \in \mathbb{R}^{\HSdim \times \HSdim}$, i.e., $\msg(x) = \mV \hst(x) $. In practice, the Transformer architecture often includes additional linear layers and residual connections. For notational simplicity, we omit these components, though our results can be readily extended to incorporate them. Moreover, the above definition can be applied to discrete tokensets $\stdef$ to get back to the standard discrete Transformers by setting $\mu = \mu_{\sampledset}$.


\paragraph{Supervised Learning.} \label{sec:problem-formulation}
We study generalization properties of Transformers in the context of a standard classification formulation. Our analysis can be extended to regression problems. 
We consider the following dataset generation procedure. 

\begin{itemize} 
	\item Let 
	$\{ \TT^*_\gamma, y_\gamma\}_{\gamma = 1}^{\Gamma}$
	be a set of continuous \tokensets each associated with a label $y_\gamma$, which together form $\Gamma$ classes. 
	\item For the dataset $\Omega=\{\ST_i, y_i\}_{i=1}^m$, each data point 
	is IID sampled based on a distribution $\nu_N$ in a following way:
	\begin{itemize}
		\item Sample a continuous \tokenset $\TT$'s and its label $y$ from $\{ \TT^*_\gamma, y_\gamma\}_{\gamma = 1}^{\Gamma}$ by following a distribution $\nu$.
		\item Sample positive integer $n$ greater than $N$ by following a distribution $P_{\geq N}$. 
		\item Sample $n$-sized discrete \tokenset from the continuous one $\ST \sample{n} \TT$ and propagate the label $y$ for $\ST$.
	\end{itemize}
\end{itemize}

The generalization ability of a neural network for this task can be analyzed using uniform convergence bounds from statistical learning theory. Denote the expected (population) risk $R_{\exp}$ and the empirical risk $R_{\mathrm{emp}}$ as follows:
\begin{equation} \label{eq:risk}
R_{\exp} = \mathop{\mathbb{E}}\limits_{\Omega\sim \nu_N^m} 
\left[ \mathcal{L}(\GAT(\ST), y) \right], \,
R_{\mathrm{emp}} = \tfrac{1}{m} \sum_{i=1}^m \mathcal{L}(\GAT(\ST_i), y) 
\end{equation}
where \(\mathcal{L}\) is a Lipschitz loss function that measures the error between \(y\) and the output of the neural network. We would like to show that the generalization error $\varepsilon = \sup_{\GAT \in \mathcal{H}} |R_{\exp}(\GAT) - R_{\mathrm{emp}}(\GAT)|$ 
converges to zero as the number of training samples $m$ and the average number of tokens $n$ per sample grow, with high probability. 

The supremum over the hypothesis class \(\mathcal{H}\) accounts for the fact that the learned model \(\GAT\) depends on the sampled training set and cannot be treated as being fixed or independent of the data generation process;  hence, the convergence bound for Monte Carlo methods cannot be applied.
To establish this generalization bound, we will relate the model output on the discrete samples to that on its continuous counterpart by bounding $\|\GAT(\TT) - \GAT(\ST)\|$. This intermediate result captures the effect of input resolution and plays a central role in our analysis of size generalization. 

\section{Main Results} \label{sec:main-result}



\paragraph{Assumptions.}  Before presenting our main result, we state the assumptions on the data model and the Transformer architecture. We first define the class of \emph{admissible} \tokensets considered for discretization.

\begin{definition}[Admissible \tokenset] \label{def:admissible-tokenset}
	A continuous \tokenset $\ttdef$ is called \emph{$\admconsts$-admissible} if:
	\begin{itemize}
		\item $\nf$ is a Lipschitz function with respect to the metric $\dist$ on $\chi$: $\Lip{\nf} \leq L_{\nf}$.
		\item \label{asm:regularity-of-mu}
		For the given probability measure $\mu$ on $\chi$ there exists constants $C_\chi, D_\chi$ such that for any $r \geq 0$ 
		there is a lower bound on measure of any ball $B_{r/2}(x)$ of radius $r/2$ in $\chi$:
		$\mu(B_{r/2}(x)) \geq \frac{1}{C_\chi} r^{D_\chi}$
	\end{itemize}

	Without loss of generality, we assume that $C_\chi \geq 1$ and $\|\nf\|_{2, \infty} = 1$.
\end{definition}

Note that in the case of space $\chi$ being a manifold, a constant $D_{\chi}$ will capture its intrinsic dimension. The second assumption on the regularity of $\mu$ means that there are no ``holes'' in the support, and every ball of radius $r$ carries at least on the order of $r^{D_\chi}$ probability.  This assumption can be used to derive the assumption adopted by \citet{maskey2024generalization} for the analysis of MPNNs; we offer a more detailed comparison in \Cref{sec:previous-work}.

The properties of PEs are crucial for the generalization of Transformers. We formalize it as follows. In a simple point cloud setting, an RPE can be a function of the displacement between two points, e.g., $\PE_{\TT}(\rvx_i, \rvx_j) = \PE_{\ST}(\rvx_i, \rvx_j) = \phi(\rvx_i - \rvx_j)$ for some function $\phi$. This RPE is independent of other points in the tokenset and remains constant for the given pair $(\rvx_i, \rvx_j)$ when sampling $\ST \sample{n} \TT$. However, in graph settings, RPEs must encode the graph structure and thus often depend on the entire tokenset. For example, using the shortest-path distance between vertices as an RPE depends on the full graph connectivity \citep{ying2021transformers}. Since the Transformer's output error $\|\Theta(\TT) - \Theta(\ST)\|$ depends on the RPE error induced by sampling $\|\PE_{\TT}(\rvx_i, \rvx_j) - \PE_{\ST}(\rvx_i, \rvx_j)\|$, size-generalization requires that the discrete RPE, $\PE_{\ST}(\rvx_i, \rvx_j)$, converges to its continuous counterpart, $\PE_{\TT}(\rvx_i, \rvx_j)$.


\begin{definition} \label{def:stable-positional-encoding}
A relative positional-encoding function $\PE$ is called \emph{$\PEparams$-stable} for a domain of tokensets $\mathcal{D}$ if for any continuous tokenset $\TT \in \mathcal{D}$ and a discrete tokenset $\ST = \stdef \sample{n} \TT$, the following concentration bound holds for any $\rvx_i, \rvx_j\in \sampledset$:
\begin{gather*}
 \Pr\left[\bigl\|\PE(\rvx_i, \rvx_j; \ST) - \PE(\rvx_i, \rvx_j; \TT)\bigr\| \geq Rn^{-\rho} \tau\right] 
 \\
 \leq C n^2 (e^{-\tau^2} + e^{-\tau})
\end{gather*}
\end{definition}

In \Cref{sec:stable-positional-encodings}, we show that a commonly used RPE based on the random-walk transition matrix on graphs satisfies this stability condition, whereas RPEs based on shortest-path distances do not. The importance of this property is shown in \Cref{sec:experiments}, where we empirically show that Transformers using unstable RPEs suffer from a much higher generalization error.




Now let us define the hypothesis class $\hypothesis$ of the Transformers we consider.

\begin{definition}[Hypothesis class] \label{def:hypothesis-class}
	Let $\hypothesis$ denote a class of Transformers defined on a domain of tokensets $\mathcal{D}$ with $\nlayers$ layers such that for each $\nlayer \in [\nlayers]$: $\Lip{\att^{(\nlayer)}} \leq \LLip{(\att)}$, 
	$\Lip{\msg^{(\nlayer)}} \leq \LLip{\msg}$, $\Lip{\pool} \leq \LLip{\pool}$ with $\PEparams$-stable RPE method such that $\Lip{\PE_{\TT}(~\cdot~, x)}, \Lip{\PE_{\TT}(x, ~\cdot~)} \leq \LLip{\PE}$ for all $\TT =  \ttdef \in \mathcal{D}$ and $x \in \chi$.  Also, assume that $\msg$ (see (\Cref{def:transformer}) satisfies $\msg\left(\bar 0\right) = \bar 0$ where $\bar 0$ is an all-zero vector of dimension $\HSdim$. 

\end{definition}

This definition restricts our hypothesis class to $\nlayers$-layer Transformers that use stable RPE method and Lipschitz logit, value and pooling functions. The assumption $\msg\left(\bar 0\right)$ is imposed without the loss of generality only to simplify the notations in the analysis.

\textbf{Main Theorems.}
Now let us formulate the main results of the paper: convergence of the Transformer outputs on discretized tokensets and the generalization bound it yields on the learning task.

\begin{theorem}[Convergence of The Outputs] \label{lem:expected-error-short}
	Let $\TT = \ttdef$ be an $\admconsts$-admissible continuous \tokenset. Consider subsampling \tokenset $\ST = \stdef$ from  $\TT$. Then the following bound holds:
	\begin{gather*}
		\mathop{\mathbb{E}}\limits_{\ST \sample{n} \TT}
		\left[\sup_{\GAT \in \hypothesis}
		\left\|
		\GAT(\TT)
		-
		\GAT(\ST)
		\right\|_{\infty}
		\right] 
		\\
		\leq 
		(H_1 + H_2 \log n)n^{-\frac{1}{D_\chi + 2}} + H_3 n^{-\PEstable} \log n
	\end{gather*}
	where
	\begin{align*}
		H_1 &= O(\HIdef)
		\\
		H_2 &= O(\HIIdef) \\
		H_3 &= O(\HIIIdef)
	\end{align*}
\end{theorem}

For the full theorem statement and the proof refer  to \Cref{apx:generalization-bound} (\Cref{lem:expected-error-full}). 

\begin{theorem}[Generalization Error]
	\label{thm:main}
    Given a dataset $\Omega$ generated based on $\nu_N^m$, as depicted in the Supervise Learning task in Section~\ref{sec:data-model}
    and a hypothesis class of Transformers $\hypothesis$ (\Cref{def:hypothesis-class}), there exist constants $C_1, C_2, C_3 > 0$ such that 
	\begin{gather*}
		\mathop{\mathbb{E}}\limits_{\Omega\sim \nu_N^m} 
		\left[
			\sup_{\GAT \in \hypothesis}\left| R_{\emp}(\GAT)-R_{\exp}(\GAT)\right|
		\right] \leq 
		\\
			\mathop{\mathbb{E}}_{n \sim \natdist}\left[\underbrace{
				\frac{ \Lip{\loss} (C_1 + C_2 \log n)}{n^{{1}/{(D_\chi + 2)}}}
			}
			_
			{\text{discretization error}}
			+
			\underbrace{
				 \frac{ \Lip{\loss} C_3 \log n}{n^{\PEstable}} 
			}
			_
			{
				\text{RPE error}
			}
			\right]
			\\
			+ 
			\underbrace{
				\frac{2^\Gamma  \| \loss \|_{\infty}}{\sqrt{m}}
			}
			_
			{\substack{\text{statistical} \\ \text{error}}}
	\end{gather*}

	where $R_{\emp}, R_{\exp}$ are defined in \Cref{eq:risk}, and $\natdist$ is the distribution of the tokenset's size used to generate the dataset. Constants $C_1, C_2, C_3$ are of the same order as $H_1, H_2, H_3$ as in \Cref{lem:expected-error-short} correspondingly. 
	
\end{theorem}

This bound shows that the generalization error decomposes into three sources: (i) a discretization error that decays with the number of tokens $n$ and reflects how well finite samples approximate the underlying continuous \tokensets, (ii) an RPE error that captures the stability of PEs with respect to $n$, and (iii) the usual statistical error from having only $m$ training examples. In contrast to classical Monte Carlo concentration which provides an $\mathcal{O}(n^{-1/2})$ guarantee, the bound holds \emph{uniformly} over the entire hypothesis class. This distinction is crucial for the setting as the Transformer parameters $\GAT$ are not fixed but may vary with the training data. 
Our result ensures that the sampling error is controlled simultaneously for all $\GAT \in \hypothesis$, including the one chosen by the learning algorithm, and that it improves as the number of sampled tokens $n$ increases and as the stable RPE is used.

\subsection{Comparison with MPNNs} \label{sec:previous-work}

\citet{maskey2024generalization} considered a similar setting for MPNNs and random graph-signal models, where they derived a generalization bound
for classification tasks. They assumed their probability-metric space $(\chi, \dist, \mu)$ had a covering number less than $C_{\chi}r^{-D_\chi}$. However, this assumption is too weak for our case: we provide an example in \Cref{apx:example} of output error being $\Theta \left(\exp\left(\LLip{\msg}^\nlayers\right)\right)$ for the transformer from the considered hypothesis class (\Cref{def:hypothesis-class}) on the continuous tokenset with underlying covering number $2$. This exponential dependence on $\LLip{\msg}$ causes the bound to become vacuous quickly. The issue arises from the non-linear normalization and exponential terms in the Transformer's attention mechanism, whereas the analysis in \citet{maskey2024generalization} leverages the unparameterized average aggregation of standard MPNNs. 
Instead, our work adopts a slightly stronger assumption that the underlying measure is regular. Another difference from MPNNs is that the relational information is encoded into RPEs for Transformers' usage, not via edges of a graph to perform an average aggregation.
With our assumptions and the effect of RPE, we have a convergence rate of $\tilde O\left(n^{-\frac{1}{D_\chi + 2}} + n^{-\PEstable}\right)$ for the generalization error (\Cref{thm:main}), where $\PEstable$ is the stability parameter of the positional encoding method used (\Cref{def:stable-positional-encoding}), compare to $\tilde O\left(n^{-\frac{1}{2(D_\chi + 2)}}\right)$ achieved in \citet{maskey2022generalization}. Moreover, our bound often quantifies the dependence of the error on the stability of RPE, a factor that does not arise in MPNNs but is crucial for Transformers. 

\section{Proof Overview}

This section outlines the main components of the proof of \Cref{thm:main}. First, we analyze how discretization affects one layer of the Transformer. Then,  we explain how to derive the concentration bound on the output of a Transformer with multiple layers. Finally, we sketch how to turn it into the generalization bound.

We begin by comparing the outputs of a single Transformer layer given a continuous tokenset $\ttdef$ and its discrete sample $\stdef$. The distortion caused by sampling can be decomposed into a discretization error and a perturbation error. The discretization error quantifies how sampling tokens from the continuous domain affects the attention-based aggregation, assuming the continuous version of the RPEs, i.e., $\PE_{\TT}(\rvx_i, \rvx_j)$, is retained for all $\rvx_i, \rvx_j \in \sampledset$. The perturbation error then captures how the perturbations in the hidden states, combined with the use of the discretized RPEs  $\PE_{\ST}(\rvx_i, \rvx_j)$, propagate through the layer.

A key challenge in analyzing Transformer models stems from the non-linear aggregation and non-Lipschitz $e^{x}$ used in the attention mechanism. Consequently, the latent embeddings $h(x)$ at each layer are only locally Lipschitz, not globally Lipschitz, with respect to their initial token index $x$. Therefore, our analysis of the single-layer induced error must rely on the input embeddings $h(x)$ being only locally Lipschitz, which we state as follows.

\newcommand{\Const}{C}

	\begin{lemma}[One-Layer Discretization Error] \label{thm:agg-error-informal}
		
	Let $\TT = \ttdef$ be admissible continuous tokenset and $\stdef$ be subsampled \tokenset: $\ST \sample{n} \TT$. Assuming $n$ is large enough ($n = poly(\Lip{\att}, \Lip{\msg}, \Lip{\PE})$), for any $\tau > 2$ the following bound holds for every Lipschitz function $\att: \mathbb{R}^{\HSdim \times \HSdim} \rightarrow \mathbb{R}^\HSdim$ and every \emph{locally} Lipschitz functions $\hst: \chi \rightarrow \mathbb{R}^\HSdim$: 
	\begin{align*}
	\sup_{
		\dist(x, y) \leq n^{-{1}/{(D_\chi+2)}}
		} 
	\frac{\|\hst(x) - \hst(y)\|}{\dist(x, y)} \leq L
	\end{align*}
	 with probability at least $1 - O\left(n^{2 - \tau}\right)$:
	\begin{gather*}
	\max_{\rvx \in \sampledset}
	\left\|
		\frac{1}{\nsamples} \sum_{i=1}^\nsamples
		\Att_{\sampledset}(\rvx, \rvx_i)  
		\Msg(\rvx_i) 
		-
		\int\limits_{x \in \chi} 
		\Att_{\chi}(\rvx, x)  
		\Msg(x) 
		d \mu(x)
	\right\| 
	\\
	 \leq 
	 8 \Lip{\msg} \|\hst\|_{2, \infty}
	 \left(
		 L   \Lip{\att}
		 +
		 \sqrt{C_\chi} \tau  \log n
	 \right) 
	 \nsamples^{-\frac{1}{D_\chi + 2}}
	\end{gather*}

	where value is calculated as $\Msg(x) = \msg(\hst(x))$ and attention is calculated as
	\begin{align*}
		\Att_\chi(x, y) = \frac{\exp(\att(\hst(x), \hst(y)))}{\int\limits_{z \in \chi} \exp(\att(\hst(x), \hst(z))) d\mu(z)} 
		\\
		\Att_{\sampledset}(\rvx, \rvx_i) = \frac{\exp(\att(\hst(\rvx), \hst(\rvx_i)))}{\tfrac{1}{\nsamples}\sum_{j=1}^\nsamples 
		\exp(\att(\hst(\rvx), \hst(\rvx_j)))}
	\end{align*}

	\end{lemma}

 The full theorem statement that states specific conditions on $n$ and a precise definition of the constant and its proof are in  \Cref{thm:concentration-error-full} in \Cref{apx:one-layer}.
 
 Now, let us state the perturbation error result. 
    \begin{lemma}[One-Layer Perturbation Error]\label{lem:perturbation-error-informal}
		Consider applying one layer of transformer $\GAT \in \hypothesis$ with RPE function $\PE$ to some hidden state $\hst$ to get a next hidden state $\hst'$ as in \Cref{def:transformer}. Now suppose both {the} hidden state and {the} RPE were perturbed: $\hat \hst, \hat \PE$. Denote the corresponding next state as $\hat \hst'$.  Assuming the bound on the total perturbation $\forall x, y$:
		$
			 \left\|\hst(x) - \hat \hst(x)\right\| 
			+ 
			 \left\|\PE(x, y) - \hat\PE(x, y)\right\| = O(1/\Lip{\att})
		$,
		we have the following bound:
		\begin{gather*}
			\sup_{x} \left \|\hat\hst'(x) - \hst'(x) \right \|
			\leq \\
			8\LLip{\msg}\LLip{\att} \|\hst\|_{2, \infty}
			(
				\sup_{x}\left\|\hst(x) - \hat\hst(x)\right\| 
				+ 
				\sup_{x, y} \left\|\PE(x, y) - \hat\PE(x, y)\right\| 
			).
		\end{gather*}

    \end{lemma}

With \Cref{lem:perturbation-error-informal} and \Cref{thm:agg-error-informal} and our assumption on the regularity of the measure (\Cref{def:admissible-tokenset}), the local Lipschitzness of token embeddings $h(x_i)$ can be preserved across layers. This is non-trivial because global Lipschitzness cannot be preserved across layers due to the non-linear attention mechanism. 

Then, combining \Cref{lem:perturbation-error-informal}, \Cref{thm:agg-error-informal}, the assumption on stable RPE (\Cref{def:stable-positional-encoding}), and applying them recursively, we can establish the bound on the total error of the Transformer with multiple layers.

	\begin{theorem}[Output Error]\label{thm:concentration-error-short}
		Let  $\TT = \ttdef$ be a continuous \tokenset and let $\ST = \stdef$ be a corresponding sampled \tokenset $\ST \sample{n} \TT$. Assuming $n$ is large enough ($n > poly(\LLip{\att}, \LLip{\msg}, \LLip{\PE})$),  we have that for any $\tau > 2$  the following bounds holds uniformly over $\GAT \in\hypothesis$ (\Cref{def:hypothesis-class}) with probability $1 - O\left(n^{2-\tau}\right)$:
		\begin{gather*}
			\| \GAT(\TT) - \GAT(\ST) \| \leq  
			\underbrace{(H_1 + H_2 \tau \log n) \nsamples^{-\frac{1}{D_\chi + 2}} 
			}_{\text{discretization error}}
			+ 
			\underbrace{
				H_3 n^{-\PEstable} \log n
			}_{\text{RPE error}}
		\end{gather*}
		where constants $H_1, H_2, H_3$ polynomially depend on the Lipschitz parameters of the hypothesis class (\Cref{def:hypothesis-class}) and the parameters $\admconsts$.
		\begin{align*}
			H_1 &= O(\HIdef)
			\\
			H_2 &= O(\HIIdef) \\
			H_3 &= O(\HIIIdef)
		\end{align*}
	\end{theorem}

	For the full theorem statement  and the proof refer  to \Cref{apx:one-layer} (\Cref{thm:total-sampling-error-full}). Using the tail bound for the output error (\Cref{lem:tail-bound} in \Cref{apx:previous-results}), we can convert the high-probability bound into the expectation one stated in \Cref{lem:expected-error-short}.

	The proof of the main generalization bound stated in \Cref{thm:main} follows from \Cref{lem:expected-error-short} by combining it with the Breteganolle-Huber-Carol inequality (\Cref{thm:breteganolle-huber-carol} in \Cref{apx:previous-results}). Full theorem statement and the proof are provided in \Cref{apx:generalization-bound} (see \Cref{thm:expected-ge-full}).


    

	\subsection{Stable Positional Encodings} 	 \label{sec:stable-positional-encodings}

	A central assumption in our hypothesis class (\Cref{def:hypothesis-class}) is the stability of RPEs under subsampling from the continuous domain (\Cref{def:stable-positional-encoding}). While the importance of RPE stability has been studied in prior work \citep{huang2024stabilityexpressivepositionalencodings, kanatsoulis2025learningefficientpositionalencodings, luo2021stablefastaccuratekernelized,wang2022equivariant}, to our knowledge, these analyses do not address stability in the context of size-varying inputs.

A full characterization of which RPEs achieve stability in the size-varying setting is left for future work. Here, we demonstrate that this assumption is both valid and crucial in practice, as some RPEs widely used in graph settings satisfy the stability assumption while others do not \citep{rampavsek2022recipe}, potentially leading to a generalization gap.

Specifically, we show that for any two vertices $\rvx_i, \rvx_j$, the scaled random-walk transition probability $n[\mP^k]_{ij}$ converges to $p^{(k)}(\rvx_i, \rvx_j)$ as $n\rightarrow \infty$, where $\mP$ is the transition matrix for the graph with vertices $\{\rvx_1, \ldots, \rvx_n\}$, and $p^{(k)}$ is the $k$-step transition kernel of the corresponding graphon. Thus, random-walk landing probabilities, when used as RPEs, are stable. In contrast, we will later show that RPEs based on shortest-path distance are not stable. The following proposition formalizes the stability of the random-walk approach.

	\begin{proposition}\label{prop:stable-rpe-graphon}
		Consider
		\begin{itemize}
			\item Symmetric graphon $W: \chi \times \chi \to [0, 1]$ defined on probabilistic-metric space $(\chi, \dist,  \mu)$;
			\item Sparsity parameter $\sparsity = n^{\alpha - 1}$ for $\alpha > 1/2$;
			\item Undirected graph $G = (V, E)$ subsampled from the graphon: $V = \{\rvx_1, \ldots, \rvx_n \} \stackrel{\text{i.i.d.}}{\sim} \mu^n$ and edges are sampled independently as $E_{ij} = E_{ji} = \text{Bernoulli}(\sparsity W(\rvx_i, \rvx_j))$.
		\end{itemize}
		Define RPE function $\PE$ using $k$-step transition probability for $3 \leq k \leq \sqrt{n}$:
		\begin{itemize}
		\item for graphon: for points $x, y \in \chi$ be defined as a $k$-step transition kernel of the corresponding graphon:
		\begin{gather*}
			\PE(x, y; \TT) = \int\limits_{z \in \chi^{k-1}} p(x, z_1) p(z_1, z_2) \ldots p(z_k, y) d\mu(z) \\
			\text{ where } \quad p(x, y) = \frac{\ef(x, y)}{\int_{z \in \chi} \ef(x, z) d \mu(z)}
		\end{gather*}
		\item for graph: for $X_i, X_j \in \sampledset$ let 
		\begin{align*}
		\PE(X_i, X_j; \ST) = [n \mP^k  ]_{ij} \quad \text{ for } \quad \mP = \mD^{-1}\mA
		\end{align*}
		 where $\mA$ is the adjacency matrix: $\mA_{ij} = \hat w(\rvx_i, \rvx_j)$ for $i\neq j$, $\mA_{ii} = 0$ and $\mD$ is a diagonal degree matrix: $\mD_{ii} = \sum_{j} \mA_{ij}$.
		\end{itemize} 
			
		Then, $\PE$ is $\PEparams$-stable where $\rho = \min\left(\frac{\alpha}{2}, -\frac{1}{2} + \alpha\right)$ and $R =  \frac{4 \cdot 8^{k+1}}{\delta^{k+1}}$, specifically, for all $i, j \in [n]$:
		\begin{gather*}
			 \Pr[
				\|\PE_{\TT}(X_i, X_j) - \PE_{\ST}(X_i, X_j) \|_{\infty} 
				\geq 
				R n^{\shortminus\rho} \tau
			]
			\leq
			O(k n^2 e^{\shortminus\tau^2})
		\end{gather*}
	\end{proposition}

	The proof is provided in \Cref{apx:stable-positional-encoding}.  Note that the shortest-path distance-based RPEs are not stable. For example, given a constant value graphon $W(x, y) = 1/2$, the shortest-path distance between any sampled vertices $X_i, X_j$ is equal to $1$ with probability $1/2$ (if they are connected by an edge), and is $\geq 2$ with probability $1/2$ (if they are not connected by an edge). Hence, the discrete version $\PE_{\ST}(X_i, X_j)$ which is random will never converge to its continuous deterministic counterpart $\PE_{\TT}(X_i, X_j)$, no matter how many vertices are sampled. In our experiments (\Cref{sec:experiments:classification}), we show that using the random-walk-based RPEs lead to better generalization compared to the shortest-path-distance-based RPEs. This well aligns with existing empirical results in \citet{dwivedi2022graph, rampavsek2022recipe, ying2021transformersreallyperformbad}. 

	An important direction for future research would be establishing stability guarantees in size-variable scenarios for other positional encoding schemes used in practice, e.g., stable and expressive PEs \citep{huang2024stability} and resistance distance-based ones \citep{zhang2024rethinkingexpressivepowergnns}.


\section{Experiments}\label{sec:experiments}

Our experiments have two primary objectives. First, we aim to empirically validate our theoretical output error bound, $\E \left[\left\| \GAT(\TT) - \GAT\left(\ST\right) \right\|_{2} \right] = \tilde{O}(n^{-\nicefrac{1}{(2 + D_\chi)}} + n^{-\PEstable})$ (\Cref{lem:expected-error-short}), for a Transformer $\GAT$ from our hypothesis class $\hypothesis$ (\Cref{def:hypothesis-class}). This key intermediate result indicates the order dependence on the generalization bound derived in \Cref{thm:main}. Second, we seek to highlight the importance of RPE stability by comparing the performance of stable versus unstable RPEs in a graph regression setting.

\subsection{Transformer output error} \label{sec:experiments:worst-case}
Here, we evaluate the worst-case output error of a Transformer model $\GAT$ when applied to a continuous tokenset $\TT$ and its sampled counterpart $\ST$. We consider two data types: point clouds sampled from a mesh and graphs sampled from a graphon.

Our model architecture consists of a lightweight Transformer with a single attention layer, followed by global mean pooling and a two-layer Multi-Layer Perceptron (MLP). Attention scores are computed using the dot product between key and query vectors, augmented by a relative positional encoding bias:
\begin{gather*}
	\Att(\rvx_i, \rvx_j) = \mathrm{softmax}\left(\mK \nf(\rvx_i)^\top \mQ \nf(\rvx_j) + \phi(\PE(\rvx_i, \rvx_j))\right)
\end{gather*}
where $\mK$ and $\mQ$ are learnable linear projections, $\nf(\cdot)$ denotes input features, $\PE: \chi \times \chi \to \Real^{{\PEdim}}$ is the RPE, and $\phi: \Real^{{\PEdim}} \to \Real$ is a two-layer MLP with ReLU activation function. 
The operator norms of all parameter matrices are normalized to 1. 
This design choice aligns with the assumptions used in our theoretical analysis.

\begin{figure}[t]
	\centering
	\begin{subfigure}[b]{0.2\textwidth}
		\hspace*{0.5cm}
		\centering
		\includegraphics[
			width=\textwidth]{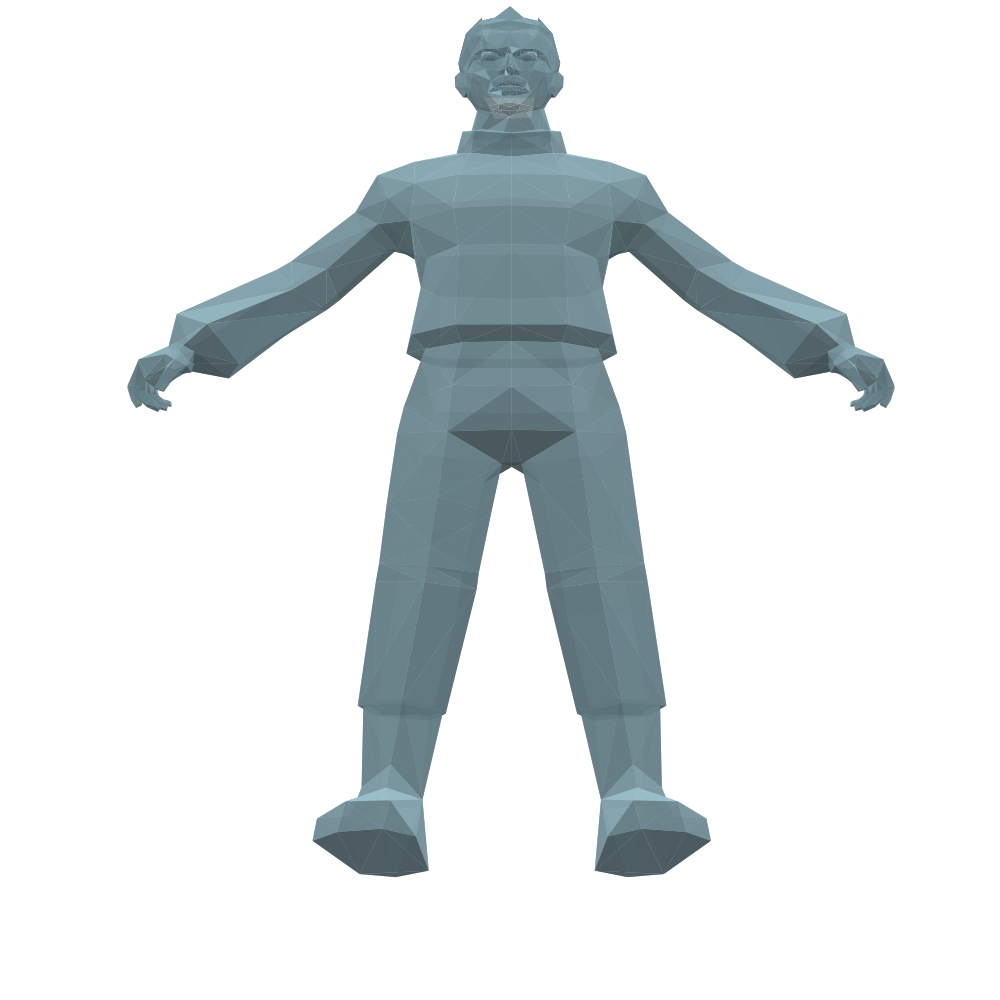}
		\label{fig:3d_mesh}
	\end{subfigure}
	\hfill
	\begin{subfigure}[b]{0.2\textwidth}
		\centering
		\includegraphics[trim=0.8cm 1.2cm 1.2cm 1.5cm, clip, width=\textwidth]{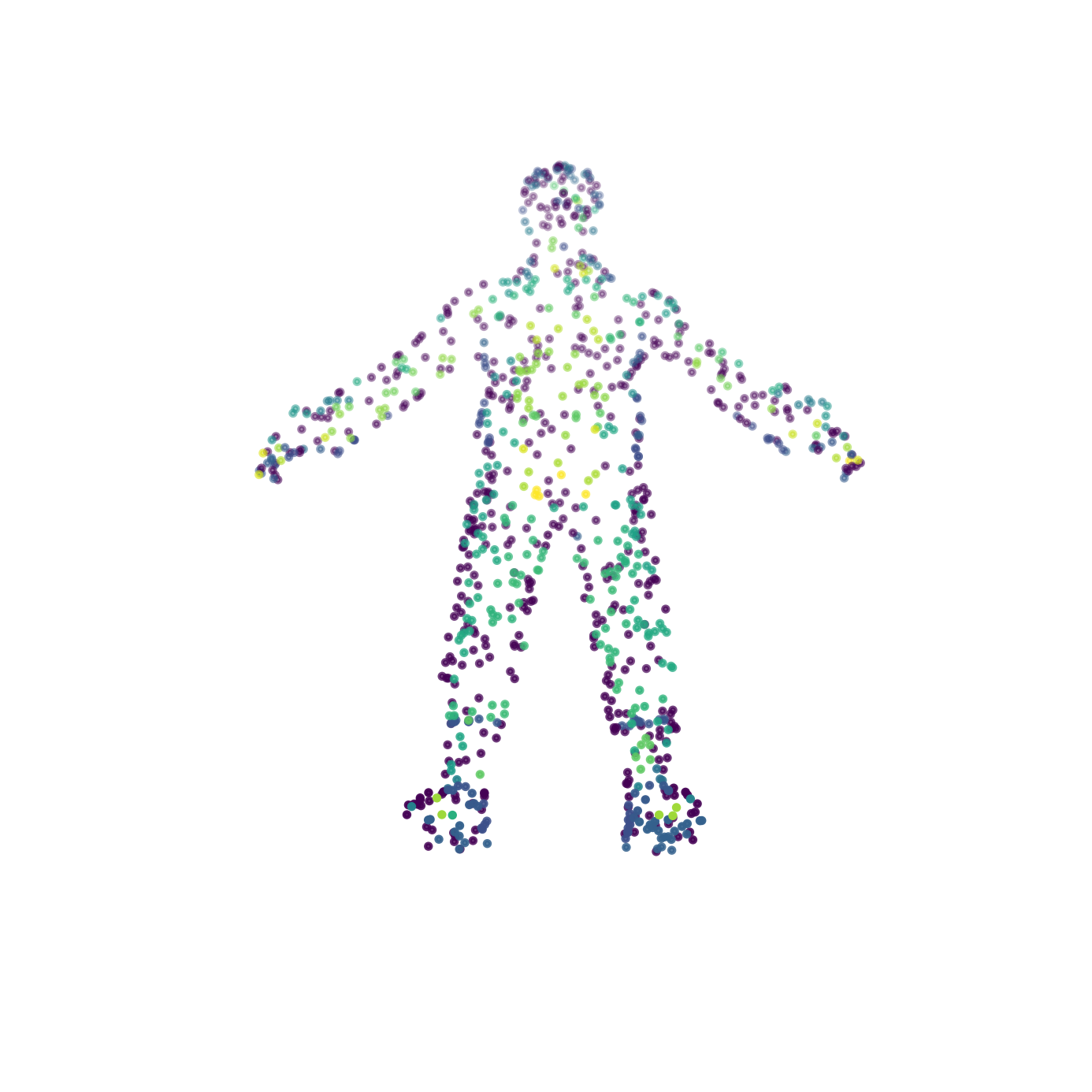}
		\label{fig:3d_400}
	\end{subfigure}
    \vspace{-8mm}
	\caption{Point cloud experimental setup: ModelNet40 mesh $\TT$ (left) and its subsample $\ST \sample{} \TT$ (right).}
	\label{fig:pointcloud_mesh}
\end{figure}

As a relaxation of a continuous tokenset $\TT$, we generate one high-resolution tokenset $\ST^*$: for a graph, we sample $5000$ tokens from it and for a point cloud, we sample 10000 tokens. Then, for every target size, we sample $100$ tokensets with $n$ tokens from the continuous one: $\ST_{n, 1}, \ldots, \ST_{n, 100}$. We consider $n \in \{800, 900, \ldots, 2000\}$ for graphs and $n \in \{2800, 3000, \ldots, 5000\}$ for point clouds. We treat each tokenset $\ST_{n, i}$ as a separate dataset, and for each of them, we train a Transformer $\GAT$ that \emph{maximizes the output error} $\left\|  \GAT(\ST^*) - \GAT\left(\ST_{n, i}\right) \right\|_{2}$ to mimic the worst case in the hypothesis class. 

\paragraph{Point Clouds.} As a continuous tokenset, we use a person mesh from the ModelNet40 dataset \citep{wu20153d} (see \Cref{fig:pointcloud_mesh}). We use surface normal vectors as point attributes and choose distance vectors as positional encoding for two points $x, y \in \Real^3$: $\PE(x, y) = x - y$. As the points are uniformly sampled from a mesh, which is a two-dimensional manifold $\chi$, we have that $D_\chi = 2$ and our theory yields $\E \left[\sup_{\GAT \in \hypothesis} \left\|  \GAT(\TT) - \GAT\left(\ST_{n, i}\right) \right\|_{2} \right] = \tilde O(n^{-\frac{1}{4}})$. 
The experiment in \Cref{fig:pointclouds} is consistent with this result.

\paragraph{Graphs.} 
We consider a smooth two-block model defined via a continuous graphon. Specifically, let a latent space $\chi = [0, 1]$ be associated with the uniform measure, and define a graphon $W \colon \chi \times \chi \to [0, 1]$ as  
\begin{align}\label{eq:cont-2-block-graphon}
W(x, y) = \left( \frac{\sin(2\pi x) \cdot \sin(2\pi y) + 1}{2} \right)^5 \cdot p + q,
\end{align}
with $p = 1$ and $q = 10^{-3}$. This choice creates a soft two-block structure, where edge probabilities are high within blocks and decay smoothly across them (see \Cref{fig:graphon_graphs}). As the high-resolution tokenset $\ST^*$, we use a fully-connected weighted graph of 5000 vertices $\{ x_1, \ldots, x_{5000} \}$, where each $x_i$ is sampled uniformly from [0,1]. The edge weights are given by $w(e_{ij}) = W(x_i, x_j)$. For each low-resolution tokenset $\ST_{n, i}$, for $n \in \{800, 900, \ldots, 2000\}$, 
we construct an unweighted graph 
where edges $e_{ij}$ are independently sampled from a Bernoulli distribution, i.e., $e_{ij} = e_{ji} \sim \text{Bernoulli}(W(x_i, x_j))$. Each node $x_i$ receives a feature vector $\nf(x_i) = [x_i, 1 - x_i]$.

	\begin{figure}[t]
		\centering
		\begin{subfigure}[b]{0.2\textwidth}
			\centering
			\includegraphics[trim=0.6cm 0.6cm 0.1cm 0cm, clip, width=0.99\textwidth]{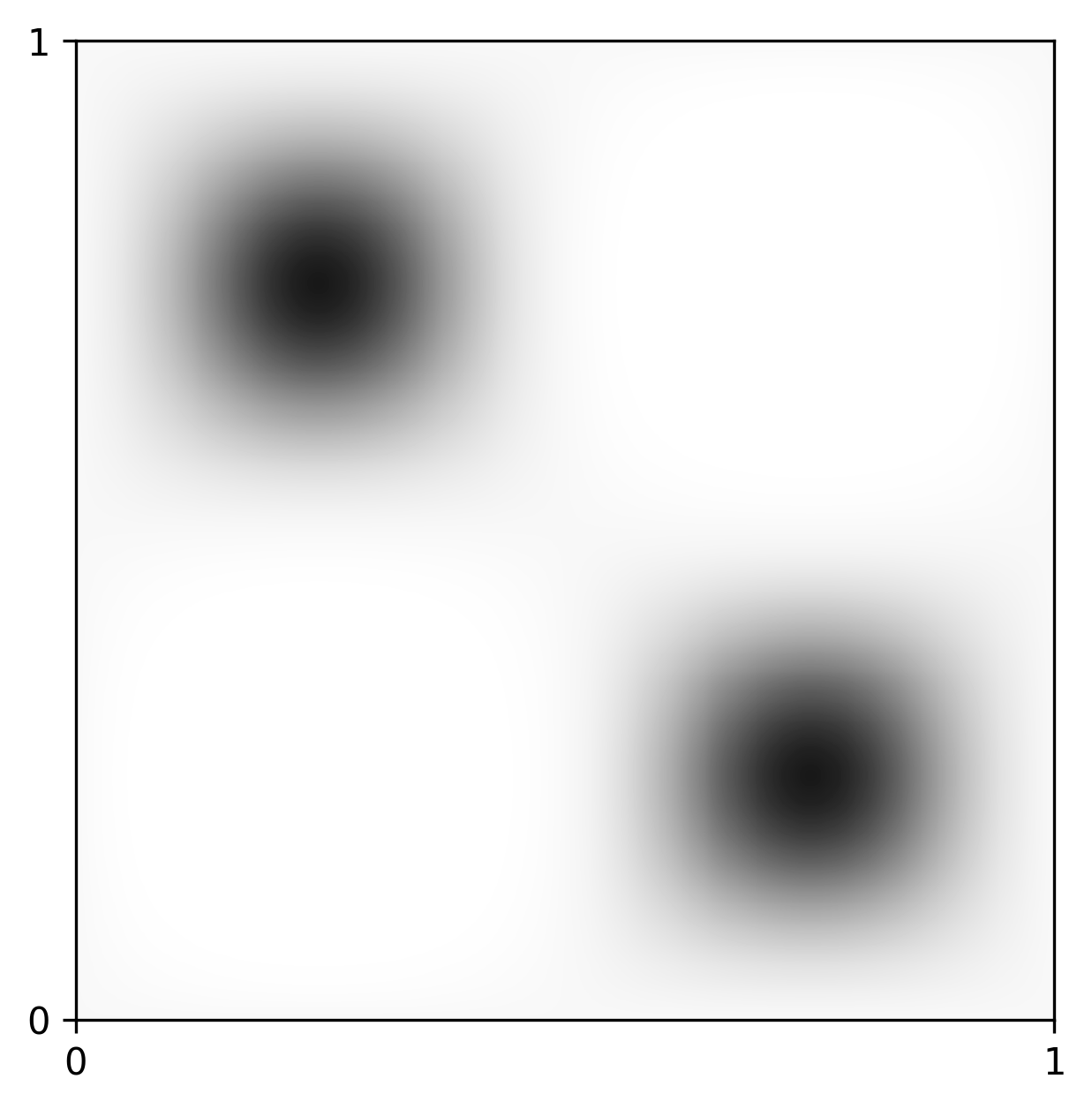}
			\label{fig:graph_20000}
		\end{subfigure}
		\hfill
		\begin{subfigure}[b]{0.2\textwidth}
			\centering
			\includegraphics[width=\textwidth]{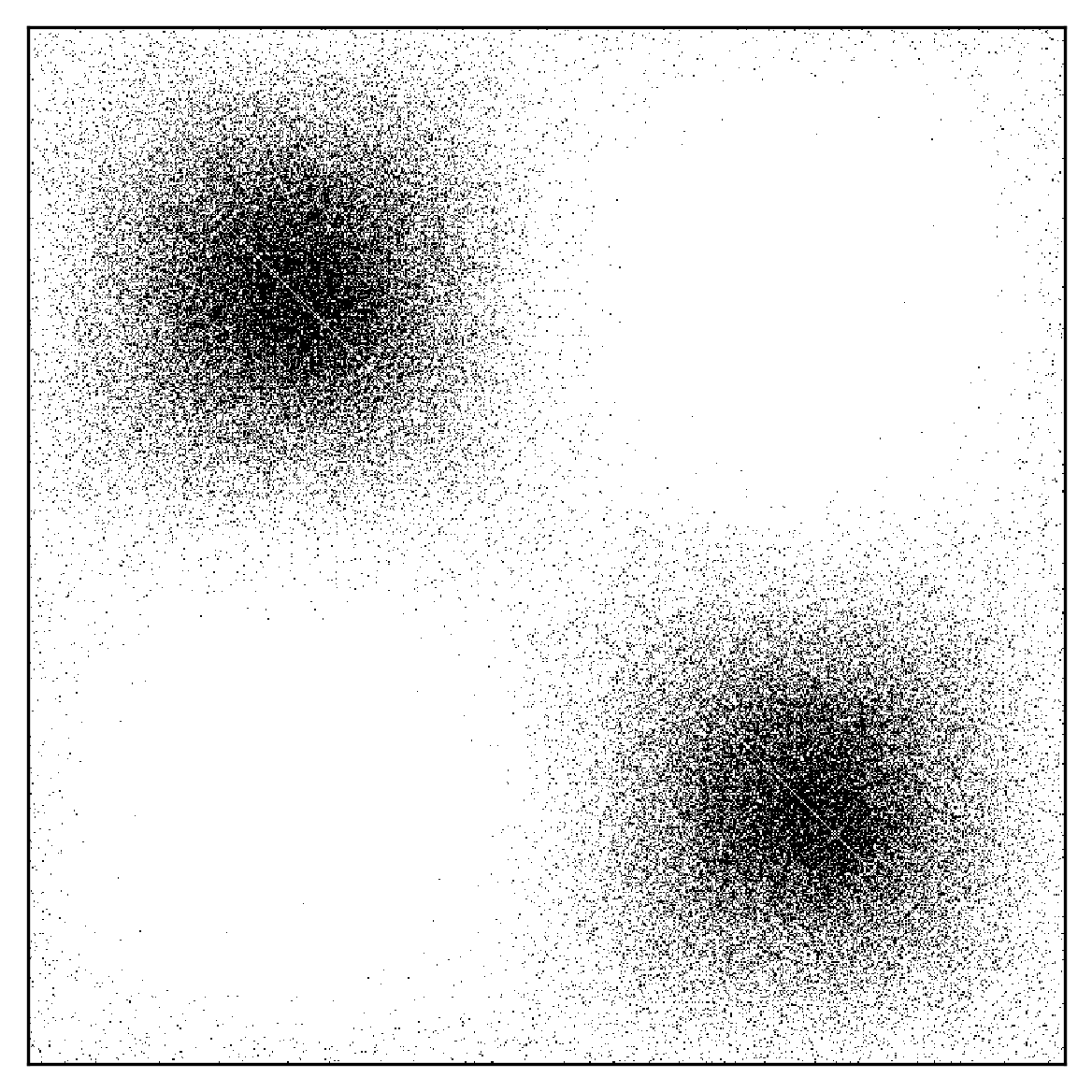}
			\label{fig:graph_400}
		\end{subfigure}
\vspace{-6mm}
	\caption{Graph experimental setup: continuous two-block graphon $\TT$  (left) and its corresponding discrete graph subsample $\ST \sample{} \TT$ (right).}
	\label{fig:graphon_graphs}
\end{figure}



For positional encoding, we use a RPE based on the powers of the random walk transition matrix $\mP$ of the sampled graph $G$. Specifically, we adopt $\PE_G(x_i, x_j) = n [\mP^3]_{ij}$, as described in \Cref{prop:stable-rpe-graphon}. Note that this encoding is $\frac{1}{2}$-stable under the graphon model by \Cref{prop:stable-rpe-graphon}. Since our latent space $\chi$ is 1-dimensional, by \Cref{lem:expected-error-short} we have that $D_\chi = 1$, hence the expected worst-case output difference satisfies
$
\E \left[\sup_{\GAT \in \hypothesis} \left\|  \GAT(\TT) - \GAT\left(\ST_{n, i}\right) \right\|_{2} \right] = \tilde{O}(n^{-\nicefrac{1}{3}}),
$
which aligns with our experimental observations shown in \Cref{fig:graphs}.
\begin{figure}[h]
    \centering
    \begin{subfigure}[b]{0.42\textwidth}
        \centering
		\includegraphics[width=\linewidth]{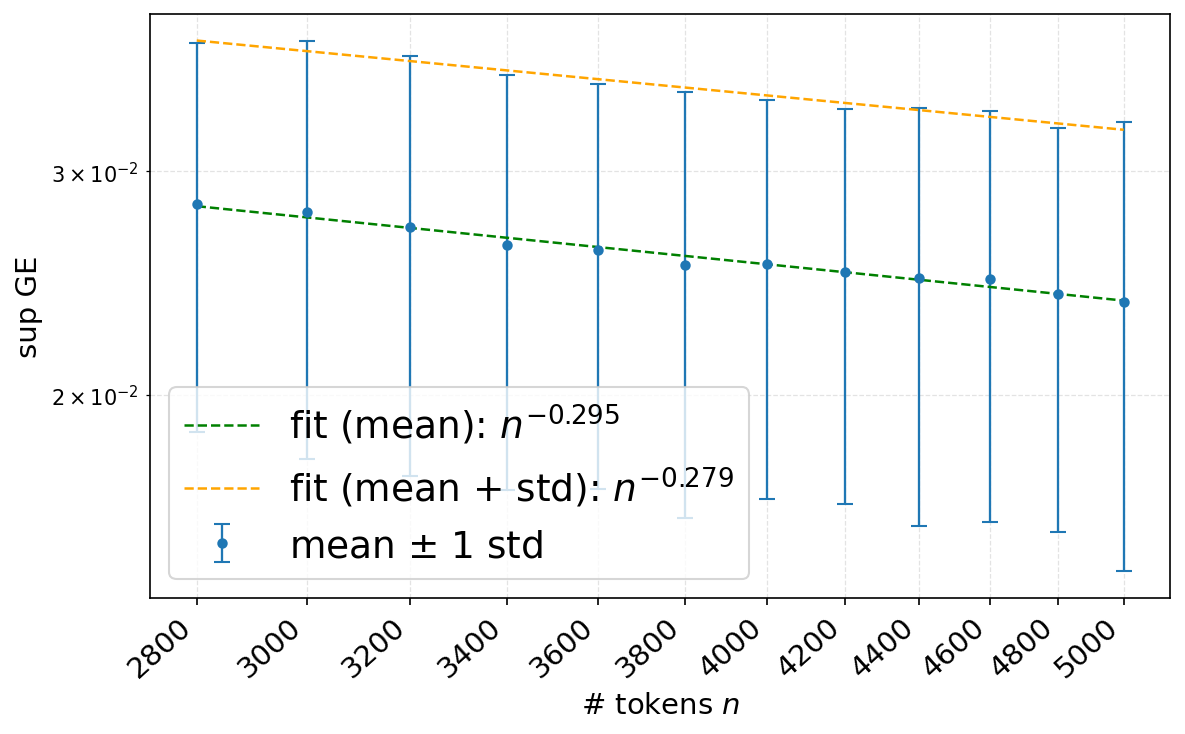}
        \vspace{-6mm}
        \caption{Point Clouds}   
        \label{fig:pointclouds}
    \end{subfigure}
    \hfill
    \begin{subfigure}[b]{0.42\textwidth}
        \centering
        \includegraphics[width=\linewidth]{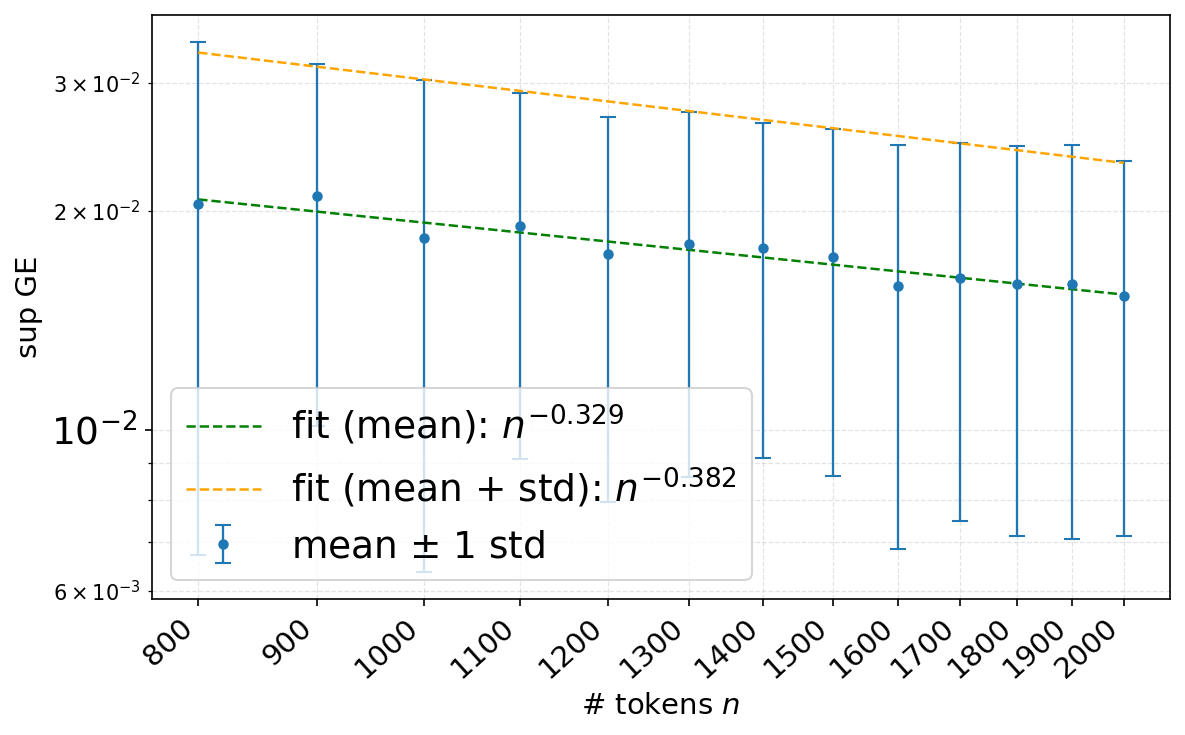}
         \vspace{-6mm}
        \caption{Graphs}
        \label{fig:graphs}
    \end{subfigure}
	\caption{
        Generalization error vs. the number of tokens $n$ on a log-log scale. Each point shows mean error $+$ one standard deviation over 100 runs. The x-axis is the number of tokens $n$, and the y-axis is the generalization error $\sup_{\GAT \in \hypothesis} \left\|  \GAT(\ST) - \GAT\left(\TT\right) \right\|_{2}$ for $\ST \sample{n} \TT$. \textbf{(a)} Point clouds: slopes $O(n^{-0.29})$ (mean) and $O(n^{-0.28})$ (mean$+$std). \textbf{(b)} Graphs: slopes $O(n^{-0.38})$ (mean) and $O(n^{-0.33})$ (mean$+$std).
    }
    \label{fig:combined_ge}
\end{figure} 
For more details on the experiments and the data generation procedure see \Cref{apx:experiments:worst-case}.

\subsection{Importance of Stable RPE} 
\label{sec:experiments:classification}

We also show the importance of RPE stability for generalization of Transformers. To do this, we consider a regression task using a similar lightweight Transformer as in the previous section (\Cref{sec:experiments:worst-case}). 

We consider two graphon models $W_0$ and $W_1$ over $\chi = [0, 1]$ with uniform measure over it. The first graphon, $W_0(x, y)$, defines a stochastic block model: $W_0(x, y) = 0.9$ if $x, y$ are in the same block: $x, y \in [0, 0.5]$ or $x, y \in [0.5, 1]$ and $10^{-3}$ otherwise. The second graphon, $W_1(x, y) = 0.3$, defines an Erdős–Rényi model. Labels are $y_0 = 0$ and $y_1 = 1$, corresponding to $W_0$ and $W_1$ respectively. Each graph $G$ in the dataset is sampled by drawing $n$ vertices $x_1, \dots, x_n \sim \text{Uniform}[0, 1]$ and sampling edges as $e_{ij} = e_{ji} \sim \text{Bernoulli}(W(x_i, x_j))$. Given $m$ tokensets, we form a dataset $\Omega = \{ (G_i, y_i) \}_{i=1}^m$. Each node $x_i$ receives a feature vector $\nf(x_i) = [1, 0]$ if $x_i \in [0, 0.5]$ and $\nf(x_i) = [0, 1]$ if $x_i \in [0.5, 1]$.

For evaluation, we construct a high-resolution test set $\Omega^*$ of 100 tokensets with 5000 tokens each. 
For training, we vary the token counts per tokenset $n \in {400, 500, \ldots 1600}$, and generate datasets $\Omega_n$ with 1000 tokensets each. Each model $\GAT$ is trained to minimize the cross-entropy loss
$\loss(\Theta(G), y) = - y \log [\Theta(G)]_1 - (1 - y) \log [\Theta(G)]_0$.
We compute the generalization error as:
$
\epsilon = \left| R_{\text{test}} - R_{\text{train}} \right|
$
where $R_{\text{test}} = \frac{1}{|\Omega^*|} \sum_{(G, y) \in \Omega^*} \loss(\GAT(G), y)$ and $R_{\text{train}} = \frac{1}{|\Omega_n|} \sum_{(G, y) \in \Omega_n} \loss(\GAT(G), y)$.

We compare two RPE schemes: (a) $\PE(x_i, x_j) = \text{shortest-path-distance}(x_i, x_j)$; (b) $\PE(x_i, x_j) = n [\mP^3]_{ij}$ where $\mP$ is the probability transition matrix of the corresponding graph (as defined in \Cref{prop:stable-rpe-graphon}). \Cref{fig:rpe-comparison} shows that RPE in (b) yields much better generalization, especially when the training dataset has significantly fewer tokens per tokenset compared to the test dataset.

\begin{figure}[t]
    \centering
    \includegraphics[width=\linewidth]{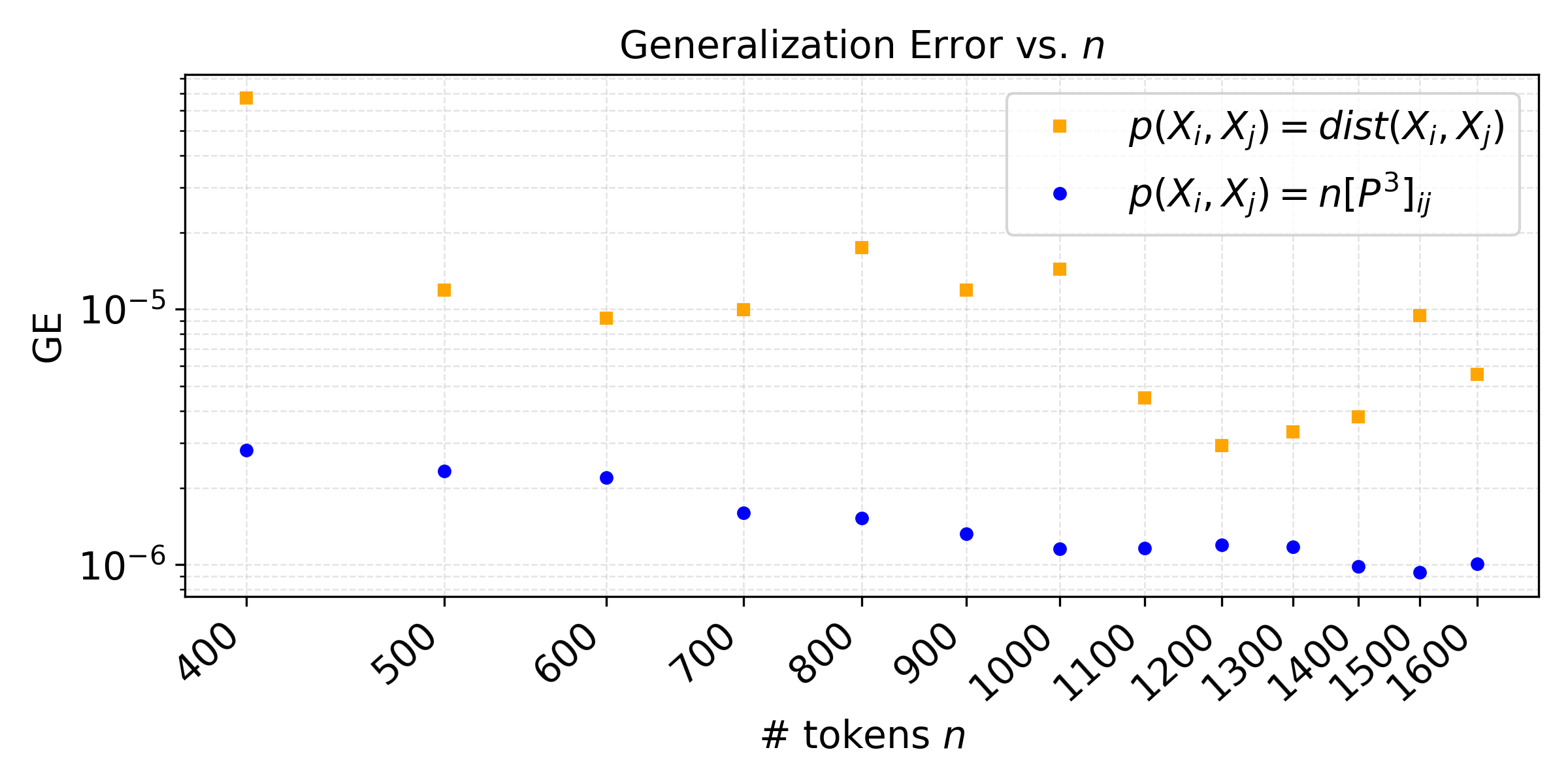}
    \vspace{-6mm}
    \caption{Generalization error (log scale) vs. the number of tokens $n$ for two RPE methods. Transition-matrix-based RPE (blue) generalizes better than shortest-path-based RPE (orange). 
	}
    \vspace{-2mm}
    \label{fig:rpe-comparison}
\end{figure}

For more details on the experiments and the data generation procedure see \Cref{apx:experiments:classification}.

\section{Discussion and Outlook}

In this work, we have shown both theoretically and empirically how the generalization properties of Transformers depend on the size of the tokenset under some assumptions on the data distribution. Specifically, we derived a generalization bound for any Transformer $\GAT$ belonging to a hypothesis class $\hypothesis$ (\Cref{def:hypothesis-class}), demonstrating that the generalization error decreases at a rate of $\tilde{O}\left(n^{-\nicefrac{1}{(D_\chi + 2)}} + n^{-\PEstable}\right)$ under mild regularity assumptions on the data distribution, where $D_\chi$ captures the regularity of the underlying space and $\PEstable$ is the stability parameter of the RPE method. To validate this result, we conducted experiments (\Cref{sec:experiments:worst-case}) on two representative data types: point clouds and graphs, and measured the worst-case output error of Transformers: $\sup_{\GAT \in \hypothesis} \left\|  \GAT(\TT) - \GAT\left(\ST\right) \right\|_{2}$ for a discrete tokenset $\ST$ with $n$ tokens sampled from the continuous one $\TT$. In both cases, the observed convergence rates closely match our theoretical predictions, supporting the view that data resolution (tokenset size) plays a key role in the generalization ability of token-based models. We also illustrated the importance of stable relative positional encoding  methods for better generalization in our experiments (\Cref{sec:experiments:classification}).

An important direction for future research is to investigate whether similar bounds can be established for more complex data modalities with causal relationships, such as natural language or images. Adapting our theoretical framework to account for such structured and non-i.i.d. data could yield valuable insights into the generalization behavior of Transformers in real-world applications.

Another direction is to further investigate the stability properties of positional encoding methods used in Transformers 
(\Cref{def:stable-positional-encoding}). 
Such an investigation could clarify their influence on model generalization.

\newpage 
\section*{Acknowledgement}
The authors are supported by NSF awards IIS-2428777 and IIS-2239565; the JP-Morgan Chase Faculty Award, the NVIDIA Academic Grant Program, the NAIRR Award, and the IDEaS Cyberinfrastructure Awards.

\bibliographystyle{plainnat}
\bibliography{ref}

\appendix

\renewcommand{\thesection}{\Alph{section}}
\setcounter{theorem}{0}
\renewcommand{\thetheorem}{\thesection.\arabic{theorem}}
\renewcommand{\thelemma}{\thesection.\arabic{lemma}}
\renewcommand{\thecorollary}{\thesection.\arabic{corollary}}
\renewcommand{\theproposition}{\thesection.\arabic{proposition}}
\renewcommand{\thedefinition}{\thesection.\arabic{definition}}

\onecolumn

\paragraph{Overview.} The appendix is organized as follows:
\begin{itemize}
	\item We introduce the preliminaries (\Cref{apx:preliminaries}), such as key definitions from metric geometry and measure theory, along with classical concentration inequalities and approximation results used in our proofs.
	\item In \Cref{apx:example}, we discuss the importance of the regularity assumption on the probability measure $\mu$ of the continuous tokenset $\TT = \ttdef$ (\Cref{asm:regularity-of-mu}).
	\item \Cref{apx:deferred-proof-of-the-main-results} contains detailed proofs of our main results: \Cref{lem:expected-error-short} and \Cref{thm:main}.
	
	\begin{itemize}
		\item In \Cref{apx:one-layer}, we analyze the output of a single Transformer layer, specifically, how a perturbation error of the inputs affects this single-layer model and how the local Lipschitzness of the hidden state function propagates through the layers.
		\item In \Cref{apx:uniform-convergence}, we provide the full proof for the error bound stated in \Cref{lem:expected-error-short}. We show that given that the empirical measure of the discrete tokenset $\ST (\sample{n} \TT)$ is close to the underlying measure $\mu$ of the continuous tokenset $\TT$, \emph{all} the intermediate hidden states of the Transformer $\GAT \in \hypothesis$ are close to the intermediate hidden states of the continuous tokenset $\TT$. We prove that this event happens with high probability, assuming $n$ is large enough, hence giving us a bound for the expected error.
		\item In \Cref{apx:generalization-bound}, we show how the output error bound yields the generalization bound \Cref{thm:main} for the classification task.
	\end{itemize}
	\item In \Cref{apx:stable-positional-encoding}, we prove that positional encodings based on the random-walk transition matrix are stable on graphs sampled from graphons.
	\item In \Cref{apx:experiments}, we provide more details on the experiment section in the main text. 
\end{itemize}

\section{Preliminaries} \label{apx:preliminaries}
In this section, we provide the necessary background material and preliminary results that support our main theoretical contributions. We begin with key definitions from metric geometry and measure theory that are essential for understanding our framework, followed by classical concentration inequalities and approximation results that we employ in our proofs.

\subsection{Definitions} \label{sec:preliminaries:definitions}

\begin{definition}[Probability–metric space]
	A \emph{probability–metric space} is a triple \((\chi,\dist,\mu)\) such that
	\begin{itemize}
	  \item \((\chi,\dist, \mu)\) is a (non-empty) metric space, i.e.\ \(\dist:\chi\times \chi\to[0,\infty)\) satisfies the usual axioms;
	  \item \(\mu\) is a probability measure on the Borel \(\sigma\)-algebra \(\mathcal{B}(\chi)\):
			\begin{gather*}
			  \mu:\mathcal{B}(\chi)\longrightarrow[0,1], \quad
			  \mu(\chi)=1, \\
			  \mu\Bigl(\,\dot\bigcup_{i=1}^{\infty}B_i\Bigr)=\sum_{i=1}^{\infty}\mu(B_i)
			  \ \text{for disjoint } \{B_i\}\subseteq\mathcal{B}(\chi).
			\end{gather*}
	\end{itemize}
	\end{definition}

\begin{definition}[Local Lipschitzness]
	We will use the following notation for the local Lipschitzness of a function $f$. Given $r \geq 0$ we define:
	\begin{align*}
		\LocLip{f}{r} := \max_{x, y \in \chi, \|x - y\| \leq r} \frac{\left\| f(x) - f(y) \right\|}{\|x - y\|}
	\end{align*}
\end{definition}

\begin{definition}[Ahlfors–\(Q\)-Regularity] \label{def:ahlfors-q-regularity}
	Let \((X,\dist,\mu)\) be a metric measure space. We say that \((X,d,\mu)\) is \textbf{Ahlfors–\(Q\)-regular} (or simply \emph{\(Q\)-regular}) if there exist positive constants \(C \ge 1\) and \(r_0 > 0\) such that for every \(x \in X\) and all radii \(0 < r < r_0\), we have
	\[
	C^{-1}\,r^Q 
	\;\le\; \mu\bigl(B(x,r)\bigr) 
	\;\le\; C\,r^Q,
	\]
	where \(B(x,r)\) denotes the open ball of center \(x\) and radius \(r\). Equivalently, one says \(\mu\) has ``Ahlfors–\(Q\)-regular growth'' if its measure of small metric balls is comparable to \(r^Q\).
\end{definition}

\subsection{Previous Results} \label{apx:previous-results}
In this section we present previous results that we will be used in our derivations.

\begin{theorem}[Hoeffding's inequality]\label{thm:hoeffding}

	Let \( X_1, X_2, \ldots, X_n \) be independent random variables such that for each \( i \),  
	\( X_i \in [a_i, b_i] \) almost surely. Let  
	\[
	\bar{X}_n = \frac{1}{n} \sum_{i=1}^n X_i
	\quad \text{and} \quad
	\mu = \mathbb{E}[\bar{X}_n].
	\]
	Then for any \( \tau > 0 \),
	\[
	\mathbb{P} \left( |\bar{X}_n - \mu| \geq \tau \right)
	\leq 2 \exp\left( - \frac{2 n^2 \tau^2}{\sum_{i=1}^n (b_i - a_i)^2} \right).
	\]
		
\end{theorem}

\begin{theorem}[Bernstein's inequality]\label{thm:bernstein}
	Let \( X_1, X_2, \ldots, X_n \) be independent, zero-mean random variables such that  
	\( |X_i| \leq M \) almost surely for all \( i \), and define  
	\[
	S_n = \sum_{i=1}^n X_i, \quad \sigma^2 = \sum_{i=1}^n \mathbb{E}[X_i^2].
	\]
	Then for any \( \tau > 0 \),
	\[
	\mathbb{P}(S_n \geq \tau) \leq \exp\left( -\frac{\tau^2}{2\sigma^2 + \frac{2M\tau}{3}} \right).
	\]
\end{theorem}

\begin{lemma}[Bretagnolle-Huber-Carol inequality]\label{thm:breteganolle-huber-carol}
	If the random vector $\left(m_1, \ldots m_{\Gamma}\right)$ is multinomially distributed with parameters $m$ and $\gamma_1, \ldots, \gamma_{\Gamma}$, then

	$$
	\mathbb{P}\left(\sum_{i=1}^{\Gamma}\left|m_i-m \gamma_i\right| \geq 2 \sqrt{m} \lambda\right) \leq 2^{\Gamma} \exp \left(-2 \lambda^2\right)
	$$

	for any $\lambda>0$.
\end{lemma}

\begin{lemma}[Tail bound]\label{lem:tail-bound}
	For non-negative random variable $X$ we have that:
	\begin{align*}
		\E[X] = \int_0^\infty \Pr[X \geq t] dt
	\end{align*}
\end{lemma}

\begin{corollary}\label{cor:expectation-of-subgaussian}
 For non-negative random variable $X$ such that for all $\tau \geq 0$ there are constants $A_1, A_2 \geq 1$ and $B_1, B_2 > 0$ such that:
 \begin{align*}
 	\Pr[X \geq \tau] \leq A_1 e^{-\tau^2/{B_1} } + A_2 e^{-\tau/{B_2} }
 \end{align*}
 
 \begin{align*}
 	\E[X] \leq 
	 \sqrt{2B_1}(\sqrt{\log A_1} + 1) + {2B_2}(\log A_2 + 1)
 \end{align*}
\end{corollary}

\begin{proof}
	Note that: 
	\begin{align*}
		\int_{\tau=0}^\infty \min\{ A_1 e^{-\tau^2/{B_1} } , 1 \} d\tau 
		&= 
		\int_{\tau=0}^{\sqrt{2B_1 \log A_1}} \max\{ A_1 e^{-\tau^2/{B_1} } , 1 \} d\tau \\
		&+ 
		\int_{\tau=\sqrt{2B_1 \log A_1}}^\infty \max\{ A_1 e^{-\tau^2/{B_1} } , 1 \} d\tau \\
		&\leq \sqrt{2B_1 \log A_1} + \int_{\tau=\sqrt{2B_1 \log A_1}}^\infty e^{-\tau^2/(2B_1) } d\tau \\
		&\leq \sqrt{2B_1}(\sqrt{\log A_1} + \sqrt{\pi})
	\end{align*}

	And similarly:
	\begin{align*}
		\int_{\tau=0}^\infty \min\{ A_2 e^{-\tau/{B_2} } , 1 \} d\tau 
		&= 
		\int_{\tau=0}^{2B_2 \log A_2} \max\{ A_2 e^{-\tau/{B_2} } , 1 \} d\tau \\
		&+ 
		\int_{\tau=2B_2 \log A_2}^\infty \max\{ A_2 e^{-\tau/{B_2} } , 1 \} d\tau \\
		&\leq 2B_2 \log A_2 + \int_{\tau=2B_2 \log A_2}^\infty e^{-\tau/(2B_2) } d\tau \\
		&\leq 2B_2(\log A_2 + 1)
	\end{align*}
\end{proof}

\begin{theorem}[Jensen's inequality]\label{thm:jensen-inequality}
	For a convex function $f$ and a random variable $X$ we have that:
	\begin{align*}
		\E[f(X)] \geq f(\E[X])
	\end{align*}
\end{theorem}

\begin{theorem}[Generalized Bernoulli's Inequality with Multiple Bases]
	\label{thm:gen-bern}
	Let $x_1, x_2, \ldots, x_r$ be real numbers, all greater than $-1$, and all having the same sign.  
	Then
	\[
	(1 + x_1)(1 + x_2)\cdots(1 + x_r) \ge 1 + x_1 + x_2 + \cdots + x_r.
	\]
	\end{theorem}

\section{Discussion on the Regularity Assumption} \label{apx:example}

Previously we made as assumption about the probability measure $\mu$ on the underlying space $\chi$ being "regular", meaning not having any "holes"or regions with small measure (see \Cref{def:ahlfors-q-regularity}). This assumption is stronger than the one used in \citet{maskey2024generalization}. Here we provide an example to justify this stronger assumption. We consider a simple case where the token space is just two points but there probability of one of them being sampled is very small. We show that to approximate this tokenset well one should sample exponentially many tokens with respect to the Lipschitz constant of the logit function, even though the covering number of the continuous tokenset is 2.

\begin{theorem}
	Consider the following probability space and the corresponding tokenset:
	\begin{align*}
		\chi = \{a, b\}, a = -1, b = 1, \quad \mu: \Pr[a] = 1 - e^{-L}, \quad \Pr[b] = e^{-L} \text{ for some $L > 10$}
	\end{align*}
	Consider Transformer defined with the following logit and value functions: 
	\begin{align*}
		\logit(x, y) = L(x - y) \quad \text{ and }  \quad 
		\msg(y) = y 
	\end{align*}
	For $n = e^L$ tokens the following bound holds:
	\begin{align*}
		\Pr_{\ST \sample{n} \TT}\left[\left\| \Theta({\TT}) - \Theta({\ST}) \right\| \geq \frac{1}{4} \right] \geq \frac{1}{e}
	\end{align*}
\end{theorem}

\begin{proof}
	First, let us compute the output of the defined Transformer on the continuous tokenset $\TT$:
\begin{align*}
	\logit(a, a) = 0, \quad \logit(b, a) = \logit(a, b) = 2L, \quad \logit(b, b) = 0 \\
	\hst^{(1)}_{\TT}(a) = 
	\frac{e^0(1 - e^{-L})a + e^{2L} e^{-L} b}{e^0(1 - e^{-L}) + e^{2L} \cdot e^{-L}} = 
	\frac{-1 + e^{-L} + e^{L}}{1 - e^{-L} + e^{L}} \Rightarrow | \hst^{(1)}_{\TT}(a) + 1 | \leq \frac{1}{e^L} \\
	\hst^{(1)}_{\TT}(b) = 
	\frac{e^{2L}(1 - e^{-L})a + e^{0} e^{-L} b}{e^{2L}(1 - e^{-L}) + e^{0}  e^{-L}} = 
	\frac{-e^{2L} + e^{L} + e^{-L}}{e^{2L} - e^{L} + e^{-L}} \Rightarrow | \hst^{(1)}_{\TT}(b) - 1 | \leq \frac{1}{e^L} \Rightarrow \\
	\Theta({\TT}) = h^{(1)}_{\TT}(a) \cdot (1 - e^{-L}) + h^{(1)}_{\TT}(b) \cdot e^{-L} \Rightarrow | \Theta({\TT}) - 1 | \leq \frac{5}{e^L}
\end{align*}

Consider sampling tokenset $\ST = (\sampledset) \sample{n} \TT$ with $n = e^L$ tokens from this distribution: $\sampledset = \{\rvx_1, \rvx_2, \ldots, \rvx_{n}\}$. Then we have that with constant probability, we will have only $x_1$ in the sampled set:
\begin{align*}
	\Pr[\forall i \in [n], \rvx_i = a] = \left(1 - e^{-L}\right)^{e^L} \geq \frac{1}{2e}
\end{align*}
 In case the sampled set contains only token $a$, the output of the Transformer is:
 \begin{align*}
	\forall i \in [n], \hst^{(1)}_{\TT}(\rvx_i) = a \Rightarrow \Theta({\TT}) = a = -1
\end{align*}
Hence, the output error of the Transformer is:
\begin{align*}
	\left\| \Theta({\TT}) - \Theta({\ST}) \right\| \geq 1
\end{align*}

\end{proof}

\section{Deferred Proofs of the Main Result} \label{apx:deferred-proof-of-the-main-results}

In this section we provide the proofs of the main results \Cref{lem:expected-error-short} and \Cref{thm:main}.

\subsection{Analysis for One Layer of Transformer} \label{apx:one-layer}

In this section we prove a few results about properties of the one layer of Transformer. We reuse the notation from \Cref{def:transformer}. First, we compare the output of the Transformer for two close points.
\begin{lemma}[Comparing two points]\label{lemma:att:comparing-two-close-points}
	Given two points $x, y \in \chi$ such that $\| x - y \| = r$ and a Locally Lipschitz function $\logit: \chi^2 \to \Real$ we have that for any point $z \in \chi$:
	\begin{gather*}
		\frac{\exp(\logit(x, z))}{\exp(\logit(y, z))} 
		\in \left[ \exp(-\LocLip{\logit}{r} r), \exp(\LocLip{\logit}{r} r) \right] 
		\\
		\frac{\Att(x, z)}{\Att(y, z)} 
		\in \left[ \exp(-2\LocLip{\logit}{r} r), \exp(2\LocLip{\logit}{r} r) \right]
	\end{gather*}
\end{lemma}

\begin{proof}
	For any $z \in \chi$:
	\begin{gather*}
		\left | \logit(x, z) - \logit(y, z) \right | \leq \LocLip{\logit}{r} r \\
		\frac{\exp(\logit(x, z))}{\exp(\logit(y, z))} \in \left[ \exp(-\LocLip{\logit}{r} r), \exp(\LocLip{\logit}{r} r) \right] \Rightarrow
		\\
		\frac{\int_{z \in \chi} \exp(\logit(x, z))}{\int_{z \in \chi} \exp(\logit(y, z))} \in \left[ \exp(-\LocLip{\logit}{r} r), \exp(\LocLip{\logit}{r} r) \right] \Rightarrow \\
		\frac{\Att(x, z)}{\Att(y, z)} \in \left[ \exp(-2\LocLip{\logit}{r} r), \exp(2\LocLip{\logit}{r} r) \right]
	\end{gather*}
\end{proof}

\begin{corollary}
	Assuming that we are given $x, y$ such that $r = \|x - y\|$ and $\LocLip{\logit}{r} r \leq \frac{1}{2}$:
	\begin{align*}
		\left| \frac{\Att(x, z)}{\Att(y, z)} - 1 \right| \in 4\LocLip{\logit}{r} r
	\end{align*}
\end{corollary}

\begin{proof}
	Follow from the bounds $e^{\delta} < 1 + 2\delta$ and $e^{-\delta} > 1 - 2\delta$ for $\delta \in [0, 1]$.
\end{proof}

Now let us show the bound on the Local-Lipschitzness of the next hidden state function.

\begin{lemma}\label{lem:att:local-lipschitzness}
	Given $r$ such that  $(\LocLip{\hst}{r} + \Lip{\PE})r \leq \frac{1}{2\Lip{\att}}$ we have:
	\begin{align*}
		\LocLip{\next(\hst)}{r}
		&\leq
		4\Lip{\att} \Lip{\msg} \|\hst\|_{2, \infty}
		\Bigl(
			\LocLip{\hst}{r} + \Lip{\PE}
		\Bigr) r 
		\\
		\| \next(\hst) \|_{\infty} &\leq \Lip{\msg} \|\hst\|_{2, \infty}
	\end{align*}

\end{lemma}

\begin{proof}
	Suppose that $x, y \in V$ and $\|x - y\| \leq r$. Using \Cref{lemma:att:comparing-two-close-points}:
	\begin{align*}
		\left\| \next(\hst)(x) - \next(\hst)(y) \right\| 
		\leq 
		\left\| \int_{z \in \chi} \Att(x, z) \Msg(z) -  \Att(y, z) \Msg(z) d\mu(z) \right\|
		\\ 
		=
		\left\| \int_{z \in \chi} (\Att(x, z) -  \Att(y, z)) \Msg(z) d\mu(z) \right\| 
		\leq 
		4\Lip{\att} \Lip{\msg} \|\hst\|_{2, \infty}
		\Bigl(
			 \LocLip{\hst}{r} + \Lip{\PE}
		\Bigr) r
	\end{align*}
\end{proof}

	\begin{corollary}\label{lemma:local-lipschitzness-layers}
		Given $\Const_I > 1$ and $T, r$ such that  
		\begin{align*}
			\left(
				\constI^\nlayers \Lip{\nf}
				+
				\frac{\constI^{\nlayers} - 1}{\constI-1}
				\Lip{\PE}
			\right)
			r
			\leq
			\frac{1}{2\Lip{\att}}
		\end{align*}
		
		we have:
		\begin{align*}
			\forall \nlayer \in [0, T-1]: 
			\LocLip{\hst^{(\nlayer)}}{r}
			&\leq
			4^k\Lip{\att}^k \Lip{\msg}^{\frac{k(k-1)}{2}} \|\nf\|^k_{2, \infty}(  \Lip{\nf} + k \Lip{\PE} ) 
			\\
			\|\hst^{(\nlayer)} \|_{\infty} &\leq \Lip{\msg}^{\nlayer} \|\nf\|_{2, \infty}
		\end{align*}
		where 
		\begin{align*}
			\constI &= 4 (\Lip{\att} \|\msg\|_{2, \infty} + \Lip{\msg})
		\end{align*}
		
	\end{corollary}

\begin{proof}

Fist, note that by \Cref{lemma:local-lipschitzness-layers} we have that $\|\hst^{(\nlayer-1)} \|_{\infty} \leq \Lip{\msg}^{\nlayer-1} \|\nf\|_{2, \infty}$. Hence, 
\begin{align*}
	\Lip{\hst^{(\nlayer)}}
	\leq
	4\Lip{\logit} \Lip{\msg}^k \|\nf\|_{2, \infty} \left( \Lip{\hst}^{(\nlayer-1)} + \Lip{\PE} \right) 
\end{align*}

Applying this inequality recursively we get that: 
\begin{align*}
	\Lip{\hst^{(\nlayer)}}
	\leq
	4^k\Lip{\logit}^k \Lip{\msg}^{k(k-1)/2} \|\nf\|^k_{2, \infty}(  \Lip{\nf} + k \Lip{\PE} )
\end{align*}

\end{proof}

Now let us show how the perturbation error propagates from one layer to the next.

    \begin{lemma}[Perturbation Error]\label{lem:perturbation-error}
		Assuming
		\begin{align*}
			2\sup_{x} \left\|\hst(x) - \hat \hst(x)\right\| 
			+ 
			\sup_{x, y} \left\|\PE(x, y) - \hat\PE(x, y)\right\| 
			\leq 
			\tfrac{1}{2\Lip{\att}}
		\end{align*}

		we have the following bounds:
		\begin{gather*}
			\sup_{x} \left \|\next(\hat\hst)(x) - \next(\hst)(x) \right \|
			\leq \\
			8\Lip{\msg}\Lip{\att} \|\hst\|_{2, \infty}  
			\left(
				2 \sup_{x}
				\left\|\hst(x) - \hat\hst(x)\right\| 
				+
				\sup_{x, y}
				\left\|\PE(x, y) - \hat \PE(x, y)\right\| 
			\right)	
		\end{gather*}

    \end{lemma}

\begin{proof}
    \newcommand{\ER}{\Delta}
    Let
    \begin{align}
        \ER(x, y) := \|\hst(x) - \hat \hst(x)\| + \|\hst(y) - \hat \hst(y)\| + \|\PE(x, y) - \hat\PE(x, y)\|
    \end{align}
    \paragraph{Attention.}
    \begin{align*}
        |\logit(x, y) - \hat \logit(x, y)| =
        \left | \att(\hst(x), \hst(y), \PE(x, y)) - \att(\hat\hst(x), \hat\hst(y), \PE(x, y)) \right|
        \\ \leq \Lip{\att}  (\|\hst(x) - \hat \hst(x)\| + \|\hst(y) - \hat \hst(y)\| + \|\PE(x, y) - \hat\PE(x, y)\|) = \Lip{\att} \ER(x, y)
    \end{align*}
    Hence, 
    \begin{align*}
        \frac{\exp(\logit(x, y))}{\exp(\hat\logit(x, y))} 
		&\in \left[
            \exp\Bigl(-\Lip{\att} \ER(x, y) \Bigr)
            ,
            \exp\Bigl(\Lip{\att}  \ER(x, y) \Bigr)
        \right]
        \Rightarrow \\
        \frac{\hat\Att(x, y)}{\Att(x, y)} 
		&\in 
        \left[
            \exp\Bigl(-2\Lip{\att}  \ER(x, y) \Bigr)
            ,
            \exp\Bigl(2\Lip{\att} \ER(x, y) \Bigr)
        \right] 
        \Rightarrow \\
        \hat\Att(x, y) - \Att(x, y) 
		&\in 
        \left[
            \left(e^{-2\Lip{\att}  \ER(x, y)} - 1\right)\Att(x, y), 
            \left(e^{ 2\Lip{\att}  \ER(x, y)} - 1\right)\Att(x, y) 
        \right] \Rightarrow 
		\\
        \left \|\hat \Att(x, y) - \Att(x, y)\right \| 
		&\leq 
        4\Lip{\att} \Att(x, y) \ER(x, y)
    \end{align*}

    \paragraph{Value.}
    
    \begin{align*}
        \left \|\hat \Msg(y)  - \Msg(y)\right\| \leq \Lip{\msg}\|\hat \hst(y) - \hst(y) \|
    \end{align*}

    \paragraph{Aggregated message with respect to attention scores.}

    \begin{align*}
        \left \|
        \int_{y \in \chi} \hat\Att(x, y)\hat\Msg(y) -\Att(x, y)\Msg(y) d \mu(y)
        \right \|
        \\
        \leq 
        \left \|
        \int_{y \in \chi} \left(\hat\Att(x, y) -\Att(x, y)\right)\hat\Msg(y) d \mu(y)
        \right \| 
        +
        \left \|
        \int_{y \in \chi} \Att(x, y)\left(\hat\Msg(y)  - \Msg(y)\right) d \mu(y)
        \right \| 
		\\
        \leq 
		4\Lip{\att}  
        \|\msg\|_{2, \infty}
		\int_{y \in \chi} \Att(x, y) \ER(x, y) d \mu(y)
		+
		\Lip{\msg} \int_{y \in \chi} \Att(x, y) \|\hst(y) - \hat\hst(y)\| d \mu(y)
    \end{align*}

    Therefore:
    \begin{gather*}
        \left \|\next(\hat\hst)(x) - \next(\hst)(x) \right \| \leq
		\\
		2\Lip{\msg}(2\Lip{\att} \|\hst\|_{2, \infty} + 1) 
		\left(
			2 \sup_{x}
			\left\|\hst(x) - \hat\hst(x)\right\| 
			+
			\sup_{x, y}
			\left\|\PE(x, y) - \hat \PE(x, y)\right\| 
		\right)	
    \end{gather*}
\end{proof}

\subsection{Uniform Convergence} \label{apx:uniform-convergence}

In this section we show, that for a continuous \tokenset $\TT$ and a discrete \tokenset sampled from it $\ST \sample{n} \TT$, there exist an event in which \emph{all} the intermediate hidden states of \emph{any} Transformer $\GAT \in \hypothesis$ (see \Cref{def:hypothesis-class}) applied to $\ST$ are close to the intermediate hidden states of the continuous \tokenset $\TT$. Specifically, this happens when the empirical measure $\mu_{\sampledset} = \frac{1}{n} \sum_{i=1}^n \delta_{\rvx_i}$ based on the sampled tokens $\{ \rvx_i \}_{i=1}^n$ is "close" to the underlying measure $\mu$ of the continuous \tokenset $\TT = \ttdef$. Here, $\delta_{\rvx_i}$ is the Dirac measure at $\rvx_i$.

\subsubsection{Concentration of the empirical measure}

In this section we derive the concentration property of the empirical measure $\mu_{\sampledset}$ of the sampled \tokenset, meaning it is a good approximation of the underlying measure of the continuous \tokenset $\TT = \ttdef$.
To do that, we first show that the assumption \ref{asm:regularity-of-mu} yields the existence of the following covering of $\chi$:

	\begin{restatable}{lemma}{lemmacovering}\label{lemma:covering}
		Given $r \geq 0$ we assume that for some fixed constants $C_\chi, D_\chi$ there exists a covering of $\chi$ by disjoint set of regions $\regions^{(r)}$ such that for any $\region \in \regions^{(r)}$:
		\begin{align*}
			\operatorname{diam}(\region) \leq 2r, \quad \mu(\region) \geq \frac{1}{C_\chi} r^{D_\chi}
		\end{align*}

		Note, that the number of regions $|\regions^{(r)}| \leq C_\chi r^{-D_\chi}$
	\end{restatable}

\begin{proof}
	Consider a maximal covering of $\chi$ by disjoint balls of radius $r$ centered at points $\{ z_i \}$. Then we have that:
	\begin{align*}
		\bigcup_{i} B_{r}(z_i) = \chi
	\end{align*}
	For any $i \neq j$ we have that $z_i \not\in B_{r}(z_j)$. Hence, $B_{r/2}(z_i) \cap B_{r/2}(z_j) = \emptyset$.
	
	Consider the following disjoint covering $\regions^{(r)} = \{ \region_i \}$ of $\chi$:
	\begin{align*}
		\bigcup_{i} \region_i = \chi \quad \text{ and } \quad \region_i \cap \region_j = \emptyset \quad \text{ for } i \neq j \\ 
		\region_i = \{ x \in \chi \mid i = \arg\min_{j} \|x - z_j\| \}
	\end{align*}

	$\region_i$ is basically a Voronoi cell of $z_i$ with respect to the set of points $\{ z_j \}$.
	Note that by construction $B_{r/2}(z_i) \subset \region_i$, hence $\mu(\region_i) \geq \mu(B_{r/2}(z_i)) \geq \frac{1}{C_\chi} r^{D_\chi}$. Moreover, $\operatorname{diam}(\region_i) \leq 2r$. Therefore, $\regions^{(r)}$ satisfies the conditions of the lemma.

\end{proof}

Now let us define an event of empirical measure concentration and bound the probability with which it holds.

\begin{definition}\label{event:token-sampling}
	Given probability-metric space $(\chi, \dist, \mu)$ and a sample $\sampledset = \{\rvx_i \}_{i=1}^{\nsamples} \sim \mu^\nsamples$.  Consider a disjoint covering $\regions^{(\nicefrac{1}{D_\chi + 2})}$. An event 
	$\eventtTS{(\chi, \dist, \mu)}{\sampledset}{\tau}$
		is such that \textbf{for all} $\region \in \regions^{(\nicefrac{1}{(D_\chi + 2)})}$:

	$$
	\left|\frac{\mu_{\sampledset}(\region)}{\mu\left(\region\right)} - 1\right| 
	\leq  
	C_\chi \nsamples^{-\frac{1}{D_\chi + 2}} \tau
	$$

	Note: when $(\chi, \dist, \mu)$ is a part of continuous \tokenset $\TT$ and $\sampledset$ corresponds to a sample from it $\ST \sample{\nsamples} \TT$, we will use the notation $\eventtTS{\TT}{\ST}{\beta}$.
\end{definition}

\begin{lemma}[Empirical Measure Concentration]\label{lemma:pr:event-ts}
	For a tokenset sampled from an $\admconsts$-admissible continuous tokenset $\TT$: $\TT \sample{n} \ST$, and $\tau > 0$ we have that:
	\begin{align*}
			\Pr[\eventtTS{\TT}{\ST}{\tau}] 
			\geq 
			1 - 2 C_\chi n \exp\left(-\frac{\tau^2}{2 + \tau \sqrt{C_\chi}} \right)
			\geq 
			1 - 2 C_\chi n \left( e^{-\tau^2/4} + e^{-\frac{\tau}{2\sqrt{C_\chi}}}\right)
	\end{align*}
\end{lemma}

\begin{proof}
	Let $r = n^{-\frac{1}{2+D_\chi}}$.
			
	Note that for any $\region \in \regions^{(r)}$ we have that $\mu_{\sampledset}(\region) = \frac{1}{n} \sum_{i} \mu_{\{\rvx_i\}}(\region)$ and $E\left[\left(\mu_{\{\rvx_i\}}(\region) - \mu(\region)\right)^2\right] = \mu(\region)(1 - \mu(\region)) \leq \mu(\region)$. By Bernstein inequality we have that for any $\region \in \regions^{(r)}$ for any $\beta > 0$:
	\begin{align*}
	\Pr\left[
		\left|\mu_{\sampledset}(\region) - \mu\left(\region\right)\right| 
		\geq  
		\frac{\beta}{\nsamples}
	\right]
	\leq
	2 \exp\left(-\frac{\beta^2/2}{n\mu(\region) + \beta/3} \right)
	\end{align*}

	Let $\beta = \tau \sqrt{n\mu(\region)}$ for some $\tau > 0$. Note, that $\sqrt{n\mu(\region)} \geq \sqrt{1/C_\chi} \cdot n^{1/(D_\chi + 2)} \geq \sqrt{1/C_\chi}$. Hence, 

	\begin{align*}
		\exp\left(-\frac{\beta^2/2}{n\mu(\region) + \beta/3} \right) 
		= 
		\exp\left(-\frac{\tau^2 n\mu(\region) / 2}{n\mu(\region) + \tau/3 \cdot \sqrt{n\mu(\region)}} \right) 
		\leq
		\exp\left(-\frac{\tau^2}{2 + \tau \sqrt{C_\chi}} \right) 
	\end{align*}

	Combining these we have that 
	$$
	\Pr\left[
		\left|\frac{\mu_{\sampledset}(\region)}{\mu\left(\region\right)} - 1\right| 
		\geq  
		\tau \sqrt{C_\chi} \cdot n^{-\frac{1}{D_\chi + 2}}
	\right]
	\leq
	2 \exp\left(-\frac{\tau^2}{2 + \tau \sqrt{C_\chi}} \right) 
	$$
	Applying the union bound over all all $|\regions^{(r)}| \leq C_\chi n^{\frac{D_\chi}{D_\chi + 2}} < C_\chi n$ we have that with probability at least $1 - 2C_\chi n \exp\left(-\frac{\tau^2}{2 + \tau \sqrt{C_\chi}} \right) $ for all $\region \in \regions^{(r)}$:
	$$
		\left|\frac{\mu_{\sampledset}(\region)}{\mu\left(\region\right)} - 1\right| 
		\leq  
		\sqrt{C_\chi} n^{-\frac{1}{D_\chi + 2}} \tau 
	$$
\end{proof}

\subsubsection{Aggregation Error: Sampling of Tokens}

In this section we look how subsampling points from a continuous \tokenset affects the aggregation of messages.

	\begin{theorem}[Aggregation Error: Sampling of Tokens] 
	\label{thm:concentration-error-full}
	Given 
	\begin{itemize}
		\item probability metric space $(\chi, \dist, \mu)$
		\item $L$ and $n$ such that $L n^{-\frac{1}{D_\chi + 2}} \leq \frac{1}{4}$
		\item sampled set $\sampledset = \left\{\rvx_1, \ldots, \rvx_\nsamples\right\} \sim \mu^\nsamples$
		\item $\tau \leq \tfrac{1}{4\sqrt{C_\chi}} n^{\nicefrac{1}{(D_\chi + 2)}}$
		\item event $\eventtTS{(\chi, \dist, \mu)}{\sampledset}{\tau}$ happens
	\end{itemize}
	
	we have that \textbf{for every} Lipschitz functions $\msg: \Real^H \to \Real, \att: \Real^H \times \Real^H \to \Real$ and Locally Lipschitz function $\hst: \chi \to \Real^H$ such that $\Lip{\att} \LocLip{\hst}{r} \leq L$ we have that:
	\begin{gather*}
	\max_{\rvx \in \sampledset}
	\left\|
		\frac{1}{\nsamples} \sum_{j=1}^\nsamples 
		\Att_{\sampledset}(\rvx, \rvx_i)  \Msg(\rvx_i) 
		-
		\int_\chi \Att_{\chi}(\rvx, x)  \Msg(x) 
		d \mu(x)
	\right\| 
	\\
	\leq 
	8 \Lip{\msg} \|\hst\|_{2, \infty}
	\left(
		\LocLip{\hst}{\nsamples^{-\nicefrac{1}{(D_\chi + 2)}}}   \Lip{\att}
		+
		\sqrt{C_\chi} \tau 
	\right) 
	\nsamples^{-\frac{1}{D_\chi + 2}}
	\end{gather*}

	where value is calculated as $\Msg(x) = \msg(h(x))$ and attention is calculated as 
	\begin{align*}
		\Att_{\sampledset}(\rvx, \rvx_i) = \frac{\exp(\att(\rvx, \rvx_i))}{\frac{1}{\nsamples}\sum_{j=1}^\nsamples \exp(\att(\rvx, \rvx_j))}
		\quad \text{and} \quad
		\Att_\chi(x, y) = \frac{\exp(\att(x, y))}{\int_{z \in \chi} \exp(\att(x, z)) d\mu(z)} 
	\end{align*}  
	\end{theorem}

\begin{proof}
	Consider disjoint covering $\regions^{(r)}$ for $r = \nsamples^{-\frac{1}{D_\chi + 2}}$ as in \Cref{lemma:pr:event-ts}. As we are given that $\eventtTS{(\chi, \dist, \mu)}{\sampledset}{\tau}$ happens and 
	$\tau \leq \tauUB$
	we have that for all $\region \in \regions^{(r)}$:
	\begin{align}\label{bound:att:emp-measure}
		\left|\frac{\mu_{\sampledset}(\region)}{\mu\left(\region\right)} - 1\right| 
		\leq 
		\underbrace{\tau \sqrt{C_\chi} n^{-\frac{1}{D_\chi + 2}}}_{=: \delta}
		\leq \frac{1}{2}
	\end{align}

    Let us consider one region $\region$ and prove a series of bounds for it. Let $L_{l} = \LocLip{\logit}{r}$ and $L_{v} = \LocLip{\Msg}{r}$.



    \paragraph{The aggregated exponents}
	 Using the bound from \Cref{lemma:att:comparing-two-close-points} above we have that for any $\rvx_i \in \sampledset \cap \region$:
	\begin{gather}
		\frac{\exp \left(\logit(X, X_i)\right)\mu(I)}{\int_{x \in \region} \exp \left(\logit(X, x)\right) d\mu(x)} 
		\in 
		\left[ e^{-L_{l} r}, e^{L_{l} r} \right] \Rightarrow \\
		\label{bound:att:emp-measure:exponents}
		\frac{\exp \left(\logit(X, X_i)\right)\mu_{\sampledset}(I)}{\int_{x \in \region} \exp \left(\logit(X, x)\right) d\mu(x)} 
		\in 
		\left[ e^{-L_{l} r}(1 - \delta), e^{L_{l} r}(1 +  \delta) \right]
	\end{gather}

	Note, that  
	\begin{align*}
		|\sampledset \cap \region| = n\mu_{\sampledset}(\region) \quad \text{and} \quad \sum_{i=1}^n \exp \left(\logit(X, X_i)\right) \mathbf{I}[\rvx_i \in \region] = n \int_{x \in \region} \exp \left(\logit(X, x)\right) d\mu_{\sampledset}(x)
	\end{align*}
	
	Hence, averaging over all $\rvx_i \in \sampledset \cap \region$ we have that:
	\begin{gather*}
		\frac{
		\int_{x \in \region} \exp \left(\logit(X, x)\right) d\mu_{\sampledset}(x)}
		{\int_{x \in \region} \exp \left(\logit(X, x)\right) d\mu(x)}
		\in 
		\left[ e^{-L_{l} r}(1 - \delta), e^{L_{l} r}(1 +  \delta) \right]
	\end{gather*}


    Summing up over all $\region \in \regions^{(r)}$:
	\begin{gather}
		\label{bound:att:emp-measure:exp-all}
		\frac
		{\int_{x \in \chi} \exp \left(\logit(X, x)\right) d\mu_{\sampledset}(x)}
		{\int_{x \in \chi} \exp \left(\logit(X, x)\right) d\mu				(x)}
		\in 
		\left[ 
			e^{-L_{l} r}(1 - \delta), 
			e^{ L_{l} r}(1 +  \delta) 
		\right]
	\end{gather}


    Combining these bounds we get that for any $\rvx_i \in \sampledset \cap \region$: 
    \begin{align*}
	\frac{
        \Att_{\sampledset}(\rvx, \rvx_i) \cdot |\sampledset \cap \region|
    }
    {
        \int_{x \in \region}{\Att_{\chi}(\rvx, x)} d \mu (x)
    } 
	=
	\frac{
        \exp(\logit(\rvx, \rvx_i)) \mu_{\sampledset}(\region)
    }
    {
        \tfrac{1}{\nsamples} \sum_{j=1}^\nsamples \exp(\logit(\rvx, \rvx_j))
    }
	\cdot 
	\frac
    {
        \int_{x \in \chi}{\exp(\logit(\rvx, x))} d \mu (x)
    }
	{
        \int_{x \in \region}{\exp(\logit(\rvx, x))} d \mu (x)
    }
	\\
    \in
	\left[ 
		e^{-2L_{l} r} \cdot \frac{1 - \delta}{1 +  \delta}, 
		e^{ 2L_{l} r} \cdot \frac{1 + \delta}{1 - \delta}
	\right]
    \end{align*}

	Note that for $\delta, L_{l}r \in [0, 1/4]$ we have that:
	\begin{align*}
		e^{-2L_{l} r} \cdot \frac{1 - \delta}{1 + \delta} \geq (1 - 2L_{l} r) \cdot (1 - 2 \delta) \geq 1 - 4(L_{l} r + \delta) \\ 
		e^{2L_{l} r} \cdot \frac{1 + \delta}{1 - \delta} \leq (1 + 3L_{l} r) \cdot (1 + 3 \delta) \leq 1 + 4(L_{l} r + \delta)
	\end{align*}

	Hence, substituting $\delta = \tau \sqrt{C_\chi} n^{-\frac{1}{D_\chi + 2}}$ we have that 
	\begin{align*}
		\left|
			\frac{
				{\Att_{\sampledset}(\rvx, x)}  \mu_{\sampledset}(x)
			}
			{
				\int_{x \in \region}{\Att_{\chi}(\rvx, x)} d \mu (x)
			}
			- 1
		\right|
		\leq
		 \left(L_{l} + \tau \sqrt{C_\chi}  \right)n^{-\frac{1}{D_\chi + 2}}
	\end{align*}


\paragraph{Adding a value function.}
For any $\rvx_i \in \sampledset \cap \region$ we have that:
\begin{align*}
	&\left\|
		\Att_{\sampledset}(\rvx, \rvx_i) \Msg(\rvx_i)  \mu_{\sampledset}(\region)
		-
		\int_{x \in \region} \Att_\chi(\rvx, x) \Msg(x) d \mu(x)
	\right\|
	\\
	\leq &
	\left\|
		\left(
			\Att_{\sampledset}(\rvx, \rvx_i) \mu_{\sampledset}(x)
			-
			\int_{x \in \region}
			\Att_{\chi}(\rvx, x) d \mu(x)
		\right)
		\Msg(\rvx_i) 
	\right\|
	\\
	+&
	\left\|
		\int_{x \in \region} \Att_\chi(\rvx, x) 
		\left(
			\Msg(x) - \Msg(\rvx_i) 
		\right) d \mu(x)
	\right\|\\
	\leq & 
	4\left(\left(L_{l} + \tau \sqrt{C_\chi} \right)\| \Msg\|_{2, \infty} + \LocLip{\Msg}{r} \right) n^{-\frac{1}{D_\chi + 2}}
	\cdot
	\int_{x \in \region} \Att_\chi(\rvx, x) d \mu(x)
\end{align*}
 \paragraph{Aggregation of the messages with respect to attention scores.}

    Using the fact that $\int_{x \in \chi} \Att(\rvx, x) = 1$ and summing up over all $\region \in \regions$ we get that:
    \begin{gather*}
       \left\|
		\int\limits_{~x \in \chi} \Att_{\sampledset}(\rvx, x)  \Msg(x) d \mu_{\sampledset}(x)
		- 
		\int\limits_{~x \in \chi} \Att_\chi(\rvx, x)  \Msg(x) d \mu(x)
	\right\|
    \\
    \leq  \sum_{\region \in \regions^{(r)}} \left\|
		\int\limits_{~x \in \region} \Att_{\sampledset}(\rvx, x)  \Msg(x) d \mu_{\sampledset}(x)
		-   
		\int\limits_{~x \in \region} \Att_\chi(\rvx, x)  \Msg(x) d \mu(x)
	\right\| 
    \\
	\leq 
	4\left(\left(L_{l} + \tau \sqrt{C_\chi} \right)\| \Msg\|_{2, \infty} + \LocLip{\Msg}{r} \right)
	n^{-\frac{1}{D_\chi + 2}}
	\cdot
	\sum_{\region \in \regions^{(r)}} 
	\int_{x \in \region} \Att_\chi(\rvx, x) d \mu(x)
	\\
	\leq 
	4\left(\left(L_{l} + \tau \sqrt{C_\chi} \right)\| \Msg\|_{2, \infty} 
	+ 
	\LocLip{\Msg}{r} 
	 \right)n^{-\frac{1}{D_\chi + 2}}
    \end{gather*}

	As $\Msg(y) = \msg(\hst(y))$ and $\logit(x, y) = \att(\hst(x), \hst(y))$ we have that:
	\begin{align*}
		\LocLip{\Msg}{r} &\leq \Lip{\msg} \LocLip{\hst}{r} \\
		\|\Msg\|_{2, \infty} &\leq \Lip{\msg} \|\hst\|_{2, \infty}\\
		L_{l} &\leq \Lip{\att} \LocLip{\hst}{r} 
	\end{align*}

	Substituting the above inequalities we get that:
	\begin{gather*}
		\max_{\rvx \in \sampledset}
		\left\|
			\frac{1}{\nsamples} \sum_{j=1}^\nsamples 
			\Att_{\sampledset}(\rvx, \rvx_i)  \Msg(\rvx_i) 
			-
			\int_\chi \Att_{\chi}(\rvx, x)  \Msg(x) 
			d \mu(x)
		\right\| 
		\\
		\leq 
		4 \Lip{\msg}
		\left(
			\LocLip{\hst}{\nsamples^{-\nicefrac{1}{(D_\chi + 2)}}} (1 + \|\hst\|_{2, \infty} \Lip{\att})
			+
			\|\hst\|_{2, \infty}  \sqrt{C_\chi} \tau 
		\right) 
		\nsamples^{-\frac{1}{D_\chi + 2}}
		\\
		\leq 
		8 \Lip{\msg} \|\hst\|_{2, \infty}
		\left(
			\LocLip{\hst}{\nsamples^{-\nicefrac{1}{(D_\chi + 2)}}}   \Lip{\att}
			+
			\sqrt{C_\chi} \tau 
		\right) 
		\nsamples^{-\frac{1}{D_\chi + 2}}
	\end{gather*}

\end{proof}

\subsubsection{Layers}

Using the results of the previous sections we can now derive the total sampling error for the entire network $\GAT$ with $\nlayers$ layers.

\newcommand{\Llogit}{\LLip{\att}}
\newcommand{\Lmsg}{\LLip{\msg}}
\newcommand{\LPE}{\LLip{\PE}}
\newcommand{\Lnf}{\LLip{\nf}}
\newcommand{\Latt}{\LLip{\att}}
\newcommand{\Lpool}{\LLip{\pool}}
\newcommand{\Fmax}{} 

\begin{definition}\label{def:stable-positional-encoding-event}
Let $\event_{\PE}(\TT, \ST, \tau)$ be the event that:
$
	\sup_{x, y} \left\|\PE(x, y; \TT) - \PE(x, y; \ST)\right\| \leq R n^{-\rho}
$. By \Cref{def:stable-positional-encoding} we have that:
	\begin{align*}
		\Pr[\event_{\PE}(\TT, \ST, \tau)] \geq 1 - O \left(n(\tau^2 + \tau) \right) 
	\end{align*}

\end{definition}
	\begin{theorem}[Total Sampling Error]
		\label{thm:total-sampling-error-full}
		Given 
		\begin{itemize}
			\item hypothesis class of Transformers $\hypothesis$ (\Cref{def:hypothesis-class})
			\item $\admconsts$-Admissible continuous \tokenset $\TT = \ttdef$ (\Cref{def:admissible-tokenset}) 
			\item $n$ such that 
			$ n^{-\frac{1}{D_\chi + 2}} 
			\cdot 
			4^k\LLip{\att}^\nlayers 
			\LLip{\msg}^{\frac{\nlayers(\nlayers + 1)}{2}} 
			( \LLip{\nf} + \nlayers \LLip{\PE} ) 
			\leq \frac{1}{4}$
			\item sampled tokenset $\ST = \stdef \sample{n} \TT$
			\item $\tau \leq \tfrac{1}{4\sqrt{C_\chi}} n^{\nicefrac{1}{(D_\chi + 2)}}$
			\item events $\event_{\PE}(\TT, \ST, \tau)$ (\Cref{def:stable-positional-encoding-event}) and
					$\eventT{\TT}{\ST}{\tau}
					$  (\Cref{event:token-sampling}) happen
		\end{itemize}
		then \emph{for any} $\GAT \in \hypothesis$ (\Cref{def:hypothesis-class}) the following holds:
		\begin{gather*}
			\| \GAT(\TT) - \GAT(\ST) \|
			\leq   
			H_1 
			+
			H_2 \tau n^{-\frac{1}{(D_\chi + 2)}}
			+ 
			H_3 \tau n^{-\rho}	
		\end{gather*}
		where 
		\begin{align*}\label{eq:c1-c2-def}
			H_1 &= 
			 16^{\nlayers + 1} L_{\pool}  \nlayers \Lmsg^{\frac{\nlayers(\nlayers + 1)}{2}}
			\LLip{\att}^{\nlayers} 
			\left(\Lip{\nf} + \nlayers \Lip{\PE}\right)  \sqrt{C_\chi}
			\\
			H_2 &=  16^{\nlayers + 1} L_{\pool}  \nlayers \Lmsg^{\frac{\nlayers(\nlayers + 1)}{2}}
			\LLip{\att}^{\nlayers} 
			{C_\chi} \\
			H_3 &= 
			\HIIIdef
		\end{align*}

		Note, that by \Cref{lemma:pr:event-ts} 
		we have that:
		$$
		\Pr[\eventT{\TT}{\ST}{2\sqrt{C_\chi} \tau} \cup \event_{\PE}(\TT, \ST, \tau)] \geq 1 - 3 C_\chi n \left( e^{-\tau^2} + e^{-\tau}\right) 
		$$
	\end{theorem}

\begin{proof}

	Let $r = n^{-\nicefrac{1}{D_\chi + 2}}$

	Let $\next_{\TT}$ be a next-hidden state operator of $\GAT$ on the continuous tokenset $\TT$ and $\next_{\ST}$ be a next-hidden state operator of $\GAT$ on the sampled tokenset $\ST$:
	\begin{align*}
		\next_{\TT}(\hst^{(\nlayer)})(\rvx) = \hst^{(\nlayer + 1)}_{\TT}(\rvx) \text{ for any } 
		\next_{\ST}(\hst^{(\nlayer)})(\rvx) = \hst^{(\nlayer + 1)}_{\ST}(\rvx) 
	\end{align*}


	\paragraph{From one layer to the next.} Given some hidden state $\hst: \chi \to \Real^\HSdim$ and $ \hst_{\sampledset}: \sampledset \to \Real^\HSdim$ we have that for any $\rvx \in \sampledset$:
	\begin{gather*}
		\left\| \next_{\ST}(\hst_{\ST})(\rvx) - \next_{\TT}(\hst)(\rvx) \right\| 
		\leq  
		\\
		\underbrace{
		\left\| \next_{\ST}(\hst_{\ST})(\rvx) - \next_{\ST}(\hst_{\TT})(\rvx) \right\|
		}_{\text{apply \Cref{lem:perturbation-error}}}
		+
		\underbrace{
		\left\| \next_{\ST}(\hst_{\TT})(\rvx) - \next_{\TT}(\hst_{\TT})(\rvx) \right\|
		}_{\text{apply \Cref{thm:concentration-error-full}}} 		
		\leq 
		\\
		8\Lmsg\Llogit \|\hst_{\TT}\|_{2, \infty}  
			\left(
				2 \sup_{x}
				\left\|\hst_{\ST}(x) - \hst_{\TT}(x)\right\| 
				+
				\sup_{x, y}
				\left\|\PE(x, y; \ST) - \PE(x, y; \TT)\right\| 
			\right)	
			\\
		+
		4\Lmsg \|\hst_{\TT}\|_{2, \infty} 
		\left(
			\Llogit\LocLip{\hst_{\TT}}{r}
			+
			\sqrt{C_{\chi}} \tau
		\right)
		n^{-\nicefrac{1}{(D_\chi + 2)}}
	\end{gather*}

	Using the result of \Cref{lemma:local-lipschitzness-layers} for $\nlayer \in [0, \nlayers-1]$
	\begin{align*}
		\LocLip{\hst_{\TT}^{(\nlayer)}}{r}
		&\leq
		4^k\Llogit^k \Lmsg^{\frac{k(k-1)}{2}} \|\nf\|^k_{2, \infty}(  \Lnf + k \LPE ) 
		\\
		\|\hst_{\TT}^{(\nlayer)} \|_{\infty} &\leq \Lmsg^{\nlayer} \Fmax
	\end{align*}
	we get that:
	\begin{gather*}
		\sup_{\rvx \in \sampledset}\left\| \hst_{\ST}^{(\nlayers)}(\rvx) - \hst_{\TT}^{(\nlayers)}(\rvx) \right\| 
		\leq 
		\\
		16 \Lmsg^{\nlayers}  \Fmax \Llogit
		\sup_{x}
		\left\|\hst_{\ST}(x) - \hst_{\TT}(x)\right\| 
		+
		16 \Lmsg^{\nlayers}  \Fmax \Llogit \Delta_{\PE}
		\\ 
		+ 
		4\Lmsg^{\nlayers(\nlayers + 1)/2}
		\Llogit^{\nlayers}
		\left(\Lnf + \nlayers \LPE\right) 
		n^{-\nicefrac{1}{(D_\chi + 2)}}
		+
		4\Lmsg^{\nlayers} \Fmax  
		\sqrt{C_{\chi}} \tau
		n^{-\nicefrac{1}{(D_\chi + 2)}}
	\end{gather*}
	where $\Delta_{\PE} =  \sup_{x, y} \left\|\PE(x, y; \TT) - \PE(x, y; \ST)\right\|$.

	Applying this bound recursively we get that:
	\begin{gather} 
		\sup_{\rvx \in \sampledset}\left\| \hst_{\ST}^{(\nlayers)}(\rvx) - \hst_{\TT}^{(\nlayers)}(\rvx) \right\| 
		\leq \nonumber 
		\\
		4 \nlayers \cdot 16^\nlayers \Lmsg^{\frac{\nlayers(\nlayers + 1)}{2}}
		\Llogit^{\nlayers}
		\left(
			\left(\Lnf + \nlayers \LPE\right) 
			n^{-\frac{1}{(D_\chi + 2)}}
			+
			\Delta_{\PE}
		\right) 
		\label{eq:hst-bound-layers}
	\end{gather}


	\paragraph{The last pooling layer.} Similarly to \Cref{thm:concentration-error-full} we get that 
	\begin{gather*}
		\left \|\frac{1}{n}\sum_{\rvx \in \sampledset} \hst_{\TT}^{(\nlayers)}(X) - \int_{x\in \chi} \hst_{\TT}^{(\nlayers)}(x) d\mu(x)\right\| 
		\leq
		4
		\left(
			\LocLip{\hst^{(\nlayers)}_{\TT}}{r} 
			+
			\left\|\hst^{(\nlayers)}_{\TT}\right\|_{\infty}  
			\sqrt{C_\chi} \tau 
		\right) 
		\nsamples^{-\frac{1}{D_\chi + 2}}
		\\
		\leq
		4^{\nlayers+1} \Llogit^{\nlayers} \Lmsg^{\frac{\nlayers(\nlayers + 1)}{2}}
		(\Lnf + \nlayers \LPE ) n^{-\frac{1}{(D_\chi + 2)}} 
		+
		4^{\nlayers} \Lmsg^{\nlayers} \Fmax \sqrt{C_\chi} \tau n^{-\frac{1}{(D_\chi + 2)}}
	\end{gather*}

	Combining this bound with \Cref{eq:hst-bound-layers} we get that and including the pooling layer we get that:
	\begin{gather*}
		\left \| \pool \left( \frac{1}{n}\sum_{\rvx \in \sampledset} \hst_{\ST}^{(\nlayers)}(X) \right)  - \pool \left( \int_{x\in \chi} \hst_{\TT}^{(\nlayers)}(x) d\mu(x) \right) \right\| 
		\leq
		\\
		\Lpool \cdot 8 \nlayers \cdot 16^\nlayers \Lmsg^{\frac{\nlayers(\nlayers + 1)}{2}}
		\Llogit^{\nlayers}
		\left(
			\left(\Lnf + \nlayers \LPE + \sqrt{C_\chi} \tau \right) 
			n^{-\frac{1}{(D_\chi + 2)}}
			+
			\Delta_{\PE}
		\right) 
	\end{gather*}

	As we assume that event $\event_{\PE}(\TT, \ST, \tau)$ (\Cref{def:stable-positional-encoding-event}) happens we get that:
	$
		\Delta_{\PE} \leq R n^{-\rho}
	$. Substituting this into the equation above we get the desired bound.
\end{proof}

\subsection{Generalization Bound}
\label{apx:generalization-bound}

In this section we state the generalization bound for the tokenset regression task which is the main result of the paper. First, we present the corrolary of \Cref{thm:concentration-error-full} what is reformulation of the bound that holds with high probability to the expectation over sampled \tokensets. Then we use it to derive a generalization bound for the regression task with size-variable inputs.

\newcommand{\lemmaeef}[1]{(H_1 + H_2 \log #1)#1^{-\frac{1}{D_\chi + 2}} + H_3 #1^{-\rho} \log #1 }
\begin{lemma}\label{lem:expected-error-full}
	Let $\hypothesis$ be a hypothesis class of Transformers defined in \Cref{def:hypothesis-class}. Let $\TT = \ttdef$ be an $\admconsts$-admissible continuous \tokenset. Consider subsampling \tokenset $\ST = \stdef$ from the continuous \tokenset $\TT$. Then for $n$ such that $\ncondition{n}$:
	
	\begin{align*}
		\mathop{\mathbb{E}}\limits_{\ST \sample{n} \TT}
		\left[\sup_{\GAT \in \hypothesis}
		\left\|
		\GAT(\TT)
		-
		\GAT(\ST)
		\right\|
		\right]
		\leq 
		\lemmaeef{n}
	\end{align*}
	where 
	\begin{align*}
		H_1 &= O(\HIdef)
		\\
		H_2 &= O(\HIIdef) \\
		H_3 &= O(\HIIIdef)
	\end{align*}
\end{lemma}

\begin{proof}
	By \Cref{thm:concentration-error-full} we have that for $\ST \sample{n} \TT$ and $\tau \leq \tauUB$:
	\begin{align*}
		\Pr\left[\left\| \GAT(\TT) - \GAT(\ST) \right\| \geq( H_1 + H_2\tau)n^{-\frac{1}{D_\chi + 2}} + H_3 \tau n^{-\rho} \right] 
		&\leq (4 C_\chi n^2 + n^{\alpha}) \left(e^{-\tau^2} + e^{-\tau}\right) 
	\end{align*}

	where $H_1, H_2, H_3$ are defined in \Cref{eq:c1-c2-def} in \Cref{thm:total-sampling-error-full}.

	By \Cref{cor:expectation-of-subgaussian} we have that:
	\begin{align*}
		\E\left[\frac{\left\| \GAT(\TT) - \GAT(\ST) \right\| - H_1}{H_2n^{-\frac{1}{D_\chi + 2}} + R n^{-\rho} }  \right] 
		&\leq  {6}(\log C_\chi + (2 + \alpha) \log n) \leq  12 (2 + \alpha) \log n
	\end{align*}

	Let $\tau^* = \tauUB$. Now, we have to consider the case when $\tau > \tau^*$. In this case we have that:
	\begin{align*}
		&\Pr\left[\left\| \GAT(\TT) - \GAT(\ST) \right\| \geq (H_1 + H_2\tau)n^{-\frac{1}{D_\chi + 2}} + H_3 \tau n^{-\rho} \right] \\
		\leq & \Pr\left[\left\| \GAT(\TT) - \GAT(\ST) \right\| \geq (H_1 + H_2\tau^*)n^{-\frac{1}{D_\chi + 2}} + H_3 \tau^* n^{-\rho} \right] \\
		\leq & 8(C_\chi n^2 + n^{\alpha}) \exp\left(-n^{\frac{1}{D_\chi + 2}} \right) \leq n^{-1}
	\end{align*}

	as $n$ is such that $n^{\frac{1}{D_\chi + 2}} \geq 24C_\chi\log n$.

	Note, that $\left\| \GAT(\TT) - \GAT(\ST) \right\|$ is always upper bounded like:
	\begin{align*}
		\left\| \GAT(\TT) - \GAT(\ST) \right\| \leq 2\Lpool \cdot \| h^{(\nlayers)} \|_{\infty} \leq 2\Lpool \Lnf \Llogit^{\nlayers}
	\end{align*}

	Combining these  bounds we get the desired bound: 
	\begin{align*}
		&\E\left[\left\| \GAT(\TT) - \GAT(\ST) \right\| \right] = \\
		&\E\left[\left\| \GAT(\TT) - \GAT(\ST) \right\| \mid \left\| \GAT(\TT) - \GAT(\ST) \right\| \leq H_1 + H_2\tau^* \right] \cdot \Pr\left[\left\| \GAT(\TT) - \GAT(\ST) \right\| \leq H_1 + H_2\tau^* \right] \\
		+&\E\left[\left\| \GAT(\TT) - \GAT(\ST) \right\| \mid \left\| \GAT(\TT) - \GAT(\ST) \right\| > H_1 + H_2\tau^* \right] \cdot \Pr\left[\left\| \GAT(\TT) - \GAT(\ST) \right\| > H_1 + H_2\tau^* \right] \\
		\leq &
		8(H_1 +  H_2 \log n )n^{-\frac{1}{D_\chi + 2}}  +  H_3 n^{-\rho}  \log n
		+ 2\Lpool \Lnf \Llogit^{\nlayers} C_\chi n^{-1} \\
		\leq & 10(H_1 +  H_2 \log n)n^{-\frac{1}{D_\chi + 2}}  +  H_3 n^{-\rho}  \log n
	\end{align*}

\end{proof}

\begin{theorem}\label{thm:expected-ge-full}
	Given a distribution of $\admconsts$-admissible continuous \tokensets with $\Gamma$ classes $({\TT}^*_\gamma, y^*_\gamma)_{\gamma=1}^{\Gamma}$, large enough $N \in \mathbb{Z}_{>0}$, distribution over integers greater than $N$: $\natdist$ and a hypothesis class $\hypothesis$ (\Cref{def:hypothesis-class}) of Transformers we have the following bound for the generalization error: 

	\begin{gather*}
		\outE 
		\left[
			\sup_{\GAT \in \hypothesis}\left| R_{emp}(\GAT)-R_{exp}(\GAT)\right|
		\right] \leq \\
		{O\left(
			\Lip{\loss}
			\mathop{\E}\limits_{n \sim \natdist} \left[ \lemmaeef{n} \right]
			+ 
			\frac{2^\Gamma  \| \loss \|_{\infty}}{\sqrt{m}}
		\right)}
       	\end{gather*}
	where $C, C'$ are defined as in \Cref{lem:expected-error-full}. Empirical risk and expected risk are defined in the following way:
	\begin{align*}
		R_{emp} = 
			\frac{1}{m}\sum_{i=1}^m \loss(\GAT(\ST_i), y_i) 
		\quad \text{and} \quad
		R_{exp} =
			\mathop{\mathbb{E}}\limits_{%
  				\substack{(\TT, y)\sim\nu \\
				n \sim \natdist
				\\[2pt] \ST \sample{n} \TT}
			} 
			\loss(\GAT(\TT), y)
	\end{align*}
	
\end{theorem}

\begin{proof}
	By triangle inequality we have:
	\begin{align*}
		\outE
		\left[\sup_{\GAT \in \hypothesis}
		\left| 
			\frac{1}{m}\sum_{i=1}^m \loss(\GAT(\ST_i), y_i) 
			-
			\mathop{\mathbb{E}}\limits_{%
  				\substack{(\TT,\,y)\sim\nu \\[2pt] \ST\sim\TT}
			} 
			\loss(\GAT(\TT), y)
		\right|
		\right]
		\leq 
		\\
		\outE
		\left[\sup_{\GAT \in \hypothesis}
		\left| 
			\frac{1}{m}\sum_{i=1}^m \loss(\GAT(\ST_i), y_i) 
			-
			\frac{1}{m}\sum_{i=1}^m \loss(\GAT(\TT_i), y_i) 
		\right|
		\right] 
		+ \\
		\mathop{\mathbb{E}}\limits_{%
			\substack{\{(\TT_i,\,y_i)\}\sim \nu^m 
			}
		}
		\left[\sup_{\GAT \in \hypothesis}
		\left| 
			\frac{1}{m}\sum_{i=1}^m \loss(\GAT(\TT_i), y_i) 
			-
			\mathop{\mathbb{E}}\limits_{%
  				\substack{(\TT,\,y)\sim\nu}
			} 
			\loss(\GAT(\TT), y)
		\right|
		\right] 
		+ \\
		\sup_{\GAT \in \hypothesis}
		\left| 
			\mathop{\mathbb{E}}\limits_{%
  				\substack{(\TT,\,y)\sim\nu}
			} 
			\loss(\GAT(\TT), y)
			-
			\mathop{\mathbb{E}}\limits_{%
  				\substack{(\TT,\,y)\sim\nu \\[2pt] \ST\sim\TT}
			} 
			\loss(\GAT(\ST), y)
		\right|
	\end{align*}

	\paragraph{First term.} By \Cref{lem:expected-error-full}  we have that:
	\begin{align*}
		\outE \left| \loss(\GAT(\TT), y) 
		-
		\loss(\GAT(\ST), y) \right| 
		\leq \\
		\outE \sup_{\GAT \in \hypothesis} \Lip{\loss} \left\| \GAT(\TT) - \GAT(\ST) \right\|_{\infty}
		\leq 
		\outEE \Lip{\loss} (C + C' \log n_i)n_i^{-\frac{1}{D_\chi + 2}}
	\end{align*}

	By linearity of expectation and triangle inequality:
	\begin{align*}
		\outE
		\left[\sup_{\GAT \in \hypothesis}
		\left| 
			\frac{1}{m}\sum_{i=1}^m \loss(\GAT(\ST_i), y_i) 
			-
			\frac{1}{m}\sum_{i=1}^m \loss(\GAT(\TT_i), y_i) 
		\right|
		\right]  
		\leq \\
		\outEE \Lip{\loss} \left( \lemmaeef{n_i} \right)
	\end{align*}

	\paragraph{Second term.} For any $\GAT \in \hypothesis$ we have that:
	\begin{align*}
		\left| 
			\frac{1}{m}\sum_{i=1}^m \loss(\GAT(\TT_i), y_i) 
			-
			\mathop{\mathbb{E}}\limits_{%
  				\substack{(\TT,\,y)\sim\nu}
			} 
			\loss(\GAT(\TT), y)
		\right|
		\leq
		\sum_{\gamma=1}^{\Gamma} 
		\left| 
			\hat\nu_{\gamma}
			\loss(\GAT(\TT^*_\gamma), y^*_\gamma) 
			-
			\nu_{\gamma}
			\loss(\GAT(\TT^*_\gamma), y^*_\gamma) 
		\right| \\
		\leq
		\| \loss \|_{\infty}
		\sum_{\gamma=1}^{\Gamma} 
		\left| 
			\hat\nu_{\gamma}
			-
			\nu_{\gamma}
		\right|
	\end{align*}

	\paragraph{Third term.} This term can be bounded similarly to the first one. By Jensen's inequality (\Cref{thm:jensen-inequality}):
	\begin{align*}
		\sup_{\GAT \in \hypothesis}
		\left| 
			\mathop{\mathbb{E}}\limits_{%
  				\substack{(\TT,\,y)\sim\nu}
			} 
			\loss(\GAT(\TT), y)
			-
			\mathop{\mathbb{E}}\limits_{%
  				\substack{(\TT,\,y)\sim\nu \\[2pt] \ST\sim\TT}
			} 
			\loss(\GAT(\ST), y)
		\right|\\
		\leq 
		\mathop{\mathbb{E}}\limits_{%
  				\substack{(\TT,\,y)\sim\nu}
			} 
		\mathop{\mathbb{E}}\limits_{%
			\substack{
				n_i \sim P \\
				\ST\sample{n_i}\TT
			}
	  	} 
		\sup_{\GAT \in \hypothesis}
		\left| 
			\loss(\GAT(\TT), y)
			-
			\loss(\GAT(\ST), y)
		\right|
		\\
		\leq
		\outEE \Lip{\loss} \left( \lemmaeef{n_i} \right)
	\end{align*}

	Combining all the bounds using triangle inequality we get the stated bound.
\end{proof}

\section{Stable Positional Encoding} \label{apx:stable-positional-encoding}

In this section we provide an example of a stable positional encoding method for graphs. Specifically, we show the convergence of the $k$-step transition probability matrix-based RPE for graphs to the $k$-step transition kernel of the underlying graphon.

\subsection{Random Graph Model and Notation.} \label{apx:stable-positional-encoding:random-graph-model-and-notation} First, let us introduce some notation and the random graph model we are going to use. We assume that each graph is generated from a continuous graphon by sampling points and edges given the number of vertices $n$ and sparsity parameter $\sparsity$. We will assume that $\sparsity = n^{\alpha-1}$ for some $\alpha \in (1/2, 1]$.

\begin{itemize}
	\item Graphon is defined by a probabilistic-metric space $(\chi, \dist, \mu)$ and a symmetric function $\ef: \chi \times \chi \to [0, 1]$.
	\begin{itemize}
		\item degree function: $\deg_{\chi}(x) = \int_{\chi} \ef(x, z) d\mu(z)$
		\item normalized adjacency kernel: $\normadj(x, y) = \frac{\ef(x, y)}{\deg_{\chi}(x)}$
		\item $k$-step transition kernel: $\normadj^{(k)}(x, y) = \int\limits_{z \in \chi^k} \normadj(x, z_1) \cdots \normadj(z_k, y) d\mu(z_1) \cdots d\mu(z_k)$
	\end{itemize}
	We assume that $\sup_{x \in \chi} \deg x = \delta > 0$.
	\item Fully-connected graph on a set of sampled points $\sampledset = \{X_i\}_{i=1}^n \sim \mu^n$:
	\begin{itemize}
		\item adjacency matrix: $\mA \in \mathbb{R}^{n \times n}: \mA_{ij} = \ef(X_i, X_j)$ for $i \neq j$ and $\mA_{ii} = 0$
		\item degree function and the corresponding diagonal matrix $\mD \in \mathbb{R}^{n \times n}: \deg_{\sampledset}(X_i) = \sum_{j=1}^n \mA_{ij}$ and $\mD_{ii} = \deg_{\sampledset}(X_i)$
		\item probability transition matrix: $\mP \in \mathbb{R}^{n \times n}: \mP = \mD^{-1} \mA$
	\end{itemize}
	\item Graph with subsampled edges:
	\begin{itemize}
		\item adjacency matrix: $\hat\mA \in \mathbb{R}^{n \times n}: \hat\mA_{ij} = \hat\mA_{ji} = \text{Bernoulli}(\sparsity \mA_{ij})$ where $\text{Bernoulli}(p)$ is a Bernoulli random variable with parameter $p$
		\item degree function and the corresponding diagonal matrix $\hat\mD \in \mathbb{R}^{n \times n}: \hat\deg_{\sampledset}(X_i) = \sum_{j=1}^n \hat\mA_{ij}$ and $\hat\mD_{ii} = \hat\deg_{\sampledset}(X_i)$
		\item probability transition matrix: $\hat\mP \in \mathbb{R}^{n \times n}: \hat\mP = \hat\mD^{-1} \hat\mA$
	\end{itemize}
\end{itemize}

We are going to use $\diag, \offdiag : \mathbb{R}^{n \times n} \to \mathbb{R}^{n\times n}$ operators to extract the diagonal and off-diagonal parts of a matrix:
\begin{align*}
	\diag(A)_{ii} &= A_{ii}, \quad \diag(A)_{ij} = 0 \text{ for } i \neq j \\
	\offdiag(A)_{ij} &= A_{ij}, \quad \offdiag(A)_{ii} = 0 
\end{align*}


We are going to prove the following bound:

\begin{proposition}\label{lem:stable-positional-encoding-example-full}
	Assuming $\tau \leq ...$ and $ 3 \leq k \leq \sqrt{n}$ we have that:
		\begin{align*}
			\Pr\left[
				\left\|n \hat\mP^{k}  - \mPi^{(k)} \right\|_{\infty} 
				\geq 
				2\left(n^{-\frac{\alpha}{2}}+n^{\frac{1}{2}-\alpha} \right) \frac{8^{k+1}}{\delta^{k+1}} (1+\tau)\tau
			\right]
			\leq
			O(k n^2 e^{-\tau^2})
		\end{align*}	
\end{proposition}

We are going to split it into two bounds: one for the subsampling of vertices $\left\|n \mP^{k}  - \mPi^{(k)} \right\|_{\infty}$ and one for the subsampling of edges $\left\| n\hat\mP^{k}  - n\mP^{k} \right\|_{\infty}$.

\subsection{Subsampling of Vertices}

First, we compare the powered probability transition matrix $\mP^{k}$ to a probability transition kernel $\normadj^{(k)}$ on the sampled set $\{X_i\}_{i=1}^n$.

Let us also define matrices $\mPi, \mPi^{(k)}, \mDelta$ using functions $\normadj, \normadj^{(k)}$ and $\deg_{\chi}$ on the sampled set $\{X_i\}_{i=1}^n$:
\begin{align*}
	\mPi_{ij} = &  \normadj(X_i, X_j) \text{ for } i \neq j, \quad \mPi_{ii} = 0 \\
	\mPi^{(k)}_{ij} = & \normadj^{(k)}(X_i, X_j) \\
	\mDelta_{ii} = & \deg_{\chi}(X_i), \quad \mDelta_{ij} = 0 \text{ for } i \neq j 
\end{align*}

To prove this bound we are going to split it into multile smaller bounds.

\begin{lemma}\label{lem:mPi-mPiK}
	For $k \geq 2$ we have that:
	\begin{align*}
		\Pr
		\left[
			\left\| \tfrac{1}{(n-1)^{k-1}}\mPi^k - \mPi^{(k)} \right\|_\infty 
			\geq 
			\frac{2\tau k \cdot \delta^{-k}}{n} + \frac{k^2 \delta^{-k}}{2(n-k)^k}
		\right] \leq 2 n^2 e^{-2\tau^2}
	\end{align*}
\end{lemma}

\begin{proof}

Look at a fixed $i, j \in [n]$ and a corresponding $X_i, X_j \in \sampledset$. Define a following function $f_{(i, j)}: \chi^{n-2} \to \mathbb{R}$: 

\begin{align*}
	f_{(i, j)}(z_1, \ldots, z_{n-2}) 
	= 
	\frac{(n - 1 - k)!}{(n-2)!} 
	\sum_{\alpha \in [n-2]^{\underline{k-1}} }
	 \normadj(X_i, z_{\alpha_1}) \cdots \normadj(z_{\alpha_{k-1}}, X_j)
\end{align*}

Note, that by combinatorial rules there are $\frac{(n - 1 - k)!}{(n-2)!} $ terms in the sum. Consider iid random variables $Z_1, \ldots, Z_{n-2} \sim \mu$. As for each term in the sum $\alpha_a \neq \alpha_b$ for all $a \neq b$, random variables $Z_{\alpha_1}, \ldots, Z_{\alpha_{k-1}}$ are independent, hence, 

\begin{align*}
	\E[\normadj(X_i, z_{\alpha_1}) \cdots \normadj(z_{\alpha_{k-1}}, X_j)] &= 
	\int\limits_{z \in \chi^k}
	\normadj(X_i, z_1) \cdots \normadj(z_{k-1}, X_j) d\mu(z_1) \cdots d\mu(z_{k-1}) \\
	&= \normadj^{(k)}(X_i, X_j) \Rightarrow \\
	\E[f_{(i, j)}(Z_1, \ldots, Z_{n-2})] &= 
	\normadj^{(k)}(X_i, X_j)
\end{align*}

Now let us bound the difference in the function value if we change one of the arguments $z_t$ to $z_t'$:

\begin{align*}
	\left| f_{(i, j)}\left(z_1, \ldots, z_t, \ldots z_{n-2}\right) - f_{(i, j)}\left(z_1, \ldots, z_{t-1}, z_t', z_{t+1}, \ldots, z_{n-2}\right) \right|
	\leq 
	\\
	\tfrac{(n - 1 - k)!}{(n-2)!} 
	\sum_{\substack{\alpha \in [n-2]^{\underline{k-1}} \\ \exists q: t = \alpha_q}}
	 \normadj(X_i, z_{\alpha_1}) \cdots \normadj(z_{\alpha_{k-1}}, X_j)
	 \cdot
	 \left(
		1 
		- 
		\frac
		{\normadj(z_{\alpha_{q-1}}, z_t')\normadj( z_t', z_{\alpha_{q+1}})}
		{\normadj(z_{\alpha_{q-1}}, z_t)\normadj( z_t, z_{\alpha_{q+1}})}
	\right) 
\end{align*}

Note, that all the terms where none of the indices $\alpha_1, \ldots, \alpha_{k-1}$ are equal to $t$ cancelled out. Also, note that the number of terms where at least one of the indices $\alpha_1, \ldots, \alpha_{k-1}$ is equal to $t$ is at most $(k-1)\frac{(n-2)!}{(n-k)!}$. As $\normadj(\cdot, \cdot) \in [0, \delta^{-1}]$ we have that:
\begin{align*}
	\left| f_{(i, j)}\left(z_1, \ldots, z_t, \ldots z_{n-2}\right) - f_{(i, j)}\left(z_1, \ldots, z_{t-1}, z_t', z_{t+1}, \ldots, z_{n-2}\right) \right| \\
	\leq 
	\frac{(n-1-k)!}{(n-2)!} \cdot (k-1) \cdot \frac{(n-2)!}{(n-k)!} \cdot \delta^{-k}
	\leq 
	\frac{k \cdot \delta^{-k}}{n-k}
	\leq
	\frac{2k \cdot \delta^{-k}}{n}
\end{align*}

Hence, by McDiarmid's inequality we have that:

\begin{align}\label{eq:mcdiarmid-f-ij}
	\Pr
	\left[
		\left| f_{(i, j)}(Z_1, \ldots, Z_{n-2}) - \normadj^{(k)}(X_i, X_j) \right| 
		\geq 
		\frac{\tau k \sqrt{n} \cdot \delta^{-k}}{n-k}
	\right] 
	\leq 
	2 \exp\left(-{2\tau^2}\right)
\end{align}

Now let us look at the $ij$ entry of $\left(\tfrac{1}{n-1}\mPi\right)^k$. By looking at the power of the matrix and remembering that $\mPi_{ii} = \mPi_{jj} = 0$ we have that:
\begin{align*}
	\left[\tfrac{1}{(n-1)^{k-1}}\mPi^k\right]_{ij} = \frac{1}{(n-1)^{k-1}} \sum_{l_1, \ldots, l_k \in [n] \setminus \{i, j\}} \normadj^{(k)}(X_i, X_{l_1}) \cdots \normadj^{(k)}(X_{l_k}, X_j)
\end{align*}

Using generalized Bernoulli's inequality (\Cref{thm:gen-bern}) we have that:
\begin{align*}
	0 < \frac{(n - 1 - k)!}{(n-2)!} - \frac{1}{(n-1)^{k-1}} 
	= 
	\frac{(n - 1 - k)!}{(n-2)!} 
	\left(1
	- 
	\prod_{l=1}^{k-1}\left(1 - \frac{l}{n-1}\right)	
	\right) \leq 
	\\
	\frac{(n - 1 - k)!}{(n-2)!} 
	\sum_{l=1}^{k-1} \frac{l}{n-1} 
	=
	\frac{(n - 1 - k)!}{(n-2)!} 
	\frac{k(k-1)}{2(n-1)} < \frac{k^2}{2(n-k)^k}
\end{align*}

Therefore, 
\begin{align*}
	\left|\left[\tfrac{1}{(n-1)^{k-1}}\mPi^k\right]_{ij} - f_{(i, j)}(X_1, \ldots, X_{n-2}) \right| 
	\\
	= 
	\left |\frac{1}{(n-1)^{k-1}} - \frac{(n - 1 - k)!}{(n-2)!} \right| \sum_{l_1, \ldots, l_k \in [n] \setminus \{i, j\}} \normadj^{(k)}(X_i, X_{l_1}) \cdots \normadj^{(k)}(X_{l_k}, X_j)
	\leq \frac{k^2 \delta^{-k}}{2(n-k)^k} 
\end{align*}

Note, that all variables $X_u$ where $u\neq i, u \neq j$ are iid sampled from $\mu$ and independent from $X_i, X_j$. We can combine the concentration bound with \Cref{eq:mcdiarmid-f-ij} and using triangle inequality we get that:

\begin{align*}
	\Pr
	\left[
		\left| \frac{1}{(n-1)^{k-1}}[\mPi^k]_{ij} - \normadj^{(k)}(X_i, X_j) \right| 
		\geq 
		\frac{2\tau k \cdot \delta^{-k}}{n} + \frac{k^2 \delta^{-k}}{2(n-k)^k}
	\right] \geq 1 - 2 e^{-2\tau^2}
\end{align*}

\paragraph{Summing things up.} Applying union bound over all $i, j \in [n]$ we get that:

\begin{align*}
	\Pr
	\left[
		\left\| \frac{1}{(n-1)^{k-1}}[\mPi^k] - \mPi^{(k)} \right\|_\infty 
		\geq 
		\frac{2\tau k \cdot \delta^{-k}}{n} + \frac{k^2 \delta^{-k}}{2(n-k)^k}
	\right] \leq 2 n^2 e^{-2\tau^2}
\end{align*}

\end{proof}

\begin{lemma}\label{lem:mD-mDelta}
	For $\mD$ and $\mDelta$ defined as in \Cref{apx:stable-positional-encoding:random-graph-model-and-notation}, we have that:
	\begin{align*}
		\Pr\left[\left\| \tfrac{1}{n-1}\mD - \mDelta \right\|_\infty \leq \frac{2\tau}{\sqrt{n}}\right] 
		\geq 
		1 - 2 n e^{-2\tau^2}
	\end{align*}
\end{lemma}

\begin{proof}
	For each $i \in [n]$ we have that:
	\begin{align*}
		\tfrac{1}{n-1}\mD_{ii} = \sum_{j \in [n] \setminus \{i\}} \normadj(X_i, X_j) \\
		\mDelta_{ii} = \deg_{\chi}(X_i) = \int_{\chi} \normadj(X_i, z) d\mu(z)
	\end{align*}

	Hence, by Hoeffding's inequality (\Cref{thm:hoeffding}) we have that:
	\begin{align*}
		\Pr\left[\left| \tfrac{1}{n-1}\mD_{ii} - \mDelta_{ii} \right| \geq \frac{2\tau}{\sqrt{n}}\right] \leq 2 e^{-2\tau^2}
	\end{align*}

	Applying union bound over all $i \in [n]$ we get the stated result.
\end{proof}

\begin{lemma}\label{lem:whatever-to-rename}
	For $\mP, \mPi, \mDelta, \mD$ as defined in \Cref{apx:stable-positional-encoding:random-graph-model-and-notation}
	\begin{align*}
			\left\| 
			\mP^k
			- 
			\tfrac{1}{(n-1)^{k}}\mPi^k \right\|_\infty 
			\leq 
			k \delta^{-k} \cdot 
			\sup_{i \in [n]} 
			\frac{\left|  \mD_{ii} - (n-1)\mDelta_{ii} \right|}{\mD_{ii} \cdot (n-1)\mDelta_{ii}} 
	\end{align*}
\end{lemma}

\begin{proof}
	\begin{align*}
		\mP^k = \mD^{-1} \mA \mD^{-1} \mA \cdots \mD^{-1} \mA \\
		\mPi^k = \mDelta^{-1} \mA \mDelta^{-1} \mA \cdots \mDelta^{-1} \mA \\
	\end{align*}

	Note that:
	\begin{align*}
		\left\| \mP^{k+1} - \tfrac{1}{(n-1)^{k+1}}\mPi^{k+1} \right\|_\infty 
		\leq 
		\left\| \mP( \mP^{k} - \tfrac{1}{(n-1)^{k}}\mPi^{k}) \right\|_\infty 
		+ 
		\left\| (\mP - \tfrac{1}{n-1}\mPi)\tfrac{1}{(n-1)^{k}}\mPi^{k} \right\|_\infty \\
		\leq 
		\sup_{i \in [n]} \sum_{j \in [n]} \left| \mP_{ij} \right| \cdot 
		\left\| \mP^{k}_{ij} - \tfrac{1}{(n-1)^{k}}\mPi^{k}_{ij} \right\|_\infty
		+
		\left\| \mP - \tfrac{1}{n-1}\mPi \right\|_\infty \cdot \frac{(n-1)^{k}}{(n-1)^{k}} \| \mPi \|^k_\infty
		\\
		\leq 
		\left\| \mP^{k}_{ij} - \tfrac{1}{(n-1)^{k}}\mPi^{k}_{ij} \right\|_\infty
		+
		\delta^{-k} \left\| \mP - \tfrac{1}{n-1}\mPi \right\|_\infty 
	\end{align*} 
	Expanding this recursively we get that:
	\begin{align*}
		\left\| \mP^{k+1} - \tfrac{1}{(n-1)^{k+1}}\mPi^{k+1} \right\|_\infty 
		\leq 
		(k + 1) \delta^{-k} \left\| \mP - \tfrac{1}{n-1}\mPi \right\|_\infty 
	\end{align*}

	Using the bound:
	\begin{align*}
		\left\| (\mD^{-1}  - \tfrac{1}{n-1}\mDelta^{-1}) \mA \right\|_\infty \leq \left\| \mD^{-1}  - \tfrac{1}{n-1}\mDelta^{-1} \right\|_\infty \cdot \| \mA \|_\infty 
		= 
		\sup_{i \in [n]} 
		\frac{\left|  \mD_{ii} - (n-1)\mDelta_{ii} \right|}{\mD_{ii} \cdot (n-1)\mDelta_{ii}} 
	\end{align*}
	we can finish the proof.
\end{proof}

\begin{theorem}\label{thm:mP-k-mPiK}
	Assuming $n$ is large enough, we have that:
	\begin{align*}
		\Pr
		\left[\left\|{(n-1)}\mP^k - \mPi^{(k)} \right\|_\infty 
			\geq 
			\frac{4\tau k \cdot \delta^{-k}}{n} + \frac{k^2 \delta^{-k}}{2(n-k)^k}
		\right] 
		\leq 4 n^2 e^{-2\tau^2}
	\end{align*}
\end{theorem}

\begin{proof}
	Combination of results from \Cref{lem:mPi-mPiK} and \Cref{lem:mD-mDelta} and \Cref{lem:whatever-to-rename} yields the stated bound.
\end{proof}

\subsection{Sampling of edges}
In this subsection we are going to establish a series of bounds to compare matrices $\mA$ and $\hat\mA$, $\mD$ and $\hat\mD$, $\mP$ and $\hat\mP$. 

Let $d = \sup_{\rvx_i \in \sampledset}\deg_{\sampledset} (X_i)$.

\begin{lemma}\label{lemma:P-bounds} Assuming $\mD_{ii} \geq nd$ for some $d$ as above, we have that for any $k \geq 1$:
	\begin{align*}
		\| n \mP^{k} \|_{\infty} &\leq d^{-1} 
	\end{align*}
\end{lemma}

\begin{proof}
	\begin{align*}
		\| n \mP^{k} \|_{\infty}  \leq \sup_{i, j \in [n]} \sum_{t \in [n]} n\mP_{it} \cdot [\mP^{k-1}]_{tj} = \sup_{i, j \in [n]} \sum_{t \in [n]} \frac{n\mA_{it}}{\mD_{tt}} \cdot [\mP^{k-1}]_{tj} \\
		\leq d^{-1} \sup_{j \in [n]} \sum_{t \in [n]}  [\mP^{k-1}]_{tj}
	\end{align*}
	As for any $j$ $ \sum_{t \in [n]}  [\mP^{k-1}]_{tj}$ by the properties of the probability transition matrix,  the stated bound follows.
\end{proof}

\begin{lemma}\label{lem:event-1-mA-A-tau}
	Define the following event $\event(\mA, \hat\mA, \tau)$:
	\begin{align*}
		\left\| \mD - \tfrac{1}{\sparsity} \hat\mD \right\|_\infty \leq 2n^{1 - \alpha/2} \tau
	\end{align*}
	Then we have that:
	\begin{align*}
		\Pr\left[\event(\mA, \hat\mA, \tau)\right] \geq 1 - 2 n e^{-\frac{\tau^2}{1 + \tau}}
	\end{align*}
\end{lemma}

\begin{proof}
  By definition, $\mD_{ii} = \sum_{j \in [n]} \mA_{ij}$ and $\hat\mD_{ii} = \sum_{j \in [n]} \hat\mA_{ij}$ where $\hat\mA_{ij} = \text{Bernoulli}(\sparsity \mA_{ij})$. 
  \begin{align*}
	\E\left[\mA_{ij} - \tfrac{1}{\sparsity} \hat\mA_{ij}\right] 
	&= 0 \\
	\E\left[(\mA_{ij} - \tfrac{1}{\sparsity} \hat\mA_{ij})^2\right] 
	&=
	\sparsity\mA_{ij} (\mA_{ij} - \tfrac{1}{\sparsity}\cdot 1)^2
	+
	(1 - \sparsity\mA_{ij} ) (\mA_{ij} - \tfrac{1}{\sparsity}\cdot 0)^2 \\
	& = \tfrac{1}{\sparsity} \mA_{ij} (1 - {\sparsity}\mA_{ij}) \leq \tfrac{1}{\sparsity}\
  \end{align*}
	
  Then by Bernstein's inequality \Cref{thm:bernstein} for $t = 2n^{1 - \alpha/2} \tau$ we get the following bound:
  \begin{align*}
    \Pr\left[
		\left| \mD_{ii} - \tfrac{1}{\sparsity} \hat\mD_{ii} \right| 
		\geq 
		2n^{1 - \alpha/2} \tau 
		\right] 
	\leq 
	2 \exp\left(-\frac{t^2/2}{n/\sparsity + t/(3\sparsity)}\right)
	\leq
	2 e^{-\frac{\tau^2}{1 + \tau}}
  \end{align*}
  
  Applying union bound over all $i \in [n]$ we get the stated result.

\end{proof}

\begin{corollary}\label{cor:event-1}
	Assuming $\tau \leq dn^{\frac{\alpha}{2}}/4$ and event $\event(\mA, \hat\mA, \tau)$ happens, all of the following bounds hold:
	\begin{align*}
		\left\| {\tfrac{1}{\sparsity}\mD^{-1} } - \hat\mD^{-1} \right\|_{\infty} &\leq 2n^{-1 - \alpha/2} d^{-2} \tau 
		\\
		\forall i: \quad \hat\mD_{ii} &\geq \frac{n^\alpha d}{2} 
		\\
		\sup_{i \in [n]} \sum_{j \in [n]} \hat \mP_{ji} &\leq 4d^{-1}\\
		\| \hat \mP \|_{\infty} &\leq 2d^{-1}n^{-\alpha}
	\end{align*}
\end{corollary}
\begin{proof}
	Consider any $i \in [n]$. First,  by the definition of $\event(\mA, \hat\mA, \tau)$  and the fact that $\tau \leq dn^{\frac{\alpha}{2}}/4$ we have that 
	\begin{align*}
		2n 
		\geq 
		n(1 + d/4)
		\geq 
		\mD_{ii} +  2n^{1 - \alpha/2} \tau \geq  \frac{\hat\mD_{ii}}{\sparsity} \geq \mD_{ii} -  2n^{1 - \alpha/2} \tau \geq dn - dn/2 = dn/2
	\end{align*}

	Hence, 
	\begin{align*}
		\left| \frac{1}{\sparsity\mD_{ii} } - \frac{1}{\hat\mD_{ii}} \right| 
		= 
		\left| \frac{\hat\mD_{ii}/\sparsity - \mD_{ii}}{\mD_{ii}  \hat\mD_{ii}/\sparsity } \right| 
		\leq \frac{4n^{1 - \alpha/2} \tau}{n^2d^2} = 4n^{-1 - \alpha/2} d^{-2} \tau  \\
		\frac{n^\alpha d}{2} \leq \hat \mD_{ii} \leq 2n^\alpha
	\end{align*}

	Also, 
	\begin{align*}
		\sum_{j \in [n]} \hat \mP_{ji} = \sum_{j \in [n]} \frac{ \hat\mA_{ji} }{\hat\mD_{jj}} 
		\leq 
		\hat \mD_{ii} \cdot \frac{2}{n^\alpha d} \leq 4d^{-1} \\
	\end{align*}

	And for any $j \in [n]$ we have that $\hat \mP_{ji} = \hat \mA_{ji} / \hat \mD_{jj} \leq \frac{2}{n^\alpha d}$.

\end{proof}



\begin{lemma} \label{lemma:event-2-mA-A-tau}
	Let $\event_2(\mA, \hat \mA, \tau)$ be such that for all $i \neq j \in [n]$:
	\begin{align*}
		\left\| \offdiag\left(\mA^2 - (\tfrac{1}{\sparsity} \hat\mA)^2\right)\right\|_{\infty}
			 \leq 
			 2\tau n^{\frac{3}{2} - \alpha}
	\end{align*}
	Then 
	\begin{align*}
		\Pr\left[\event_2(\mA, \hat \mA, \tau)\right] \geq 1 -  2 e^{-\frac{\tau^2}{1 + \tau}}
	\end{align*}
\end{lemma}

\begin{proof}
	Note that for any off diagonal entry $(i, j), i\neq j$ we have the following equality:
	\begin{align*}
		\left[ \mA^2 - (\tfrac{1}{\sparsity} \hat\mA)^2\right]_{ij}  
		= 
		\sum_{l=1}^n  \mA_{il} \mA_{lj} - \tfrac{1}{\sparsity^2} \hat\mA_{il} \hat\mA_{lj} 
	\end{align*}
	Consider random variables $Z_l = \mA_{il} \mA_{lj} - \tfrac{1}{\sparsity^2} \hat\mA_{il} \hat\mA_{lj}$ for $l\in[n]$.
	\begin{align*}
		\E\left[Z_l\right] &= 0 \\
		\left|Z_l\right| &\leq \tfrac{1}{\sparsity^2}  \\
		\sum_{l=1}^{n}\E\left[Z_l^2\right] 
		&
		 = \sum_{l=1}^{n}\tfrac{1}{\sparsity^2} \mA_{il} \mA_{lj} (1 - {\sparsity^2}\mA_{il} \mA_{lj}) \leq \tfrac{n}{\sparsity^2} = n^{3 - 2\alpha}
	  \end{align*}

	  By Bernstein's inequality for $t = 2n^{\frac{3}{2} - \alpha} \tau \leq 2n\tau$ (as $\alpha \geq 1/2$) and the fact that $\mD_{ii}\mD_{ll} \geq n^2d^2$ we have that:
	  \begin{align*}
		 \Pr\left[
			\left| \sum_{l=1}^n  \mA_{il} \mA_{lj} - \tfrac{1}{\sparsity^2} \hat\mA_{il} \hat\mA_{lj} \right|
			 \geq 
			 2\tau n^{\frac{3}{2} - \alpha}
		\right] &\leq
		2 \exp\left(-\frac{4n^{3 - \alpha}\tau^2/2}{n^{3 - 2\alpha} + 2n^{3-2\alpha}\tau/3}\right) \\
		&\leq 2 e^{-\frac{\tau^2}{1 + \tau}}
	  \end{align*}

	  Applying union bound for all $i\neq j \in [n]$ we have the stated inequality.
\end{proof}

\begin{corollary}\label{eq:sum-hat-mA-hat_mA}
	Assuming $\event_2(\mA, \hat \mA, \tau)$ happens, we have that:
	\begin{align}
	\left\|  \offdiag (\hat\mA^2) \right\|_{\infty} 
	\leq 
	n^{2\alpha - 1}(1 + \tau)
	\end{align}
\end{corollary}

\begin{proof}
	If $\event_2(\mA, \hat \mA, \tau)$ happens and as $\alpha \geq 1/2$, we have that for all $i\neq j$:
	\begin{align}
	\sum_{l=1}^n  \hat\mA_{il} \hat\mA_{lj} \leq \sparsity^2 \left(\sum_{l=1}^n \mA_{il} \mA_{lj} + 2\tau n^{\frac{3}{2} - \alpha} \right)
	\leq n^{2\alpha - 2} \cdot n(1 +\tau) = n^{2\alpha - 1}(1 + \tau)
	\end{align}
\end{proof}

\begin{lemma}\label{lem:mP-2-mPhat-2-ij}
	Assuming events $\event(\mA, \hat \mA, \tau)$ and $\event_2(\mA, \hat \mA, \tau)$ happen, we have that:
	\begin{align*}
		\left\|n \cdot \offdiag(\mP^2 -  \hat\mP^2) \right\|_{\infty} 
		\leq 
		24\left(
			n^{-\frac{\alpha}{2}} d^{-3}
			+
			n^{\frac{1}{2} - \alpha} d^{-2}
		\right) 
		\tau 
	\end{align*}
\end{lemma}

\begin{proof}
	Consider any $i\neq j$.
	As $\mP^2 = \mD^{-1} \mA \mD^{-1} \mA$ and $\hat\mP^2 = \hat\mD^{-1} \hat\mA \hat\mD^{-1} \hat\mA$ we have that:
	\begin{align*}
		\left[ \mP^2 -  \hat\mP^2 \right]_{ij} 
		= 
		\left|\sum_{l=1}^n \frac{ \mA_{il} \mA_{lj}}{\mD_{ii}\mD_{ll}}  - \frac{ \hat\mA_{il} \hat\mA_{lj}}{\hat\mD_{ii}\hat\mD_{ll}} \right|
		&\leq 
		\frac{1}{\mD_{ii}\mD_{ll}} 
		\left| \sum_{l=1}^n  \mA_{il} \mA_{lj} - \tfrac{1}{\sparsity^2} \hat\mA_{il} \hat\mA_{lj} \right|
		\\
		&+
		\sum_{l=1}^n 
		\hat\mA_{il} \hat\mA_{lj}
		\left| \frac{1}{\sparsity^2\mD_{ii}\mD_{ll}}  - \frac{1}{\hat\mD_{ii}\hat\mD_{ll}} \right|
	\end{align*}
	
	\paragraph{First term.} By definition of event $\event_2(\mA, \hat \mA, \tau)$ (\Cref{lemma:event-2-mA-A-tau}) and the fact that $\mD_{ii}\mD_{ll} \geq n^2d^2$ we have that:
	\begin{align*}
		\frac{1}{\mD_{ii}\mD_{ll}} 
		\left| \sum_{l=1}^n  \mA_{il} \mA_{lj} - \tfrac{1}{\sparsity^2} \hat\mA_{il} \hat\mA_{lj} \right|
		\leq n^{-2}d^{-2} \cdot 2\tau n^{\frac{3}{2} - \alpha} = 2n^{-\frac{1}{2} - \alpha} d^{-2} \tau
	\end{align*}

	\paragraph{Second term.} For each $l$ we have that:
	\begin{align*}
		\left| \frac{1}{\sparsity^2\mD_{ii}\mD_{ll}}  - \frac{1}{\hat\mD_{ii}\hat\mD_{ll}} \right| 
		\leq 
		\frac{|\tfrac{1}{\sparsity^2}\hat \mD_{ii} \hat \mD_{ll} - \mD_{ii}\mD_{ll}|}{\mD_{ii}\mD_{ll}\hat \mD_{ii}\hat \mD_{ll}} 
		\leq 
		\frac{|\tfrac{1}{\sparsity} \cancel{\hat \mD_{ii}} (\tfrac{1}{\sparsity}\hat \mD_{ll} - \mD_{ll})|}{\mD_{ii}\mD_{ll}\cancel{\hat \mD_{ii}}\hat \mD_{ll}} 
		+ 
		\frac{|(\tfrac{1}{\sparsity}\hat \mD_{ii} -  \mD_{ii}) \cancel{\mD_{ll}}|}{\mD_{ii}\cancel{\mD_{ll}}\hat \mD_{ii}\hat \mD_{ll}} \\
		\leq
		\frac{n^{1-\alpha} \cdot 2n^{1 - \alpha/2} \tau }{d^2n^2 \cdot n^{\alpha}d/2}
		+
		\frac{2n^{1 - \alpha/2} \tau }{dn \cdot n^{2\alpha}d^2/4}
		\leq 12n^{-5\alpha/2} d^{-3} \tau
	\end{align*}

	Hence, 
	\begin{align*}
		\sum_{l=1}^n 
		\hat\mA_{il} \hat\mA_{lj}
		\left| \frac{1}{\sparsity^2\mD_{ii}\mD_{ll}}  - \frac{1}{\hat\mD_{ii}\hat\mD_{ll}} \right| 
		\leq 
		12n^{-\frac{5\alpha}{2}} d^{-3} \tau\cdot 
		\underbrace{
			\sum_{l=1}^n 
			\hat\mA_{il}\hat\mA_{lj}
		}_{
			\substack{\leq  n^{2\alpha - 1}(1 + \tau) \\ \text{by  \Cref{eq:sum-hat-mA-hat_mA}}}
		}
		\leq 
		12n^{-\frac{\alpha}{2} - 1} d^{-3} (1 + \tau)
	\end{align*}
\end{proof}

\begin{lemma}\label{lem:mP-2-mPhat-2-ii}
	Assuming event $\event(\mA, \hat \mA, \tau)$ happens, we have that:
	\begin{align*}
		\left\| n \cdot \diag(\mP^2 -  \hat\mP^2) \right\|_{\infty} \leq d^{-1} n^{1-\alpha} \tau
	\end{align*}
\end{lemma}

\begin{proof}
	\begin{align*}
		\left[ \hat\mP^2 \right]_{ii} 
		&= 
		\sum_{l=1}^n \frac{ \hat\mA_{il}^2 }{\hat\mD_{ii}\hat\mD_{ll}} 
		= 
		\sum_{l=1}^n \frac{ \hat\mA_{il} }{\hat\mD_{ii}\hat\mD_{ll}} 
		\lesssim 
		d^{-1} n^{-\alpha} \tau  \\
		\left[ \mP^2 \right]_{ii} 
		&= \sum_{l=1}^n \frac{ \mA_{il}^2 }{\mD_{ii}\mD_{ll}} 
		\leq \frac{1}{dn} 
	\end{align*}
\end{proof}

\begin{lemma} \label{lem:mP-k-mPhat-P} 
	Define event $\event_3(\mA, \hat \mA, \tau, k)$ as:
	\begin{align*}
		\left\| n \mP^k(\mP -  \hat\mP)  \right\|_{\infty} \leq 4 n^{- \frac{\alpha}{2}} d^{-2} \tau
	\end{align*}

	Then
	\begin{align*}
		\Pr\left[\event_3(\mA, \hat \mA, \tau, k)\right] \geq 1 - 3 n^2 e^{-\frac{\tau^2}{1 + \tau}}
	\end{align*}

\end{lemma}
\begin{proof} 
	\begin{align}\label{eq:tri}
		\left\|  \mP^k(\mP -  \hat\mP)  \right\|_{\infty} 
		= 
		\left\|  \mP^k(\mD^{-1}\mA -  \hat\mD^{-1}\hat\mA)  \right\|_{\infty} 
		\leq \nonumber 
		\\
		\left\|  \mP^k\mD^{-1}(\mA -  \tfrac{1}{\sparsity}\hat\mA)  \right\|_{\infty} 
		+ 
		\left\|  \mP^k(\tfrac{1}{\sparsity}\mD^{-1} - \hat\mD^{-1})\hat\mA  \right\|_{\infty} 
	\end{align}

	
	\paragraph{First term.}
	Let us look at any $i, j \in [n]$. Note, that 
	\begin{align*}
		\left| \mP^k\mD^{-1}(\mA -  \tfrac{1}{\sparsity}\hat\mA)  \right|_{ij}  = 
		\sum_{l=1}^n  \frac{[\mP^k]_{il}}{\mD_{ll}}
		\left(
			{\mA_{lj}} - \frac{1}{\sparsity}{\hat\mA_{lj}}
		\right)
	\end{align*}

	Let $Q_l = \frac{[\mP^k]_{il}}{\mD_{ll}}
			\left(
				{\mA_{lj}} - \frac{1}{\sparsity}{\hat\mA_{lj}}
			\right)$. First, note that for fixed $l$ we have that as $\mP^k \leq \frac{1}{nd}$:
	\begin{align*}
		\E \left[ Q_l \right] 
		&= 0 ,
		\\
		\left|Q_l\right| &\leq \frac{[\mP^k]_{il}}{\mD_{ll}} \cdot \frac{1}{\sparsity} \leq \frac{n^{1-\alpha}}{n^2d^2} 
		\leq 
		2n^{- 1 - \alpha} d^{-2} 
		\\
		\sum_{l=1}^n \E \left[ Q_l^2 \right] 
		&
		\leq \sum_{l=1}^n [\mP^k]^2_{il}  \cdot \frac{1}{\mD^2_{ll}\sparsity} 
		\leq n \cdot  d^{-2}n^{-2}  \cdot d^{-2}n^{-2} \cdot n^{1-\alpha}  
		= d^{-4} n^{-2-\alpha}
	\end{align*}

	Hence, by Bernstein's inequality for $t = 2n^{-1 - \frac{\alpha}{2}} d^{-2} \tau$ we have that:
	\begin{align*}
		\Pr\left[\left| \sum_{l=1}^n  \frac{[\mP^k]_{il}}{\mD_{ll}}
			\left(
				{\mA_{lj}} - \frac{1}{\sparsity}{\hat\mA_{lj}}
			\right)
		\right| \geq 2n^{-1-\frac{\alpha}{2}} d^{-2} \tau \right] 
		&\leq 
		2\exp
		\left(-
			\frac{2n^{-2 - \alpha} d^{-4} \tau^2}
			{d^{-4} n^{-2-\alpha} + \frac{4}{3}n^{- 2 - \frac{3\alpha}{2}} d^{-4}\tau}
		\right) 
		\\
		&\leq
		2 e^{-\frac{\tau^2}{1 + \tau}}
	\end{align*}

	Applying the union bound across all $i, j \in [n]$ we have that:
	\begin{align}\label{eq:mP-k-mPhat-P-1}
		\Pr\left[
			\left\|  \mP^k(\mD^{-1}\mA -  \hat\mD^{-1}\hat\mA)  \right\|_{\infty} 
			\geq 
			2n^{- \frac{\alpha}{2}} d^{-2} \tau
		\right] \leq 2n^2 e^{-\frac{\tau^2}{1 + \tau}}
	\end{align}
	
	\paragraph{Second term.} 
	Let us assume that event $\event(\mA, \hat \mA, \tau)$ happens (\Cref{lem:event-1-mA-A-tau}). Consider any $i, j \in [n]$. Note, that
	\begin{align*}
		\left[  \mP^k(\tfrac{1}{\sparsity}\mD^{-1} - \hat\mD^{-1})\hat\mA  \right]_{ij} 
		= 
		\sum_{l=1}^n [\mP^k]_{il} \cdot 
		\hat\mA_{lj} 
		\left(
			\frac{1}{\sparsity\mD_{ll}} - \frac{1}{\hat\mD_{ll}}
		\right) 
	\end{align*}
	Combining the facts that $\sum_{l=1}^n [\mP^k]_{il} = 1$ by the property of transition matrix and  $\left\| \mD - \tfrac{1}{\sparsity} \hat\mD \right\|_\infty \leq 2n^{1 - \alpha/2} \tau$ by definition of $\event(\mA, \hat \mA, \tau)$ we have that:
	\begin{align*}
		\sum_{l=1}^n [\mP^k]_{il} \cdot 
		\hat\mA_{lj} 
		\left(
			\frac{1}{\sparsity\mD_{ll}} - \frac{1}{\hat\mD_{ll}}
		\right) 
		\leq
		\sum_{l=1}^n [\mP^k]_{il} \cdot 
		1
		\cdot 
		2n^{-1 - \frac{\alpha}{2}} d^{-2} \tau = 2n^{-1 - \frac{\alpha}{2}} d^{-2} \tau
	\end{align*}
	As $\Pr\left[\event(\mA, \hat \mA, \tau)\right] \geq 1 - 2 n e^{-\frac{\tau^2}{1 + \tau}}$ by \Cref{lem:event-1-mA-A-tau} we have that:
	\begin{align}\label{eq:mP-k-mPhat-P-2}
		\Pr \left[
			\left\|  \mP^k(\tfrac{1}{\sparsity}\mD^{-1} - \hat\mD^{-1})\hat\mA  \right\|_{\infty} 
		\geq
		2n^{-1 - \frac{\alpha}{2}} d^{-2} \tau
		\right]
		\leq
		2 n e^{-\frac{\tau^2}{1 + \tau}}
	\end{align}

	\paragraph{Combing results.} Applying union bound over \Cref{eq:mP-k-mPhat-P-1}, \Cref{eq:mP-k-mPhat-P-2} and using triangle inequality \Cref{eq:tri} we have the statement of the lemma.

\end{proof}

\begin{theorem}\label{thm:mP-k-mPhat-P-i}

	\begin{align*}
		\Pr\left[
			\left\| n (\mP^{k} -  \hat\mP^{k}) \right\|_{\infty}
			\geq 
			\left(n^{-\frac{\alpha}{2}}+n^{\frac{1}{2}-\alpha} \right) 
			\frac{4^{k+1}}{d^{k+1}} (1+\tau)\tau
		\right]
		\leq
		O(k n^2 e^{-\tau^2})
	\end{align*}

\end{theorem}

\begin{proof}
	We will prove that the bound holds:
	\begin{align*}
		\left\| n (\mP^{k} -  \hat\mP^{k}) \right\|_{\infty} \leq \left(n^{-\frac{\alpha}{2}}+n^{\frac{1}{2}-\alpha} \right) \frac{4^{k+1}}{d^{k+1}} (1+\tau)\tau
	\end{align*}
	assuming that the events $\event_1(\mA, \hat \mA, \tau), \event_2(\mA, \hat \mA, \tau)$ and $ \bigcup_{i=2}^{k-1} \event_3(\mA, \hat \mA, \tau, i)$ happen. Then applying the union bound we have that it happens with probability at least $1 - O(k n^2 e^{-\tau^2})$. Let us prove this by induction.

	\paragraph{Base case.} $k = 3$: 
	\begin{align*}
		\left\| \mP^{3} -  \hat\mP^{3} \right\|_{\infty}
		\leq
		\left\| \mP^2(\mP -  \hat\mP) \right\|_{\infty}
		+
		\left\| (\mP^{2} -  \hat\mP^{2})\hat\mP \right\|_{\infty}
	\end{align*}
	
	By as event $\event_3(\mA, \hat \mA, \tau, 2)$ happens (\Cref{lem:mP-k-mPhat-P}) we have that:
	\begin{align}\label{eq:mP-2-mPhat-2-i}
			\left\| n\mP^2(\mP -  \hat\mP) \right\|_{\infty}
			\leq
			4n^{- \frac{\alpha}{2}} d^{-2} \tau
	\end{align}

	To bound the second term, note that:
	\begin{align*}
		n\left\| (\mP^{2} -  \hat\mP^{2})\hat\mP \right\|_{\infty}
		\leq 
		n\left\| \diag(\mP^{2} -  \hat\mP^{2})\hat\mP \right\|_{\infty}
		+ 
		n\left\| \offdiag(\mP^{2} -  \hat\mP^{2})\hat\mP \right\|_{\infty} \\
		\leq 
		\underbrace{\left\| \diag(n(\mP^{2} -  \hat\mP^{2})) \right\|_{\infty}}_{\leq d^{-1} n^{1-\alpha}\text{ by \Cref{lem:mP-2-mPhat-2-ii}}} 
		\cdot 
		\underbrace{\left\| \hat\mP \right\|_{\infty}}_{\substack{\leq 2d^{-1}n^{-\alpha} \\ \text{ by \Cref{cor:event-1}}}}
		+
		\underbrace{n\left\| \offdiag(\mP^{2} -  \hat\mP^{2}) \right\|_{\infty}}
		_
		{\substack{\leq 12\left(n^{-\frac{\alpha}{2}}+n^{\frac{1}{2}-\alpha} \right)d^{-3}  (1+\tau)\tau \\ \text{ by  \Cref{lem:mP-2-mPhat-2-ij}}}} 
		\cdot 
		\underbrace{\sup_{j \in [n]} \sum_{l=1}^n \hat\mP_{lj}}_{\substack{\leq 4d^{-1} \\ \text{ by \Cref{cor:event-1}}}}
		\\
		\leq
		2d^{-2} n^{1-2\alpha} \tau 
		+ 
		48\left(n^{-\frac{\alpha}{2}}+n^{\frac{1}{2}-\alpha} \right)d^{-4} (1+\tau)\tau 
		\\
		\leq 
		50\left(n^{-\frac{\alpha}{2}}+n^{\frac{1}{2}-\alpha} \right) d^{-4} (1+\tau)\tau
	\end{align*}

	Hence, combining this with \Cref{eq:mP-2-mPhat-2-i} we have that:
	\begin{align*}
		\left\| \mP^2(\mP -  \hat\mP) \right\|_{\infty}
		\leq
		\roughconst\left(n^{-\frac{\alpha}{2}}+n^{\frac{1}{2}-\alpha} \right) d^{-4} (1+\tau)\tau
		 \leq  
		 \left(n^{-\frac{\alpha}{2}}+n^{\frac{1}{2}-\alpha} \right) \frac{4^{4}}{d^{4}} (1+\tau)\tau
	\end{align*}
	
	\paragraph{Inductive step.} Assume events $\event_1(\mA, \hat \mA, \tau), \event_2(\mA, \hat \mA, \tau)$ and $\bigcup_{i=2}^{k-1} \event_3(\mA, \hat \mA, \tau, i)$ happen. Then the statement holds for $k$ by inductive assumption. Assume that event $\event_3(\mA, \hat \mA, \tau, k)$ happens as well. We want to show that now the bound holds for $k+1$. 
	\begin{align*}
		\left\| \mP^{k+1} -  \hat\mP^{k+1} \right\|_{\infty}
		\leq
		\left\| \mP^{k}(\mP -  \hat\mP) \right\|_{\infty}
		+
		\left\| (\mP^{k} -  \hat\mP^{k})\hat\mP \right\|_{\infty}
	\end{align*}
	
	As $\event^{(k)}(\mA, \hat \mA, \tau)$ happens, by \Cref{lem:mP-k-mPhat-P} we have that:
	\begin{align}\label{eq:mP-k-mPhat-P-i-whatever}
			\left\| n \mP^k(\mP -  \hat\mP) \right\|_{\infty}
			\leq
			4n^{ - \frac{\alpha}{2}} d^{-2} \tau
			\leq
			\left(n^{-\frac{\alpha}{2}}+n^{\frac{1}{2}-\alpha} \right) \frac{2 \cdot 4^{k}}{d^{k+1}} (1+\tau)\tau 
	\end{align}

	By  assumption that events $\event_1(\mA, \hat \mA, \tau), \event_2(\mA, \hat \mA, \tau)$ and  $\bigcup_{i=2}^{k-1} \event_3(\mA, \hat \mA, \tau, i)$ happen and the fact that $\sum_{l=1}^n \hat\mP_{lj}  \leq 2d^{-1}$ (\Cref{cor:event-1}) we have that the following bound holds as well:
	\begin{align*}
		\left\| n(\mP^{k} -  \hat\mP^{k})\hat\mP \right\|_{\infty}
		\leq 
		n \| \mP^{k} -  \hat\mP^{k} \|_\infty \cdot \sum_{l=1}^n \hat\mP_{lj} 
		\leq 
		\left(n^{-\frac{\alpha}{2}}+n^{\frac{1}{2}-\alpha} \right) \frac{4^{k}}{d^{k}} (1+\tau)\tau \cdot 2d^{-1}
	\end{align*}

	Combining this with \Cref{eq:mP-k-mPhat-P-i-whatever} we have that:
	\begin{align*}
		\left\| \mP^{k}(\mP -  \hat\mP) \right\|_{\infty}
		\leq
		\left(n^{-\frac{\alpha}{2}}+n^{\frac{1}{2}-\alpha} \right) \frac{4^{k+1}}{d^{k+1}} (1+\tau)\tau 
	\end{align*}
	
	As \begin{align*}
		\Pr[\event_1(\mA, \hat \mA, \tau)] &\geq 1 - O(n^2 e^{-\tau^2}) \text{ by \Cref{lem:event-1-mA-A-tau}}	 \\
		 \Pr[\event_2(\mA, \hat \mA, \tau)] &\geq 1 - O(n^2 e^{-\tau^2}) \text{ by \Cref{lemma:event-2-mA-A-tau}} \\
		 \Pr\left[\bigcup_{i=2}^{k-1} \event_3(\mA, \hat \mA, \tau, i)\right] &\geq 1 - O(k n^2 e^{-\tau^2}) 
	\end{align*} 
	we have that:
	\begin{align*}
		\Pr\left[
			\left\| n (\mP^{k} -  \hat\mP^{k}) \right\|_{\infty} \leq \left(n^{-\frac{\alpha}{2}}+n^{\frac{1}{2}-\alpha} \right) \frac{4^{k+1}}{d^{k+1}} (1+\tau)\tau
		\right]
		\geq
		1 - O(k n^2 e^{-\tau^2})
	\end{align*}
\end{proof}

\subsection{Combining results}

Now we are going to combine the results of the previous two subsections to prove the main result of this section (\Cref{lem:stable-positional-encoding-example-full}).

\begin{proof}
	By \Cref{thm:mP-k-mPiK} we have that $\leq 4 n^2 e^{-2\tau^2}$:
	\begin{align}\label{eq:last:mP-k-mPiK}
		\left\|{(n-1)}\mP^k - \mPi^{(k)} \right\|_\infty 
			\leq 
			\frac{4\tau k \cdot \delta^{-k}}{n} + \frac{k^2 \delta^{-k}}{2(n-k)^k}
	\end{align}

	By \Cref{lem:mD-mDelta} we have that $2 n e^{-2\tau^2}$:
		\begin{align*}
			\left\| \tfrac{1}{n-1}\mD - \mDelta \right\|_\infty \leq \tfrac{2\tau}{\sqrt{n}}
		\end{align*}

	As we assume that every $\Delta_{ii} \geq \delta$, $\tau \leq \delta$ and $n^2 \geq \delta^{-1}/4$ we have that:
	\begin{align*}
		\tfrac{1}{n}\mD_{ii} \geq \tfrac{1}{n-1}\mD_{ii} - \tfrac{1}{n(n-1)}\mD_{ii} 
		\geq 
		\delta - \tfrac{2\tau}{\sqrt{n}} - \tfrac{1}{n^2} 
		\geq 
		\tfrac{\delta}{2}
	\end{align*}

	As $\tfrac{1}{n}\mD_{ii} \geq \tfrac{\delta}{2}$ we have that by \Cref{thm:mP-k-mPhat-P-i} we have that:
	\begin{align}\label{eq:last:mP-k-mPiK-2}
		\left\| n (\mP^{k} -  \hat\mP^{k}) \right\|_{\infty} 
		\leq 
		\left(n^{-\frac{\alpha}{2}}+n^{\frac{1}{2}-\alpha} \right) \frac{8^{k+1}}{\delta^{k+1}} (1+\tau)\tau
	\end{align}

	Also, observe that by \Cref{lemma:P-bounds} we have that:
	$$\left\|n \mP^{k}  \right\|_{\infty} \leq \frac{1}{d} \leq \frac{2}{\delta}$$

	Hence, combining it with \Cref{eq:last:mP-k-mPiK} and \Cref{eq:last:mP-k-mPiK-2} by triangle inequality we have that:
	\begin{align*}
		\left\|n \hat\mP^{k}  - \mPi^{(k)} \right\|_{\infty} 
		\leq 
		\frac{2}{\delta n} 
		+ 
		\left(n^{-\frac{\alpha}{2}}+n^{\frac{1}{2}-\alpha} \right) \frac{8^{k+1}}{\delta^{k+1}} (1+\tau)\tau
		+
		\frac{4\tau k }{n\delta^{k}} + \frac{k^2}{2\delta^{k}(n-k)^k}
	\end{align*}

	Assuming $k \leq \sqrt{n}$ this bound can be simplified to:
	\begin{align*}
		\left\|n \hat\mP^{k}  - \mPi^{(k)} \right\|_{\infty} 
		\leq 
		2\left(n^{-\frac{\alpha}{2}}+n^{\frac{1}{2}-\alpha} \right) \frac{8^{k+1}}{\delta^{k+1}} (1+\tau)\tau
	\end{align*}
\end{proof}



\section{Experimental Details}
\label{apx:experiments}

This appendix provides full details to reproduce the experiments in \Cref{sec:experiments}: (i) worst-case Transformer output error on point clouds and graphons (\Cref{apx:experiments:worst-case}) and (ii) the stable-vs-unstable RPE regression experiment (\Cref{apx:experiments:classification}).



\subsection{Worst-case Transformer Output Error}

\begin{figure}[h]
    \centering
    \begin{subfigure}[b]{0.45\textwidth}
        \centering
		\includegraphics[width=\linewidth]{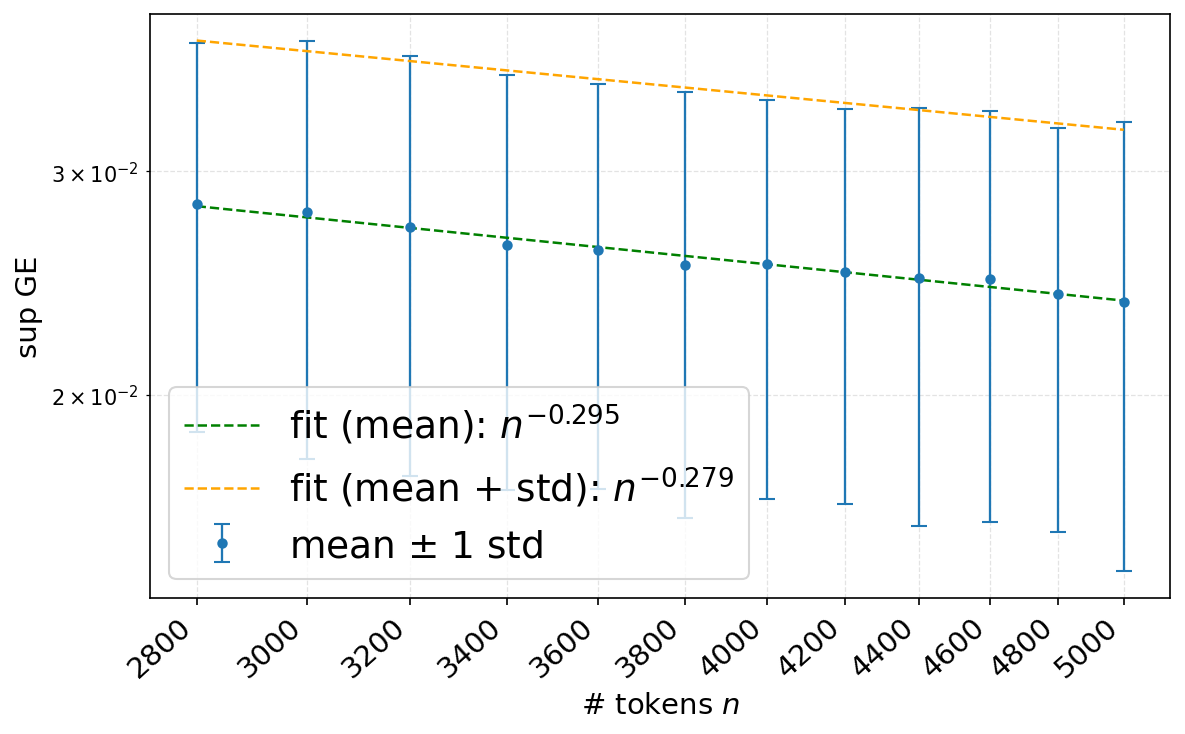}
        \vspace{-6mm}
        \caption{Point Clouds}   
        \label{fig:pointclouds-2}
    \end{subfigure}
    \hfill
    \begin{subfigure}[b]{0.45\textwidth}
        \centering
        \includegraphics[width=\linewidth]{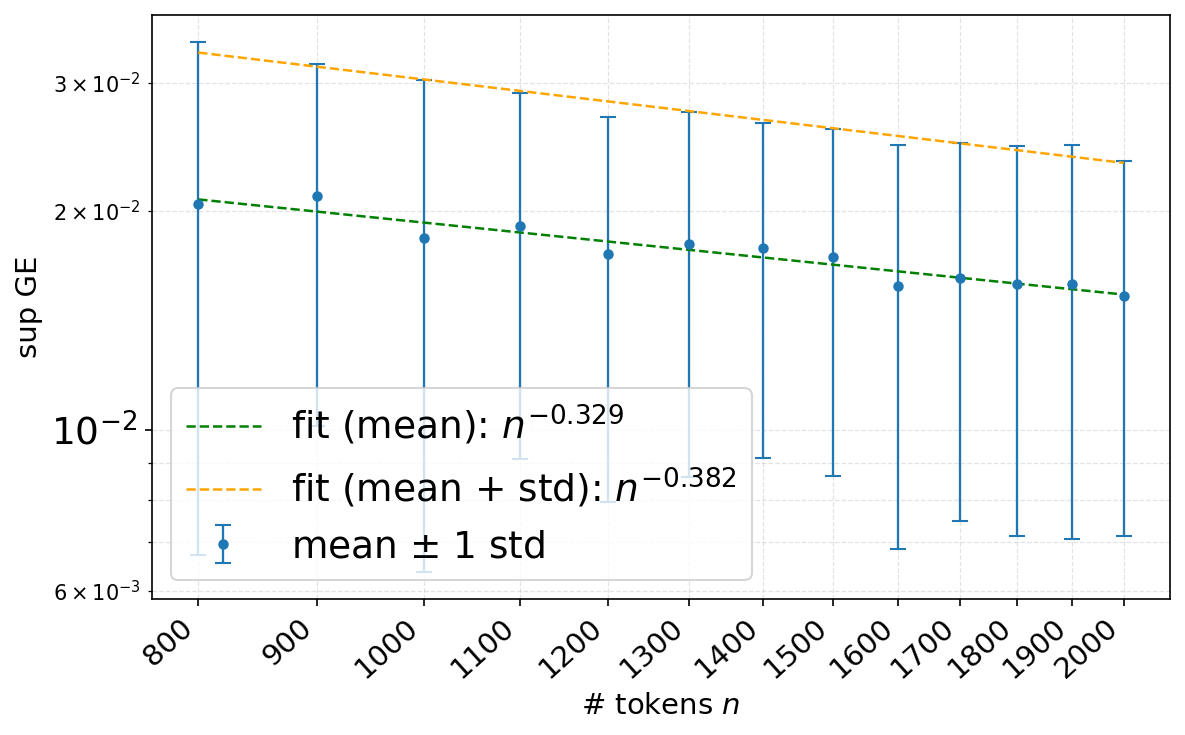}
         \vspace{-6mm}
        \caption{Graphs}
        \label{fig:graphs-2}
    \end{subfigure}
	\caption{
        Generalization error vs. the number of tokens $n$ on a log-log scale. Each point shows mean error $+$ one standard deviation over 100 runs. The x-axis is the number of tokens $n$, and the y-axis is the generalization error $\sup_{\GAT \in \hypothesis} \left\|  \GAT(\ST) - \GAT\left(\TT\right) \right\|_{2}$ for $\ST \sample{n} \TT$. \textbf{(a)} Point clouds: slopes $O(n^{-0.29})$ (mean) and $O(n^{-0.28})$ (mean$+$std). \textbf{(b)} Graphs: slopes $O(n^{-0.38})$ (mean) and $O(n^{-0.33})$ (mean$+$std).
    }
    \label{fig:combined_ge-2}
\end{figure} 

\label{apx:experiments:worst-case}

\subsubsection{Model Architecture}
\label{apx:experiments:worst-case:arch}
Single self-attention layer (single head) $\rightarrow$ global mean pooling $\rightarrow$ 2-layer MLP head. Attention logits:
\[
\logit_{ij} \;=\; (K\,\nf(x_i))^\top (Q\,\nf(x_j))\;+\;\phi(\PE(x_i,x_j)),
\]
where $\phi:\mathbb{R}^{\PEdim}\to\mathbb{R}$ is a 2-layer MLP with hidden size 128, LeakyReLU, and spectral projection (as above). Softmax over $j$, values $V\,\nf(x_j)$, output $h_i = \sum_j \Att_{ij}\, V\,\nf(x_j)$, pooled by mean, then MLP head with LeakyReLU between layers.

\paragraph{Operator-norm (spectral) normalization.}
In our implementation of the Transformer, we ensure the 1-Lipschitz constraint as follows:
\begin{itemize}
  \item \textbf{Initialization:} Each weight matrix $W$ is initialized with spectral normalization using the \texttt{spectral\_norm} function, which scales $W$ by $1/\sigma_{\max}(W)$.
  \item \textbf{Projection after each step:} Post-optimizer update, matrices are projected onto the spectral-norm ball using spectral normalization hooks, ensuring $\|\cdot\|_{\mathrm{op}}\le 1$.
  \item \textbf{Attention logits/value heads:} Spectral normalization is applied to $Q,K,V,O$ and the two MLP layers within $\phi$.
\end{itemize}

\paragraph{Training objective (worst-case).}
For each pair $(\ST^*, \ST_{n,i})$, we \emph{maximize} the output discrepancy:
\[
\mathcal{L}_{\text{worst}}(\Theta;\ST^*,\ST_{n,i}) \;=\; -\,\big\| \GAT(\ST^*) - \GAT(\ST_{n,i}) \big\|_2.
\]
We implement this by minimizing $-\|\cdot\|_2$ with SGD (i.e., gradient ascent on the norm).

\paragraph{Evaluation and reporting.}
All curves show the mean and one standard deviation over 100 independent runs per $n$. Slopes are obtained via least-squares on the log–log plot over the range of tokens considered. We report the slope of the mean curve and the mean$+$std curve. Confidence intervals for slopes are from ordinary least squares standard errors.

\subsubsection{Point Clouds}
\label{apx:experiments:worst-case:pc}

\textbf{Continuous tokenset $\TT$.} We use a single mesh from ModelNet40 \citep{wu20153d} . Surface normals and coordinates are used as point attributes. 
\begin{itemize}
  \item High-res relaxation $\ST^*$: \texttt{10,000} points uniformly sampled from the mesh \texttt{person\_0010.off}. 
  \item Low-res sets $\ST_{n,i}$: for $n \in \{	2800, 3000, \ldots, 5000\}$; for each $n$, draw \texttt{100} i.i.d. samples.
  \item Features $\nf(x)$: \texttt{[xyz, normal]}
  \item Positional encoding $\PE(x,y) = x-y \in \mathbb{R}^3$.
\end{itemize}

\textbf{Transformer Details.} The transformer used is a Tokenset with normalization of the layers with the following parameters:
\begin{itemize}
  \item Input Dimension: 2, Output Dimension: 2, Number of Layers: 1, Hidden Dimension: 5, Bias: False
  \item RPE Parameters: Input Dimension: 3, Bias: False, RPE Dimension: 1, Hidden Dimension: 5
\end{itemize}

\textbf{Training hyperparameters (point clouds).}
\begin{center}
\begin{tabular}{l l}
Epochs & \texttt{5000} \\
LR & \texttt{1e-2} \\
\end{tabular}
\end{center}

\textbf{Evaluation.} For each run we compute $\| \GAT(\ST^*) - \GAT(\ST_{n,i}) \|_2$, then report mean and mean$+$std over 25 runs per $n$. Slopes from OLS on $(\log n, \log \text{error})$ pairs.

\subsubsection{Graphs}
\label{apx:experiments:worst-case:graphs}

\textbf{Continuous graphon $\TT$.} Latent space $\chi=[0,1]$ with uniform measure. Graphon:
\[
W(x,y)=\Big(\tfrac{\sin(2\pi x)\sin(2\pi y)+1}{2}\Big)^5 \cdot p + q,\quad p=1,\; q=10^{-3}.
\]
\textbf{High-res relaxation $\ST^*$.} Weighted, fully connected graph on $N^*=\texttt{5000}$ nodes $x_i\sim \mathrm{Unif}[0,1]$, weights $w_{ij}=W(x_i,x_j)$.  
\textbf{Low-res sets $\ST_{n,i}$.} For $n \in \{800,900,\ldots,2000\}$ and each $n$ we draw \texttt{100} graphs: nodes $x_i\sim \mathrm{Unif}[0,1]$; edges $e_{ij}=e_{ji}\sim \mathrm{Ber}(W(x_i,x_j))$ (unweighted). Node features $\nf(x_i)=[x_i, 1-x_i]$.

\textbf{RPEs.} $\PE_G(i,j) = n [\mP^3]_{ij}$, where $\mP=\mD^{-1}\mA$ (row-normalized), with $\mD_{ii}=\deg(i)$.

\textbf{Transformer Details.} The transformer used is a Tokenset with normalization of the layers with the following parameters:
\begin{itemize}
  \item Input Dimension: 2, Output Dimension: 2, Number of Layers: 1, Hidden Dimension: 5, Bias: False
  \item RPE Parameters: Input Dimension: 2, Bias: False, RPE Dimension: 1, Hidden Dimension: 3
\end{itemize}

\textbf{Training hyperparameters.}
\begin{center}
\begin{tabular}{l l}
Epochs & 1000 \\
LR & 0.1 \\
\end{tabular}
\end{center}

\textbf{Evaluation.} Same as point clouds; report $O(n^{-\alpha})$ slope estimates of the fit.

\subsection{Stable vs.\ Unstable RPE: Classification Experiment}

\begin{figure}[t]
    \centering
    \includegraphics[width=0.8\linewidth]{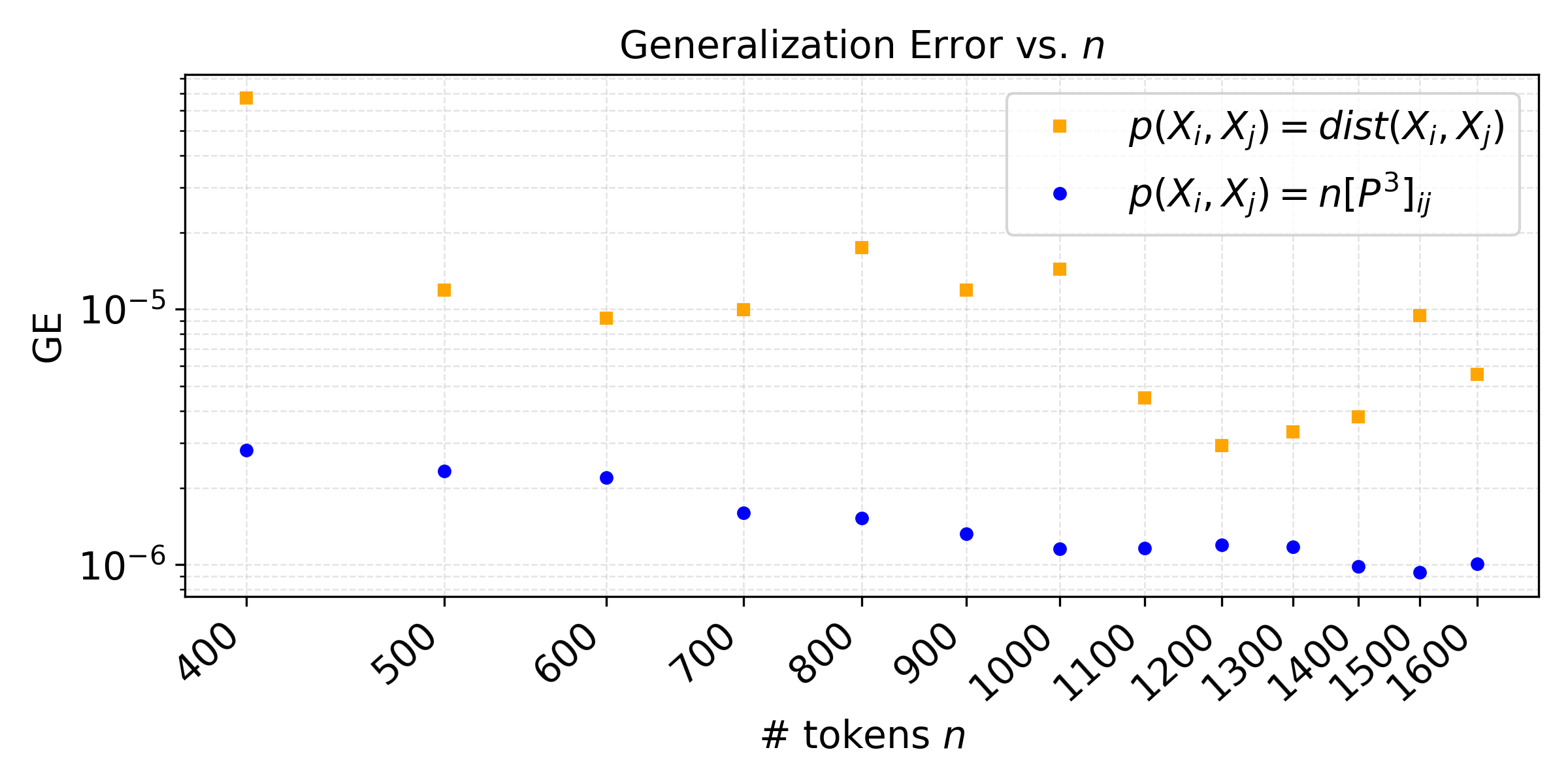}
    \vspace{-6mm}
    \caption{Generalization error (log scale) vs. the number of tokens $n$ for two RPE methods. Transition-matrix-based RPE (blue) generalizes better than shortest-path-based RPE (orange). 
	}
    \vspace{-2mm}
    \label{fig:rpe-comparison-2}
\end{figure}

\label{apx:experiments:classification}

\paragraph{Data generation.}
Two graphons on $\chi=[0,1]$ with uniform measure:
\[
W_0(x,y)=
\begin{cases}
0.9, & x,y \in [0,0.5]\ \text{or}\ x,y \in [0.5,1], \\
10^{-3}, & \text{otherwise},
\end{cases}
\qquad
W_1(x,y)=0.3.
\]
Labels $y_0=0$, $y_1=1$. For each tokenset, sample $n\in\{100,200,\ldots,1500\}$ nodes i.i.d.\ uniform and edges $e_{ij}\sim \mathrm{Ber}(W_\ell(x_i,x_j))$ with $\ell\in\{0,1\}$. Node features $\nf(x_i)=[\mathbb{I}\{x_i\le 0.5\},\mathbb{I}\{x_i>0.5\}]$ (or \texttt{[1,0]} / \texttt{[0,1]}). 

\paragraph{Datasets and splits.}
For each $n$, training set $\Omega_n$ with \texttt{1000} graphs (balanced classes) and a high-res test set $\Omega^*$ of \texttt{100} graphs with $10000$ nodes each.

\paragraph{Model and Hyperparameters.}
Tokenset Transformer with 2 layers, input dimension 2, hidden dimension 8, key/query dimension 2, MLP dimensions 2→8→2, LeakyReLU activation, without layer normalization. RPE uses hidden dimension 16 and RPE dimension 1.

\paragraph{RPE conditions.}
\begin{enumerate}[label=(\alph*)]
  \item \textbf{Unstable}: $\PE(i,j)=\text{shortest-path-distance}(i,j)$ if $i$ and $j$ are connected, and $-1$ otherwise (computed per graph).  
  \item \textbf{Stable}: $\PE(i,j)=n[\mP^3]_{ij}$ with $\mP=\mD^{-1}\mA$, where $\mD$ is the degree matrix and $\mA$ is the adjacency matrix.
\end{enumerate}

\paragraph{Loss and training.}
Cross-entropy loss; Adam (LR \texttt{1e-2}, batch size \texttt{10}, epochs \texttt{10}).

\paragraph{Generalization error metric.}
We report 
\[
\epsilon \;=\; |R_{\text{test}} - R_{\text{train}}|,
\]
with $R_{\text{train}}$ the average cross-entropy over $\Omega_n$, and $R_{\text{test}}$ the average over $\Omega^*$. For each $n$ we average $\epsilon$ over \texttt{10} independent seeds and plot mean\,$\pm$\,std on a log scale.

\end{document}